\documentclass[twoside,11pt]{article}

\usepackage[preprint]{jmlr2e}

\usepackage{amsmath,amsfonts,bm}

\def\eqref#1{equation~\ref{#1}}

\def\1{\bm{1}}

\DeclareMathAlphabet{\mathsfit}{\encodingdefault}{\sfdefault}{m}{sl}
\SetMathAlphabet{\mathsfit}{bold}{\encodingdefault}{\sfdefault}{bx}{n}

\usepackage{amsmath}
\usepackage{algorithm,algpseudocode}
\usepackage{titletoc}
\usepackage{wrapfig}
\usepackage{subcaption}
\usepackage{multirow}
\usepackage{booktabs}

\title{BMuon: Budgeting Muon's Spectral Geometry for Structured Compressibility}

\author{%
\name Rida Assalouh \email rida.assalouh@polytechnique.edu \\
\addr Inria, Paris, France \\
École polytechnique, Institut Polytechnique de Paris, Palaiseau, France
\AND
\name Tom Jacobs \email tom.jacobs@cispa.de \\
\addr CISPA Helmholtz Center for Information Security, Saarbrücken, Germany
\AND
\name Rebekka Burkholz \email burkholz@cispa.de \\
\addr CISPA Helmholtz Center for Information Security, Saarbrücken, Germany
\AND
\name Yani Ioannou \email yani.ioannou@ucalgary.ca \\
\addr Schulich School of Engineering, University of Calgary, Calgary, Alberta, Canada\\Hotchkiss Brain Institute, University of Calgary, Calgary, Alberta, Canada
\AND
\name Umut Şimşekli \email umut.simsekli@inria.fr \\
\addr Inria -- Département d’Informatique de l’École Normale Supérieure,
PSL Research University -- CNRS, Paris, France
}

\ShortHeadings{BMuon}{Assalouh et al.}

\DeclareMathOperator{\Polar}{Polar}

\begin{document}

\maketitle
\thispagestyle{empty}

\begin{abstract}
Muon has emerged as a strong optimizer for neural network training and admits an interpretation as steepest descent under a spectral-norm geometry. Yet we find that Muon-trained Transformers can degrade sharply when their MLP hidden channels are pruned after training. To reconcile Muon's optimization performance with structured compressibility, we introduce Budget Muon (BMuon), which modifies Muon's spectral geometry so that hidden channels across blocks compete for a shared allocation budget. The resulting oracle is exactly normalized steepest descent under a family of norms that interpolate between Muon's geometry and structured sparse geometries. Our principal sparsity control is an interpretable global channel-allocation budget, rather than a setting-dependent sparsity-penalty coefficient. For linear factorization models, we show that Muon admits minimum-cost exact representations with redundant hidden width, whereas under the BMuon oracle geometry and a sufficiently scarce budget, every minimum-cost representation is width-minimal. For practical training, we derive a minorization-maximization surrogate for the exact allocation problem. On ImageNet-1K with ViT-S/16, BMuon improves over Muon under gradual structured MLP pruning, and its learned allocations provide a more effective pruning criterion than weight magnitude. On a 1B-parameter LLaMA-style model trained on FineWeb, BMuon improves over Muon under post-training MLP channel pruning.

\noindent\textbf{Code:} \url{https://github.com/RidaAssalouh/BMuon}
\end{abstract}


\section{Introduction}
\label{sec:introduction}

Modern Transformer models incur substantial memory and computational costs at inference time, motivating compression for more efficient deployment \citep{kwon2022fast}.
Pruning is a natural approach to this problem, but practical gains depend on the sparsity structure: unstructured sparsity often requires specialized kernels or hardware, whereas structured pruning removes entire computational units and translates directly into efficiency gains \citep{wen2016learning,an2024fluctuation}.  
Feed-forward modules can account for roughly two thirds of Transformer
parameters \citep{geva2021transformer}, making their hidden channels a natural
target for structured compression.
Global pruning ranks channels across layers jointly, allowing different layers to retain different numbers of channels \citep{lee2021layeradaptive,an2024fluctuation}.
Post-pruning fine-tuning adds computational cost, especially for large models \citep{kwon2022fast,sun2024simple}, motivating compressibility to emerge during training so that pruning requires little or no recovery.

\begin{wrapfigure}{r}{0.47\textwidth}
\vspace{-0.6em}
\centering
\includegraphics[width=\linewidth]{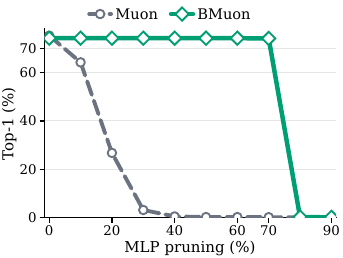}
\caption{
Post-training global MLP pruning on ImageNet-1K with ViT-S/16.
Mean $\pm$ std over 3 seeds.
See Appendix~\ref{app:motivation_pruning}.
}
\label{fig:motivation}
\vspace{-1em}
\end{wrapfigure}
Muon has demonstrated strong training efficiency across model scales, from small vision and language-model benchmarks to billion-parameter language models \citep{jordan2024muon,liu2025muon,shah2025practical}.
Its matrix update can be viewed as a linear minimization oracle over a
spectral-norm ball, giving Muon a steepest-descent interpretation in
spectral-norm geometry \citep{pethick2025training}.
Prior work has shown that implicit bias can depend on optimization geometry
\citep{gunasekar2018characterizing}. For Muon, recent theory identifies a
spectral-norm margin bias in linear and smooth homogeneous models
\citep{fan2025implicit,gronich2026implicit}.
Yet we find that Muon-trained Transformers degrade rapidly under structured MLP pruning, suggesting that Muon's spectral geometry does not by itself encourage structured compressibility.
As shown in Figure~\ref{fig:motivation}, a Muon-trained ViT-S/16 on ImageNet-1K collapses after only $30\%$ global MLP channel pruning,
with Top-1 accuracy falling from $75.23\%$ to $3.17\%$.
In contrast, our proposed \emph{Budget Muon} (BMuon) algorithm retains $74.08\%$ Top-1 accuracy after $70\%$ pruning, nearly matching its $74.17\%$ accuracy before pruning, with no post-pruning recovery. This raises the central question: can Muon’s spectral geometry be modified to encourage structured compressibility?
%


To address this question, we introduce BMuon, which reweights Muon's spectral geometry under a single global budget $K$ shared across all hidden channels of all MLP blocks.
For hidden channel $j$ in block $\ell$, BMuon assigns an allocation $\mu_j^\ell\in[0,1]$ to the corresponding row of the up-projection matrix and column of the down-projection matrix, with $\sum_{\ell,j}\mu_j^\ell=K$.
For a fixed allocation, BMuon introduces a diagonal reweighting inside
Muon's spectral-norm constraint, with the components of the update variable
corresponding to channel $j$ in block $\ell$ weighted according to its
allocation $\mu_j^\ell$.
BMuon chooses the allocation globally, intuitively asking where the update budget should be spent to obtain the greatest first-order decrease in the loss.
%


%
NuMuon~\citep{dolatabadi2026numuon} is the closest prior compression-aware
variant of Muon. It imposes a nuclear-norm constraint on Muon's update
direction to improve low-rank compressibility. BMuon instead targets structured
compressibility over hidden channels by introducing a global channel
budget into Muon's spectral-norm geometry.
Prior work has also incorporated compressibility into training.
\citet{Lu2022} train networks for unstructured pruning with Frank--Wolfe under fixed layerwise constraints on the weights, while \citet{Zimmer2022} use Frank--Wolfe under fixed norm constraints for convolutional filter pruning and low-rank compression.
BMuon instead leaves the weights unconstrained and modifies Muon's spectral update geometry for global hidden-channel compression.
Resource-constrained structure learning and global structured pruning have been
studied across CNNs and Transformers
\citep{gordon2018morphnet,yu2022unified,yang2023global,
an2024fluctuation,wang2026localglobal}.
GRASPrune~\citep{wang2026grasprune} learns gates post hoc on a frozen LLM to
select MLP channels and attention KV-head groups under a single global parameter
budget, while DABP~\citep{laroui2026dabp} jointly trains CNN weights and
continuous channel gates under a global resource budget.
BMuon uses no forward gates; its budget instead controls the channel-wise
reweighting of Muon's spectral update constraint during training.


Theoretically, we show that the BMuon oracle is exactly normalized steepest
descent under an induced norm $\Omega_{\alpha,K}$, where $\alpha$ controls the
strength of channel reweighting and $K$ is the global budget over the $N$
hidden channels.
This family recovers Muon when $\alpha=0$ or $K=N$. At $\alpha=1$, its unit ball is the convex hull of operator-norm-bounded updates supported on at most $K$ coupled hidden channels.
This is related to group $k$-support norms through the use of structured
support~\citep{argyriou2012sparse,rao2017group}, but its atoms are bounded in operator norm rather than Euclidean norm.
It is also related to spectral $k$-support norms~\citep{mcdonald2014spectral,mcdonald2016fitting}, but differs in the notion of sparsity: spectral $k$-support uses sparsity in singular values, whereas BMuon uses MLP-channel sparsity.
At $\alpha=K=1$, the norm reduces to an $\ell_1$ aggregation over
coupled channel groups, yielding a group-lasso-like geometry
\citep{yuan2006model}.
We then characterize the representation cost induced by BMuon. Following \citet{dai2021representation}, this is the minimum parameter-space cost among exact two-matrix factorizations that realize the same linear predictor.
We prove that, under an explicit budget-scarcity condition, every minimum-cost BMuon factorization is width-minimal, whereas Muon's minimum cost can also be attained by factorizations with redundant hidden width.


Solving the exact allocation problem at every training step is expensive. We therefore derive a minorization--maximization surrogate that
yields a tractable allocation update, followed by the corresponding weighted
Muon update.
On CIFAR-100 and ImageNet-1K, BMuon outperforms Muon under both gradual and
post-training global MLP-channel pruning, even when both use the same
magnitude-based pruning criterion.
Under gradual pruning, using the learned allocations themselves to rank
channels further improves performance, showing that they also provide a useful
measure of channel importance.
The non-uniform architectures identified by BMuon in the gradual-pruning
experiments also outperform matched uniform-width architectures when retrained
from scratch, providing evidence that the learned allocation captures useful
differences in how MLP capacity should be distributed across layers.
Finally, on a 1B-parameter LLaMA-style model trained on FineWeb, BMuon
improves robustness to post-training global MLP-channel pruning, showing that
the effect extends to language modeling at substantially larger scale.


\paragraph{Contributions.} 
To summarize, our main contributions are:
\begin{itemize}
    \item We introduce Budget Muon (BMuon), which incorporates a
    single global hidden-channel budget into Muon's spectral update geometry,
    while coupling the matrix slices that define each MLP channel.

    \item We characterize the induced norm family $\Omega_{\alpha,K}$ and its
    relation to Muon and structured sparse geometries in Section \ref{sec:budget-muon}. In the linear
    factorization setting of Section \ref{sec:compactness}, we show that for a sufficiently scarce
    budget, every minimum-cost BMuon factorization is width-minimal, whereas
    Muon's minimum cost is also attained by redundant-width
    factorizations.

    \item We derive a practical minorization--maximization surrogate for the
    global allocation problem, yielding a tractable allocation update in Section \ref{sec:practical}.

    \item In Section~\ref{sec:experiments}, we show across CIFAR-100,
    ImageNet-1K, and a 1B-parameter LLaMA-style model that BMuon improves
    structured MLP-channel compressibility relative to Muon.
    In the vision experiments, the learned allocations also provide effective
    pruning scores and identify non-uniform MLP architectures that outperform
    matched uniform-width architectures when retrained from scratch.
\end{itemize}

\section{Budget Muon: Oracle and Geometry}
\label{sec:budget-muon}

\subsection{Setup and Muon Preliminaries}
\label{sec:setup-preliminaries}

\paragraph{MLP channels and global pruning.}
We consider \(L\) single-hidden-layer MLP blocks, as in Transformer feed-forward
modules, with block \(\ell\) given by
\begin{equation}\label{equation : setting MLP}
x \mapsto
W_2^\ell\,\phi\!\left(W_1^\ell x+b^\ell\right),
\qquad
W_1^\ell\in\mathbb{R}^{m_\ell\times d_\ell},
\quad
W_2^\ell\in\mathbb{R}^{p_\ell\times m_\ell},
\end{equation}
where \(m_\ell\) is the hidden width. Hidden channel \(j\) is associated with
its incoming weights \(W_{1,j:}^\ell\), bias \(b_{j}^\ell\), and
outgoing weights \(W_{2,:,j}^\ell\). Pruning the channel removes these
components jointly. The matrix parameters \(W_1^\ell,W_2^\ell\) are optimized
by Muon, while biases are optimized separately by AdamW. We write \(W:=\big((W_1^\ell,W_2^\ell)\big)_{\ell=1}^L\) for the collection of MLP weight matrices and \(N:=\sum_{\ell=1}^L m_\ell\) for the total number
of hidden channels across all blocks. Given channel scores \(s_j^\ell\), global pruning ranks all \(N\) channels jointly and removes the lowest-scoring ones until the target global sparsity is reached. The resulting hidden width in each block is therefore determined by this global ranking.

\paragraph{Muon as spectral-norm steepest descent.}
For an MLP matrix \(W_i^\ell\), let \(M_i^\ell\) denote the momentum matrix
used by Muon, following Appendix~\ref{app:muon-algo}. For any matrix \(A\) with
compact singular value decomposition (SVD) \(A=U\Sigma V^\top\), define its polar factor by
\(\Polar(A):=UV^\top\). The exact Muon direction is
\(\Delta_{i,\mathrm{Muon}}^\ell=\Polar(M_i^\ell)\), and \(W_i^\ell\) is
updated in the opposite direction. In practice, Muon approximates this
direction using a finite Newton--Schulz iteration~\citep{jordan2024muon};
throughout the theoretical analysis, we use the exact polar factor. By spectral--nuclear norm duality, the exact Muon direction satisfies
\[
\Delta_{i,\mathrm{Muon}}^\ell
=
\Polar(M_i^\ell)
\in
\operatorname*{arg\,max}_{\|\Delta_i^\ell\|_{\mathrm{op}}\le 1}
\langle M_i^\ell,\Delta_i^\ell\rangle,
\qquad
\max_{\|\Delta_i^\ell\|_{\mathrm{op}}\le 1}
\langle M_i^\ell,\Delta_i^\ell\rangle
=
\|M_i^\ell\|_*.
\]
Here \(\langle A,B\rangle=\operatorname{tr}(A^\top B)\) is the Frobenius
inner product, \(\|A\|_{\mathrm{op}}=\sigma_{\max}(A)\) is the operator norm,
and \(\|A\|_*=\sum_j\sigma_j(A)\) is the nuclear norm. Thus exact-polar Muon is normalized steepest descent under the spectral
norm~\citep{bernstein2024old,pethick2025training}.
For all MLP matrices, write
\(M:=\big((M_1^\ell,M_2^\ell)\big)_{\ell=1}^L\) and
\(\Delta:=\big((\Delta_1^\ell,\Delta_2^\ell)\big)_{\ell=1}^L\).
We equip this product space with the inner-product
\(\langle M,\Delta\rangle
:=\sum_{\ell=1}^L
\big(
\langle M_1^\ell,\Delta_1^\ell\rangle+
\langle M_2^\ell,\Delta_2^\ell\rangle
\big)\)
and define Muon's product-space unit ball
\(\mathcal A_{\mathrm{Muon}}
:=\{\Delta:
\|\Delta_1^\ell\|_{\mathrm{op}}\le1,\
\|\Delta_2^\ell\|_{\mathrm{op}}\le1,\
\forall\,\ell\}\).
Equivalently,
\(\mathcal A_{\mathrm{Muon}}
=\{\Delta:\|\Delta\|_{\mathrm{Muon}}\le1\}\), where
\(\|\Delta\|_{\mathrm{Muon}}
:=\max_\ell\max\{
\|\Delta_1^\ell\|_{\mathrm{op}},
\|\Delta_2^\ell\|_{\mathrm{op}}
\}\).
Muon's update direction over all MLP matrices is then
\begin{align}
\label{eqn:muonup_org}
    \Delta_{\mathrm{Muon}}
\in
\operatorname*{arg\,max}_{\Delta\in\mathcal A_{\mathrm{Muon}}}
\langle M,\Delta\rangle.
\end{align}
Since the objective and constraints separate across matrices, this is
equivalent to applying the single-matrix Muon direction independently to each
\(W_i^\ell\).

\subsection{The BMuon Oracle}
\label{sec:oracle}

\paragraph{Global channel budget.}
For each block \(\ell\) of width \(m_\ell\), let
\(\mu^\ell\in[0,1]^{m_\ell}\) assign an allocation to each hidden channel, and
let \(D_{\mu^\ell}:=\operatorname{diag}(\mu^\ell)\). Writing
\(\mu:=(\mu^1,\ldots,\mu^L)\in[0,1]^N\), for a budget \(K\in(0,N]\), define
\(\mathcal K_K:=\{\mu\in[0,1]^N:
\sum_{\ell,j}\mu_j^\ell=K\}\).
All hidden channels share the budget \(K\), so channels across blocks compete
directly for allocation.

\paragraph{Channel-weighted Muon directions.}
Fix \(\alpha\in(0,1]\). For a fixed allocation \(\mu\in\mathcal K_K\), BMuon reweights the spectral constraints in \(\mathcal A_{\mathrm{Muon}}\) channel-wise, yielding
\[
\mathcal A_\alpha(\mu)
:=
\Big\{
\Delta:\ 
\exists\,\{P_1^\ell,P_2^\ell\}_{\ell=1}^L,\ 
\Delta_1^\ell=D_{\mu^\ell}^\alpha P_1^\ell,\ 
\Delta_2^\ell=P_2^\ell D_{\mu^\ell}^\alpha,\ 
\|P_1^\ell\|_{\mathrm{op}}\le1,\ 
\|P_2^\ell\|_{\mathrm{op}}\le1,\ 
\forall\,\ell
\Big\}.
\]
The same factor \((\mu_j^\ell)^\alpha\) scales row \(j\) of
\(\Delta_1^\ell\) and column \(j\) of \(\Delta_2^\ell\), coupling the two
components of hidden channel \(j\). When every \(\mu_j^\ell>0\), this is equivalently
\(\|D_{\mu^\ell}^{-\alpha}\Delta_1^\ell\|_{\mathrm{op}}\le1\) and
\(\|\Delta_2^\ell D_{\mu^\ell}^{-\alpha}\|_{\mathrm{op}}\le1\) for all \(\ell\).
Thus BMuon applies Muon's spectral-norm constraint after diagonally
reweighting the hidden channels. If \(\mu_j^\ell=0\), the corresponding row of \(\Delta_1^\ell\) and column of \(\Delta_2^\ell\) are forced to be zero, making the connection to
channel-sparse updates explicit.

\paragraph{BMuon oracle.}
Muon maximizes the alignment \(\langle M,\Delta\rangle\) over a fixed constraint set, \(\|\Delta\|_{\mathrm{Muon}}\le1\). In BMuon, the admissible set for a fixed allocation is \(\mathcal A_\alpha(\mu)\), and the oracle chooses the allocation and direction jointly:
\begin{equation}
\label{eq:bmuon-oracle}
\max_{\mu\in\mathcal K_K}\ \max_{\Delta\in\mathcal A_\alpha(\mu)}
\langle M,\Delta\rangle.
\end{equation}
The budget is allocated to maximize the attainable alignment with the momentum; when \(M\) is the gradient rather than the momentum, the oracle jointly chooses \(\mu\) and \(\Delta\) to maximize the first-order decrease
in the loss. By spectral--nuclear norm duality, the oracle reduces to the allocation
problem
\begin{equation}
\label{eq:bmuon-allocation}
\max_{\mu\in\mathcal K_K}\Phi_M(\mu),
\qquad
\Phi_M(\mu)
:=
\sum_{\ell=1}^L
\left[
\|D_{\mu^\ell}^\alpha M_1^\ell\|_*
+
\|M_2^\ell D_{\mu^\ell}^\alpha\|_*
\right],
\end{equation}
whose maximum is attained since \(\mathcal K_K\) is compact and \(\Phi_M\) is
continuous. Given a maximizer \(\mu^\star\), an optimal direction is, for
every \(\ell\),
\begin{equation}
\label{eq:bmuon-update}
\Delta_1^{\ell,\star}
=
D_{\mu^{\ell,\star}}^\alpha
\Polar\!\left(D_{\mu^{\ell,\star}}^\alpha M_1^\ell\right),
\qquad
\Delta_2^{\ell,\star}
=
\Polar\!\left(M_2^\ell D_{\mu^{\ell,\star}}^\alpha\right)
D_{\mu^{\ell,\star}}^\alpha.
\end{equation}
The derivation is given in Appendix~\ref{app:oracle-derivation}. The parameters are then updated as \(W_i^\ell\leftarrow W_i^\ell-\eta\Delta_i^{\ell,\star}\).
We call this an oracle because it solves the allocation problem
in Eq.~(\ref{eq:bmuon-allocation}) exactly at every step; the practical approximation is introduced in Section~\ref{sec:practical}.

\paragraph{Allocation curvature.}
The exponent \(\alpha\) controls the curvature of the allocation objective
in Eq.~(\ref{eq:bmuon-allocation}).
\begin{theorem}[Curvature transition]
\label{thm:curvature}
For any momentum matrices, \(\Phi_M\) is concave in \(\mu\) for
\(0<\alpha\le\tfrac12\), convex in \(\mu\) for \(\alpha=1\), and in general
neither convex nor concave for \(\tfrac12<\alpha<1\).
\end{theorem}
For \(0<\alpha\le\tfrac12\), Eq.~(\ref{eq:bmuon-allocation}) maximizes a
concave function over the convex set \(\mathcal K_K\), and is therefore a
convex program; every local maximizer is global, and optimal
allocations may lie in the relative interior of \(\mathcal K_K\), hence may
be fractional. By contrast, at \(\alpha=1\) and integer \(K\), a convex
function is maximized over the polytope \(\mathcal K_K\), so a maximizer can be chosen at an extreme point of \(\mathcal K_K\), i.e., as a binary
allocation selecting exactly \(K\) channels. The proof and details are given
in Appendix~\ref{app:curvature}.

\subsection{Induced Optimization Geometry}
\label{sec:geometry}

Joint maximization over \(\mu\) and \(\Delta\) in
Eq.~(\ref{eq:bmuon-oracle}) is equivalent to maximizing
\(\langle M,\Delta\rangle\) over
\(\Delta\in\mathcal A_{\alpha,K}\), where
\(\mathcal A_{\alpha,K}:=\bigcup_{\mu\in\mathcal K_K}\mathcal A_\alpha(\mu)\).
Since a linear functional has the same maximum over a set and its convex hull,
the BMuon oracle maximizes \(\langle M,\Delta\rangle\) over
\(\Delta\in\mathcal B_{\alpha,K}\), where
\(\mathcal B_{\alpha,K}:=\operatorname{conv}(\mathcal A_{\alpha,K})\).

As shown in Appendix~\ref{app:norm}, \(\mathcal B_{\alpha,K}\) is compact,
convex, centrally symmetric, and has the origin in its interior. Its Minkowski
functional
\(\Omega_{\alpha,K}(\Delta):=\inf\{t>0:\Delta\in t\,\mathcal B_{\alpha,K}\}\)
is therefore a norm whose unit ball is \(\mathcal B_{\alpha,K}\). Thus, for fixed $\alpha, K$, BMuon is exactly
normalized steepest descent under \(\Omega_{\alpha,K}\):
\begin{equation}
\label{eq:bmuon-steepest-descent}
\Delta^\star
\in
\operatorname*{arg\,max}_{\Omega_{\alpha,K}(\Delta)\le1}
\langle M,\Delta\rangle.
\end{equation}
In comparison with Eq.~(\ref{eqn:muonup_org}), BMuon uses a modified norm constraint, hence a different geometry.

\subsection{From Muon to Sparse Geometries}
\label{sec:special-cases}

We prove that the family \(\Omega_{\alpha,K}\) interpolates between Muon's spectral geometry
and channel-sparse geometries. Proofs of these special cases are given in Appendix~\ref{app:special-cases}.

\paragraph{Recovering Muon.}
If \(K=N\), the constraints \(\mu_j^\ell\le1\) and
\(\sum_{\ell,j}\mu_j^\ell=N\) force \(\mu_j^\ell=1\) for every channel, so
\(\mathcal B_{\alpha,N}=\mathcal A_{\mathrm{Muon}}\) and
\(\Omega_{\alpha,N}=\|\cdot\|_{\mathrm{Muon}}\).
Muon is also recovered for any \(K\) by extending the family to
\(\alpha=0\) with \(D_\mu^0=I\).

\paragraph{A structured \(K\)-support geometry.}
Assume \(K\) is an integer and \(\alpha=1\).\\
Define
\(\operatorname{supp}_{\mathrm{ch}}(\Delta)
:=\{(\ell,j):\Delta_{1,j:}^\ell\neq0\ \text{or}\
\Delta_{2,:,j}^\ell\neq0\}\). The induced unit ball is
\[
\mathcal B_{1,K}
=
\operatorname{conv}\Big\{
\Delta:
|\operatorname{supp}_{\mathrm{ch}}(\Delta)|\le K,\ 
\|\Delta_1^\ell\|_{\mathrm{op}}\le1,\ 
\|\Delta_2^\ell\|_{\mathrm{op}}\le1,\ \forall\ell
\Big\}.
\]
Thus \(\Omega_{1,K}\) is closely related to \(k\)-support and group
\(k\)-support norms \citep{argyriou2012sparse,rao2017group}, but BMuon uses
coupled hidden-channel support and operator-norm constraints rather than
Euclidean ones. It is also related to spectral \(k\)-support norms
\citep{mcdonald2014spectral,mcdonald2016fitting}, but sparsity acts on hidden
channels rather than singular directions.

\paragraph{A coupled group-\(\ell_1\) geometry.}
At \(\alpha=K=1\),
\(\Omega_{1,1}(\Delta)
=
\sum_{\ell,j}
\max\{\|\Delta_{1,j:}^\ell\|_2,\|\Delta_{2,:,j}^\ell\|_2\}\).
This is an \(\ell_1\) aggregation over coupled hidden channels and therefore
promotes channel sparsity.


\section{Representation Cost and Intrinsic Width}
\label{sec:compactness}

Section~\ref{sec:budget-muon} showed that the BMuon oracle performs steepest descent under \(\Omega_{\alpha,K}\).
We now ask whether minimizing \(\Omega_{\alpha,K}\) among parameterizations that realize the same predictor favors representations with fewer active hidden channels.
This question is motivated by prior work showing that the norm underlying steepest descent can directly shape its implicit bias \citep{gunasekar2018characterizing,tsilivis2025flavors,fan2025implicit,jacobs2026never}.
Rather than establishing an implicit-bias result for BMuon, we characterize the exact parameterizations of a fixed predictor that minimize \(\Omega_{\alpha,K}\), following the representation-cost view of \citet{dai2021representation,jagadeesan2022inductive}.

We fix predictors \(X=(X_\ell)_{\ell=1}^L\), where each \(X_\ell\in\mathbb R^{p_\ell\times d_\ell}\) is nonzero, and denote \(r_\ell:=\operatorname{rank}(X_\ell)\).
We consider linear two-factor blocks \(W=((W_1^\ell,W_2^\ell))_{\ell=1}^L\), with \(W_1^\ell\in\mathbb R^{m_\ell\times d_\ell}\), \(W_2^\ell\in\mathbb R^{p_\ell\times m_\ell}\), and \(m_\ell\ge r_\ell\).
Each block realizes the matrix factorization \(W_2^\ell W_1^\ell\), corresponding to the setup in Eq.~(\ref{equation : setting MLP}) with the identity activation and no bias.
For any parameter norm \(\|\cdot\|\), define its representation cost
\[
\mathcal R_{\|\cdot\|}(X)
:=
\inf\left\{
\|W\|^2:
W_2^\ell W_1^\ell=X_\ell,\ \forall \ell
\right\}.
\]
We write \(\mathcal R_{\alpha,K}:=\mathcal R_{\Omega_{\alpha,K}}\) for BMuon and \(\mathcal R_{\mathrm{Muon}}:=\mathcal R_{\|\cdot\|_{\mathrm{Muon}}}\) for Muon.
For a factorization \(W\), let \(a_\ell(W)\) denote the active width of block \(\ell\), i.e., the number of hidden channels \(j\) for which \(W_{1,j:}^\ell\neq0\) or \(W_{2,:,j}^\ell\neq0\).
Every exact factorization of \(X\) satisfies \(a_\ell(W)\ge r_\ell\) for all \(\ell\), so \(r_\ell\) is the minimum possible active width of block \(\ell\).
Let \(\sigma_{\ell,1}\ge\cdots\ge\sigma_{\ell,r_\ell}>0\) denote the nonzero singular values of \(X_\ell\), let \(\sigma_\star:=\max_{\ell,i}\sigma_{\ell,i}\), and define the \(\alpha\)-effective spectral dimension \(d_{\mathrm{eff}}^{(\alpha)}(X):=\sum_{\ell=1}^L\sum_{i=1}^{r_\ell}\left(\sigma_{\ell,i}/\sigma_\star\right)^{1/(2\alpha)}\).

\begin{theorem}[Representation cost and intrinsic-width recovery]
\label{thm:intrinsic-width}
Let \(0<\alpha\le\tfrac12\) and \(0<K\le N\).
Both representation-cost infima are attained, and
\[
\mathcal R_{\mathrm{Muon}}(X)=\sigma_\star,
\qquad
\mathcal R_{\alpha,K}(X)
=
\sigma_\star
\max\bigl\{
1,\,
\bigl(d_{\mathrm{eff}}^{(\alpha)}(X)/K\bigr)^{2\alpha}
\bigr\}.
\]
Moreover, Muon admits a minimizer \(W^\star\) with \(a_\ell(W^\star)=m_\ell\) for every \(\ell\), whereas if \(K\le d_{\mathrm{eff}}^{(\alpha)}(X)\), every BMuon minimizer \(W^\star\) satisfies \(a_\ell(W^\star)=r_\ell\) for all \(\ell\).
\end{theorem}
Proof is given in Appendix~\ref{app:intrinsic-width}.

\paragraph{Interpretation.} Muon's geometric cost is insensitive to hidden-channel redundancy, since its minimum can be attained at full hidden width.
In contrast, BMuon is sensitive to redundant width: when \(K\le d_{\mathrm{eff}}^{(\alpha)}(X)\), every minimum-cost factorization is width-minimal, so any redundancy incurs additional geometric cost.
The budget raises the representation cost only when \(K<d_{\mathrm{eff}}^{(\alpha)}(X)\), in which case it is \(\sigma_\star\bigl(d_{\mathrm{eff}}^{(\alpha)}(X)/K\bigr)^{2\alpha}\); otherwise, it is \(\sigma_\star\).
Since \(d_{\mathrm{eff}}^{(\alpha)}(X)\) is nondecreasing in \(\alpha\), increasing \(\alpha\) can only enlarge the budget regime in which width-minimality is guaranteed.

\paragraph{Effective spectral dimension.} The \(\alpha\)-effective spectral dimension recovers familiar spectral quantities.
Let \(X_{\mathrm{bd}}:=\operatorname{diag}(X_1,\ldots,X_L)\).
For \(\alpha=\tfrac14\), \(d_{\mathrm{eff}}^{(\alpha)}(X)=\|X_{\mathrm{bd}}\|_F^2/\|X_{\mathrm{bd}}\|_{\mathrm{op}}^2\), the stable rank of \(X_{\mathrm{bd}}\), while for \(\alpha=\tfrac12\), \(d_{\mathrm{eff}}^{(\alpha)}(X)=\|X_{\mathrm{bd}}\|_*/\|X_{\mathrm{bd}}\|_{\mathrm{op}}\).

\section{Practical BMuon}
\label{sec:practical}

Exactly solving the allocation problem in Eq.~(\ref{eq:bmuon-allocation}) at every optimizer step is expensive. We instead perform a minorize--maximize (MM) update \citep{hunter2004tutorial}, replacing the exact allocation problem by the maximization of a surrogate. For fixed momenta $M$ and exponent $\alpha$, given an allocation $\bar\mu$, let
$\bar P_1^\ell=\Polar(D_{\bar\mu^\ell}^\alpha M_1^\ell)$ and
$\bar P_2^\ell=\Polar(M_2^\ell D_{\bar\mu^\ell}^\alpha)$.
As shown in Appendix~\ref{app:mm-solver}, we have the tight minorization
{
\setlength{\belowdisplayskip}{-1pt}
\begin{equation}
\Phi_M(\mu)
\ge
Q(\mu\mid\bar\mu)
:=
\sum_{\ell,j}
a_j^\ell(\mu_j^\ell)^\alpha,
\qquad
a_j^\ell
=
\langle M_{1,j:}^\ell,\bar P_{1,j:}^\ell\rangle
+
\langle M_{2,:,j}^\ell,\bar P_{2,:,j}^\ell\rangle,
\label{eq:mm-surrogate}
\end{equation}
}
\\
\noindent
with equality at $\mu=\bar\mu$. At the boundary $\bar\mu_j^\ell=0$, row $j$ of
$D_{\bar\mu^\ell}^\alpha M_1^\ell$ and column $j$ of
$M_2^\ell D_{\bar\mu^\ell}^\alpha$ vanish, giving $a_j^\ell=0$, so the
surrogate no longer depends on $\mu_j^\ell$. We therefore impose a small floor $\mu_j^\ell\ge\varepsilon$ in practice and maximize the surrogate over
$\mathcal K_{K,\varepsilon}
:=\{\mu\in[\varepsilon,1]^N:\sum_{\ell,j}\mu_j^\ell=K\}$,
where $0<\varepsilon<K/N$. Proposition~\ref{prop:mm-floor} bounds the gap between the floored and original
oracle objectives. The scores $a_j^\ell$ are nonnegative in general and, if for every $(\ell,j)$,
$M_{1,j:}^\ell\neq0$ or $M_{2,:,j}^\ell\neq0$, they are strictly positive
for every $\bar\mu\in\mathcal K_{K,\varepsilon}$. Hence, for $0<\alpha<1$, the surrogate is strictly concave and has a unique
maximizer. The KKT conditions, derived in Appendix~\ref{app:mm-kkt}, give
\begin{equation}
(\mu_j^\ell)^+
=
\operatorname{clip}_{[\varepsilon,1]}
\left[
\left(
\frac{\alpha a_j^\ell}{\lambda_\mu}
\right)^{1/(1-\alpha)}
\right],
\label{eq:mm-closed-form}
\end{equation}
where
$\operatorname{clip}_{[\varepsilon,1]}(x)
:=\min\{1,\max\{\varepsilon,x\}\}$,
and $\lambda_\mu>0$ is found by bisection from
$\sum_{\ell,j}(\mu_j^\ell)^+=K$. Thus, the MM allocation step introduces no additional step-size
hyperparameter.

For fixed momenta, each MM step is monotone since
$\Phi_M(\mu^+)\ge Q(\mu^+\mid\bar\mu)
\ge Q(\bar\mu\mid\bar\mu)=\Phi_M(\bar\mu)$.
Moreover, for $0<\alpha\le\tfrac12$, repeated MM steps converge in
objective value to the global optimum over $\mathcal K_{K,\varepsilon}$, and
every accumulation point of the allocation sequence is globally optimal. The proof is given in Appendix~\ref{app:mm-convergence}.

Algorithm~\ref{alg:bmuon} summarizes the resulting update. When $\mu$ is updated, the first polar evaluation constructs the MM surrogate at the current allocation. After the MM step, the polar factors are recomputed using the new allocation to form the parameter update.
\begin{algorithm}[h!]
\caption{Budget Muon}
\label{alg:bmuon}
\small
\begin{algorithmic}[1]
\Require $K\in(0,N)$, $\alpha\in(0,1)$,
$0<\varepsilon<K/N$, allocation period $\Delta T_\mu$, learning rate $\eta_t$
\State $\mu_j^\ell\gets K/N$
\Comment{$\forall\,(\ell,j)$}
\For{$t=0,1,\ldots$}
    \State Update momenta $\{M_1^\ell,M_2^\ell\}_{\ell=1}^L$ as in Muon
    \If{$t\bmod\Delta T_\mu=0$}
        \State $\widetilde P_1^\ell
        \gets\Polar(D_{\mu^\ell}^\alpha M_1^\ell)$,
        $\quad
        \widetilde P_2^\ell
        \gets\Polar(M_2^\ell D_{\mu^\ell}^\alpha)$
        \Comment{$\forall\,\ell$}
        \State $a_j^\ell\gets
        \langle M_{1,j:}^\ell,\widetilde P_{1,j:}^\ell\rangle+
        \langle M_{2,:,j}^\ell,\widetilde P_{2,:,j}^\ell\rangle$
        \Comment{$\forall\,(\ell,j)$}
        \State Set
        $\displaystyle
        \mu_j^\ell\gets
        \operatorname{clip}_{[\varepsilon,1]}
        \!\left[
        \left({\alpha a_j^\ell}/{\lambda_\mu}\right)^{1/(1-\alpha)}
        \right]$ and choose $\lambda_\mu$ so that $\sum_{\ell,j}\mu_j^\ell=K$
        \Comment{$\forall\,(\ell,j)$}
    \EndIf
    \State $\Delta_1^\ell
    \gets D_{\mu^\ell}^\alpha \Polar(D_{\mu^\ell}^\alpha M_1^\ell)$,
    $\quad
    \Delta_2^\ell
    \gets \Polar(M_2^\ell D_{\mu^\ell}^\alpha) D_{\mu^\ell}^\alpha$
    \Comment{$\forall\,\ell$}
    \State $W_1^\ell\gets W_1^\ell-\eta_t\Delta_1^\ell$,
    $\quad
    W_2^\ell\gets W_2^\ell-\eta_t\Delta_2^\ell$
    \Comment{$\forall\,\ell$}
\EndFor
\end{algorithmic}
\end{algorithm}

The budget $K$ can be scheduled during training for example by annealing it from $N$ to a target value.
Parameters outside the MLP blocks are updated as in the corresponding
Muon baseline.

The special cases in Section~\ref{sec:special-cases} place Muon at the
$\alpha=0$ endpoint and channel-sparse geometries at $\alpha=1$,
suggesting that larger $\alpha$ corresponds to more sparsity-favoring
geometries. Moreover, away from the clipping boundaries,
Eq.~(\ref{eq:mm-closed-form}) implies
$(\mu_i)^+/(\mu_j)^+=(a_i/a_j)^{1/(1-\alpha)}$, so larger $\alpha$
amplifies differences between channel scores and favors more concentrated
allocations. At the same time, $\alpha=1/2$ is the largest exponent for
which the allocation objective is guaranteed to remain concave by
Theorem~\ref{thm:curvature} and for which repeated MM steps are guaranteed to
converge in objective value by Theorem~\ref{thm:mm-convergence}.
We therefore use $\alpha=1/2$ throughout our experiments.

In practice, we approximate each polar factor by a finite Newton--Schulz
iteration, as in standard Muon, and update the allocation by a single MM step
every $\Delta T_\mu$ optimizer steps, keeping it fixed between these updates. We assess these practical approximations empirically in
Appendix~\ref{app:oracle_approximation}, using $\Delta T_\mu=5$ and a certified
reference allocation obtained by repeatedly applying the exact-polar MM update
until its relative objective loss from the exact optimum is certified to be
below a small threshold.
In these diagnostics, the relative loss in the allocation objective incurred by
the practical one-step Newton--Schulz-based update is certified to remain on
the order of $10^{-4}$ to $10^{-3}$, while keeping the allocation fixed for
the four intervening optimizer steps keeps the corresponding certified deficit
on the order of $10^{-3}$.
In the reported comparison, replacing the practical allocation
update every five optimizer steps by the certified reference allocation leaves
test accuracy and post-training compressibility unchanged within the observed
seed-to-seed variability.

Relative to Muon, BMuon adds only channelwise scalings on the budgeted MLP
matrices and, every $\Delta T_\mu$ steps, one extra polar evaluation per
such matrix; all other parameters are updated as in Muon.
In our ImageNet-1K and 1B-parameter LLaMA experiments in
Sections~\ref{sec:gradual-structured-pruning} and~\ref{sec:llama}, respectively,
we observe no measurable increase in overall training time relative to Muon
(Appendix~\ref{app:overhead}).

\section{Experiments}
\label{sec:experiments}

\subsection{Post-Training Compression}
\label{sec:post-training-compression}

\begin{wrapfigure}{r}{0.47\columnwidth}
    \vspace{-1.5em}
    \centering
    \includegraphics[width=\linewidth]{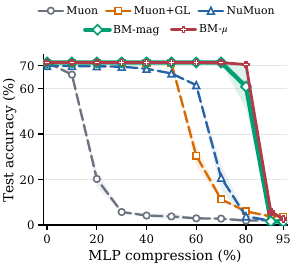}
    \caption{\textbf{CIFAR-100 ViT post-training MLP compression.}
    Mean test accuracy $\pm$ sample std. over three seeds after compression without
    recovery. Muon, Muon+GL, and BM-mag use global MLP-channel pruning ranked by
    magnitude; BM-$\mu$ uses $\mu$ for ranking, while NuMuon uses
    parameter-matched truncated-SVD compression of the MLP weights.}
    \label{fig:cifar100-posttraining}
    \vspace{-2em}
\end{wrapfigure}
\paragraph{Setup.}
We first test whether BMuon produces models that can be compressed directly
after training without post-compression fine-tuning. We train a small Vision
Transformer on CIFAR-100 for 200 epochs over three seeds. We compare Muon, Muon with a proximal Group-Lasso penalty on coupled MLP
channels (Muon+GL)~\citep{yuan2006model,wen2016learning},
NuMuon~\citep{dolatabadi2026numuon}, and BMuon. NuMuon is particularly
relevant as a direct comparison: it modifies the Muon update to improve
compressibility, but targets low-rank structure through the singular spectrum,
whereas BMuon targets structured MLP-channel compression. BMuon uses
$\alpha=1/2$ and a budget annealed to $20\%$ of the total MLP-channel mass. Full experimental details in Appendix~\ref{app:cifar100-posttraining}.

\paragraph{Post-training compression.}
We restrict the compression to the ViT MLP blocks. Muon, Muon+GL, and BMuon are compressed by
globally pruning hidden channels across all blocks. For magnitude-based pruning, channel $j$ in block $\ell$ is ranked by $s^{\ell}_j=\sqrt{\|W^{\ell}_{1,j:}\|_2^2+\|W^{\ell}_{2,:,j}\|_2^2}$.
BMuon is pruned using both this score (BM-mag) and its final learned allocation $\mu$ (BM-$\mu$). NuMuon instead compresses each MLP weight matrix independently using truncated SVD, with ranks chosen to approximately match the parameter count of the corresponding channel-pruned model.

\paragraph{Results.}
Figure~\ref{fig:cifar100-posttraining} shows the post-training fragility that
motivates this work: Muon falls to $5.72\%$ test accuracy after only $30\%$
MLP-channel pruning.
Muon+GL remains robust through $50\%$ pruning, while NuMuon degrades more
gradually under low-rank compression.
Under magnitude-based pruning, BMuon is substantially more
compressible than the structured baselines while matching Muon's dense
accuracy: BMuon reaches $71.46\%$, versus $71.49\%$ for Muon, $71.25\%$ for
Muon+GL, and $69.95\%$ for NuMuon.
Using $\mu$ for ranking provides a further gain: at $80\%$ pruning,
BM-$\mu$ retains $70.57\%$ accuracy, compared with $60.84\%$ for BM-mag,
while Muon, Muon+GL, and NuMuon have fallen to $1.93\%$, $6.01\%$, and
$3.84\%$, respectively.
The advantage of BM-mag under the shared magnitude criterion isolates the compressibility benefit of BMuon training, while the further gain from BM-$\mu$ indicates that the learned allocation also carries useful channel-importance information.

\subsection{Gradual Structured Pruning}
\label{sec:gradual-structured-pruning}
We next ask whether BMuon's compressibility advantage persists when channel pruning is interleaved with training and recovery~\citep{aketi2020gradual}, and whether its learned allocation identifies useful architectures when retrained from scratch.

\paragraph{Setup.}
We evaluate gradual MLP-channel pruning during training from scratch on a ViT for CIFAR-100 and a ViT-S/16 for ImageNet-1K, targeting final pruning levels of $90\%$ and $95\%$. All results are reported over three seeds. Each run starts dense and alternates between training and pruning. After each training phase, we rank all surviving MLP channels globally, permanently remove the lowest-ranked ones to reach the next sparsity target, and resume training from the surviving weights. We repeat this train--prune cycle until the final pruning level is reached, and then continue training for recovery. We compare three methods. \emph{Muon-mag} trains with Muon and ranks channels by the coupled weight-magnitude score introduced in Section~\ref{sec:post-training-compression}. \emph{BMuon-mag} trains the MLPs with BMuon ($\alpha=1/2$) and uses the same magnitude ranking. \emph{BMuon-$\mu$} uses the same BMuon training but ranks channels by the learned allocation $\mu$. During each BMuon training phase, the global budget is annealed to the next target channel count over the first half of the phase and held fixed for the second half.

As an additional control, we test whether the architecture selected by BMuon-$\mu$ is useful on its own. We keep only its final per-block widths, initialize a new model with those widths, and train it from scratch with standard Muon for the same training horizon. As a baseline, we also train a uniform-width model from scratch with nearly the same total number of MLP channels. Full details are given in Appendix~\ref{app:gradual-pruning-details}.

\begin{table}[h!]
\centering
\caption{\textbf{Gradual global MLP-channel pruning.}
Final Top-1 accuracy (\%), reported as mean $\pm$ sample std. over three seeds. The upper block reports the final pruned models after recovery; the lower block trains the architectures selected by BMuon-$\mu$ from scratch with Muon and compares their accuracy with uniform-width models having nearly the same total number of MLP channels. Details in Appendix~\ref{app:gradual-pruning-details}.}
\label{tab:gradual-pruning}
\small
\setlength{\tabcolsep}{7pt}
\renewcommand{\arraystretch}{1.12}

\begin{tabular}{lcccc}
\toprule
& \multicolumn{2}{c}{\textbf{CIFAR-100}}
& \multicolumn{2}{c}{\textbf{ImageNet-1K}} \\
\cmidrule(lr){2-3} \cmidrule(lr){4-5}
\textbf{Method}
& \textbf{90\%}
& \textbf{95\%}
& \textbf{90\%}
& \textbf{95\%} \\
\midrule

\multicolumn{5}{l}{\textit{Gradual pruning}} \\[-2pt]
Muon-mag
& $63.00 \pm 0.28$
& $59.70 \pm 0.71$
& $65.26 \pm 0.19$
& $61.24 \pm 0.47$ \\

BMuon-mag
& $64.78 \pm 0.89$
& $62.88 \pm 0.22$
& $68.71 \pm 0.11$
& $65.77 \pm 0.24$ \\

BMuon-$\mu$
& $\mathbf{67.24 \pm 0.22}$
& $\mathbf{63.91 \pm 0.07}$
& $\mathbf{69.06 \pm 0.19}$
& $\mathbf{67.04 \pm 0.02}$ \\

\addlinespace[2pt]
\midrule
\multicolumn{5}{l}{\textit{Architectures trained from scratch}} \\[-2pt]

Uniform
& $64.85 \pm 0.54$
& $63.62 \pm 0.36$
& $69.18 \pm 0.04$
& $66.73 \pm 0.26$ \\

BMuon-$\mu$ architecture
& $\mathbf{67.00 \pm 0.54}$
& $\mathbf{64.71 \pm 0.24}$
& $\mathbf{69.52 \pm 0.17}$
& $\mathbf{67.25 \pm 0.14}$ \\

\bottomrule
\end{tabular}
\end{table}
\paragraph{Gradual pruning performance.}
Table~\ref{tab:gradual-pruning} shows that BMuon consistently improves over Muon under gradual pruning. Using the same magnitude-based pruning criterion, BMuon-mag exceeds Muon-mag by $1.78$ and $3.18$ percentage points on CIFAR-100 at $90\%$ and $95\%$ pruning, and by $3.45$ and $4.53$ points on ImageNet-1K. Using $\mu$ as the pruning score then provides an additional gain of $2.46$ and $1.03$ points on CIFAR-100, and $0.35$ and $1.27$ points on ImageNet-1K.

\paragraph{Architecture discovery.}
When trained from scratch, the architectures selected by BMuon-$\mu$ outperform
the nearly channel-matched uniform models in all four settings.
Because both the discovered and uniform architectures are initialized from
scratch and trained with standard Muon, this comparison isolates the effect of
the architecture's width profile.
The results therefore provide evidence that BMuon-$\mu$ identifies useful
non-uniform architectures.


\subsection{Structured Compressibility at Language-Model Scale}
\label{sec:llama}

The BMuon construction on MLP hidden channels extends naturally to gated MLPs such as SwiGLU, as used in LLaMA-style models.
The same allocation acts on the rows of the gate and up projections and on the columns of the down projection; see Appendix~\ref{app:gated-mlps} for details.

\begin{wrapfigure}{r}{0.47\columnwidth}
    \vspace{-1.4em}
    \centering
    \includegraphics[width=\linewidth]{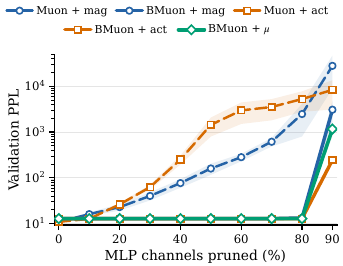}
    \caption{\textbf{Post-training MLP-channel pruning at language-model scale.}
    Validation perplexity after global structured pruning of a 1B-parameter LLaMA-style model trained on FineWeb. Mean $\pm$ std. over three seeds.}
    \label{fig:llama-pruning}
    \vspace{-1.0em}
\end{wrapfigure}
To test whether BMuon's structured-compressibility advantage extends to large-scale language modeling, we train a 1B-parameter LLaMA-style model on 13.1B FineWeb tokens over three seeds.
BMuon uses $\alpha=1/2$ and $K=0.2N$.
After training, we globally prune SwiGLU hidden channels without recovery, using weight magnitude and an activation-aware score combining root-mean-square activation and output-weight norm; for BMuon, we also rank channels by $\mu$.
The pruning curves are shown in Figure~\ref{fig:llama-pruning}; full details are given in Appendix~\ref{app:llama-details}.

On the unpruned models, Muon achieves lower perplexity ($10.92$ versus $12.84$ for BMuon), but its performance degrades rapidly under pruning.
In contrast, BMuon preserves its unpruned performance up to $80\%$ pruning, reaching $12.86$ with activation ranking and $12.85$ with $\mu$ ranking.
BMuon substantially outperforms Muon when both use the same magnitude or activation criterion at moderate and high pruning levels, while $\mu$ again provides an effective pruning score.
Overall, BMuon trades some unpruned quality for substantially greater robustness to structured channel pruning.

\section{Conclusion}
\label{sec:conclusion}

Muon-trained Transformers can degrade sharply under post-training MLP-channel pruning, motivating the question of whether Muon's spectral geometry can be modified to favor structured compressibility.
BMuon addresses this question by reweighting hidden channels within Muon's spectral-norm constraints and making them compete for a global budget.
We show that the resulting norm family recovers Muon and connects it to channel-sparse geometries.
In linear two-factor models, Muon admits minimum-cost factorizations with redundant hidden width, whereas under a sufficiently scarce BMuon budget every minimum-cost factorization is width-minimal.
For practical training, we derive a tractable minorization--maximization approximation to the global allocation oracle.
Across vision and language models, BMuon is substantially more robust than Muon to magnitude-based MLP-channel pruning.
The learned allocation also provides an effective pruning score and identifies useful non-uniform MLP architectures when retrained from scratch.

Several questions remain open.
Our theoretical results characterize representation cost in linear two-factor models; studying the implicit bias of BMuon training, particularly beyond the linear setting, is an important direction for future work.
Another direction is to design a practical BMuon algorithm with provable approximation guarantees to the oracle.
Our experiments focus on MLP hidden channels, and extending the global budget to other structured components of Transformers remains to be explored.
Finally, BMuon can trade dense-model quality for greater pruning robustness, as observed in our language-model experiments; characterizing this tradeoff at larger model scales is a natural direction for future work.

\paragraph{Acknowledgments.}
YI gratefully acknowledges the support of Alberta Innovates (ALLRP 600038-24), the Natural Sciences and Engineering Research Council of Canada (NSERC) (RGPIN-2022-03120, DGECR-2022-00358), and NSERC/Agence Nationale de la Recherche (ANR) (ALLRP 602719-24). This project was undertaken thanks to funding from IVADO and the Canada First Research Excellence Fund. This research was enabled in part by support provided by the Digital Research Alliance of Canada (\url{alliancecan.ca}).

RB and TJ gratefully acknowledge funding 
from the European Research Council (ERC) under the Horizon Europe Framework Programme
(HORIZON) for proposal number 101116395 SPARSE-ML.

U.S. is partially supported by the French government under the
management of Agence Nationale de la Recherche as part of the ``Investissements d’avenir'' program, reference ANR-19-P3IA-0001 (PRAIRIE 3IA Institute) and the grant GHOST ANR-25-CE23-4985, and
by the European Research Council Starting Grant DYNASTY – 101039676.

\bibliography{references}

\clearpage
\appendix
\startcontents[appendix]

\begin{center}
{\Large\bfseries
BMuon: Budgeting Muon's Spectral Geometry for Structured Compressibility\\
--- Appendix ---}
\end{center}

\vspace{1em}

\section*{Contents}
\printcontents[appendix]{}{0}{\setcounter{tocdepth}{2}}

\clearpage

\section{Muon Algorithm}
\label{app:muon-algo}

Muon applies momentum followed by approximate matrix orthogonalization to
eligible matrix parameters~\citep{jordan2024muon}.
Algorithm~\ref{alg:muon} gives the update for one matrix parameter using the
Nesterov-style momentum variant.

\begin{algorithm}[h]
\caption{Muon}
\label{alg:muon}
\small
\begin{algorithmic}[1]
\Require Learning rates \(\{\eta_t\}\), momentum \(\beta\)
\State Initialize \(B_0\gets 0\)
\For{\(t=1,2,\ldots\)}
    \State \(G_t\gets\nabla_W\mathcal L_t(W_{t-1})\)
    \State \(B_t\gets\beta B_{t-1}+G_t\)
    \State \(M_t\gets G_t+\beta B_t\)
    \State \(O_t\gets\operatorname{NewtonSchulz5}(M_t)\)
    \State \(W_t\gets W_{t-1}-\eta_t O_t\)
\EndFor
\end{algorithmic}
\end{algorithm}

The non-Nesterov variant instead uses \(M_t=B_t\). Throughout the paper,
\emph{momentum} refers to \(M_t\), the matrix that Muon orthogonalizes; our
geometric results are stated for arbitrary \(M_t\) and therefore do not depend
on the momentum variant used to produce it. We specify the variant used in
each experimental setup.

Practical Muon implementations may additionally rescale each matrix update. The public Muon implementation uses the aspect-ratio correction
\(\sqrt{\max\{1,m/n\}}\) for \(W\in\mathbb{R}^{m\times n}\)
\citep{jordan2024muon}, while the large-scale RMS-matching convention of
\citet{liu2025muon} uses \(0.2\sqrt{\max\{m,n\}}\).
Whenever such a rescaling is used, we state it in the corresponding
experimental setup; otherwise, no rescaling is applied. We omit it from
the geometric analysis.

Parameters not optimized by Muon, including biases, normalization parameters,
embeddings, and classifier or output heads, are optimized separately with
AdamW. With decoupled weight decay, the final update becomes
\[
W_t=(1-\eta_t\lambda)W_{t-1}-\eta_t\gamma O_t.
\]

\paragraph{Newton--Schulz orthogonalization.}
For \(M\in\mathbb{R}^{m\times n}\), let
\(\widetilde M=M\) if \(m\le n\) and \(\widetilde M=M^\top\) otherwise, and set
\(X_0=\widetilde M/(\|\widetilde M\|_F+10^{-7})\). Muon then applies five
iterations
\[
X_{k+1}
=
aX_k+b(X_kX_k^\top)X_k+c(X_kX_k^\top)^2X_k,
\qquad k=0,\ldots,4,
\]
with \((a,b,c)=(3.4445,-4.7750,2.0315)\)~\citep{jordan2024muon}.
We define
\(\operatorname{NewtonSchulz5}(M)=X_5\) if \(m\le n\), and \(\operatorname{NewtonSchulz5}(M)=X_5^\top\) otherwise. This finite iteration
approximately orthogonalizes \(M\); the main text uses the exact polar direction for the theoretical analysis.

\newpage

\section{BMuon Oracle and Geometry}
\label{app:geometry}

This appendix provides the derivations and proofs underlying
Section~\ref{sec:budget-muon}.
Appendix~\ref{app:oracle-derivation} derives the BMuon allocation problem and
update in Eqs.~(\ref{eq:bmuon-allocation})--(\ref{eq:bmuon-update}) from the
joint oracle in Eq.~(\ref{eq:bmuon-oracle}).
Appendix~\ref{app:gated-mlps} extends the construction to gated MLPs such as
SwiGLU.
Appendix~\ref{app:curvature} proves the curvature transition of
Theorem~\ref{thm:curvature}.
Appendix~\ref{app:norm} establishes the induced norm geometry
$\Omega_{\alpha,K}$ and its characterization for
$0<\alpha\le\tfrac12$.
Finally, Appendix~\ref{app:special-cases} proves the geometric special cases
stated in Section~\ref{sec:special-cases}.

\paragraph{Notation and setup.}
We consider $L$ coupled MLP blocks with
$W_1^\ell\in\mathbb R^{m_\ell\times d_\ell}$ and
$W_2^\ell\in\mathbb R^{p_\ell\times m_\ell}$, and write
$N:=\sum_{\ell=1}^L m_\ell$ for the total number of hidden channels.
Their Muon momentum matrices are denoted
$M_1^\ell$ and $M_2^\ell$, with corresponding update directions
$\Delta_1^\ell$ and $\Delta_2^\ell$.
We write
$M:=((M_1^\ell,M_2^\ell))_{\ell=1}^L$ and
$\Delta:=((\Delta_1^\ell,\Delta_2^\ell))_{\ell=1}^L$, with
$
\langle M,\Delta\rangle
:=
\sum_{\ell=1}^L
\left(
\langle M_1^\ell,\Delta_1^\ell\rangle
+
\langle M_2^\ell,\Delta_2^\ell\rangle
\right).
$
For a global budget $K\in(0,N]$, BMuon uses allocations
$\mu=(\mu^1,\ldots,\mu^L)$ in
$
\mathcal K_K
:=
\left\{
\mu\in[0,1]^N:
\sum_{\ell,j}\mu_j^\ell=K
\right\}.
$
For $\mu\in\mathcal K_K$, let
$D_{\mu^\ell}:=\operatorname{diag}(\mu^\ell)$.
For fixed $\alpha\in(0,1]$, the admissible BMuon directions are
\[
\mathcal A_\alpha(\mu)
:=
\bigl\{
\Delta:\ \exists\,\{P_1^\ell,P_2^\ell\}_{\ell=1}^L,\ 
\Delta_1^\ell=D_{\mu^\ell}^\alpha P_1^\ell,\ 
\Delta_2^\ell=P_2^\ell D_{\mu^\ell}^\alpha,\ 
\|P_1^\ell\|_{\mathrm{op}}\le1,\ 
\|P_2^\ell\|_{\mathrm{op}}\le1,\ \forall\ell
\bigr\}.
\]

The BMuon oracle solves
\[
\max_{\mu\in\mathcal K_K}
\max_{\Delta\in\mathcal A_\alpha(\mu)}
\langle M,\Delta\rangle,
\]
and its allocation objective is
\[
\Phi_M(\mu)
=
\sum_{\ell=1}^L
\left[
\|D_{\mu^\ell}^\alpha M_1^\ell\|_*
+
\|M_2^\ell D_{\mu^\ell}^\alpha\|_*
\right].
\]

\subsection{Derivation of the BMuon Oracle}
\label{app:oracle-derivation}

We derive Eqs.~(\ref{eq:bmuon-allocation})--(\ref{eq:bmuon-update}) from
Eq.~(\ref{eq:bmuon-oracle}). Fix \(\mu\in\mathcal K_K\). By definition of
\(\mathcal A_\alpha(\mu)\),
\[
\begin{aligned}
\max_{\Delta\in\mathcal A_\alpha(\mu)}
\langle M,\Delta\rangle
&=
\max_{\substack{
\|P_1^\ell\|_{\mathrm{op}}\le1\\
\|P_2^\ell\|_{\mathrm{op}}\le1,\ \forall\ell
}}
\sum_{\ell=1}^L
\Big[
\langle M_1^\ell,D_{\mu^\ell}^\alpha P_1^\ell\rangle
+
\langle M_2^\ell,P_2^\ell D_{\mu^\ell}^\alpha\rangle
\Big]
\\
&=
\sum_{\ell=1}^L
\Bigg[
\max_{\|P_1^\ell\|_{\mathrm{op}}\le1}
\langle M_1^\ell,D_{\mu^\ell}^\alpha P_1^\ell\rangle
+
\max_{\|P_2^\ell\|_{\mathrm{op}}\le1}
\langle M_2^\ell,P_2^\ell D_{\mu^\ell}^\alpha\rangle
\Bigg],
\end{aligned}
\]
where the second equality follows because the objective is a sum over matrices
and the constraints are independent across matrices.

Since \(D_{\mu^\ell}^\alpha\) is diagonal,
\(\langle M_1^\ell,D_{\mu^\ell}^\alpha P_1^\ell\rangle
=
\langle D_{\mu^\ell}^\alpha M_1^\ell,P_1^\ell\rangle\) and
\(\langle M_2^\ell,P_2^\ell D_{\mu^\ell}^\alpha\rangle
=
\langle M_2^\ell D_{\mu^\ell}^\alpha,P_2^\ell\rangle\).
Spectral--nuclear norm duality therefore gives
\[
\max_{\Delta\in\mathcal A_\alpha(\mu)}
\langle M,\Delta\rangle
=
\sum_{\ell=1}^L
\left[
\|D_{\mu^\ell}^\alpha M_1^\ell\|_*
+
\|M_2^\ell D_{\mu^\ell}^\alpha\|_*
\right]
=
\Phi_M(\mu).
\]
Maximizing this value over \(\mu\in\mathcal K_K\) yields
Eq.~(\ref{eq:bmuon-allocation}).

Moreover, \(\Phi_M\) is continuous on \(\mathcal K_K\): for
\(\alpha>0\), the map \(\mu\mapsto D_\mu^\alpha\) is continuous on
\([0,1]^N\), and the nuclear norm is continuous. Since
\(\mathcal K_K\) is compact, a maximizer \(\mu^\star\) exists.

Fix such a maximizer. By spectral--nuclear duality,
\(\Polar(A)\in\operatorname*{arg\,max}_{\|P\|_{\mathrm{op}}\le1}
\langle A,P\rangle\) for any matrix \(A\). Therefore, for every \(\ell\), we
can choose
\[
P_1^{\ell,\star}
=
\Polar\!\left(D_{\mu^{\ell,\star}}^\alpha M_1^\ell\right),
\qquad
P_2^{\ell,\star}
=
\Polar\!\left(M_2^\ell D_{\mu^{\ell,\star}}^\alpha\right).
\]
Returning to the update variables through the parameterization
\(\Delta_1^\ell=D_{\mu^\ell}^\alpha P_1^\ell\) and
\(\Delta_2^\ell=P_2^\ell D_{\mu^\ell}^\alpha\), substituting
\(\mu=\mu^\star\) and the maximizing
\(P_1^{\ell,\star},P_2^{\ell,\star}\) gives, for every \(\ell\),
\[
\Delta_1^{\ell,\star}
=
D_{\mu^{\ell,\star}}^\alpha
\Polar\!\left(D_{\mu^{\ell,\star}}^\alpha M_1^\ell\right),
\qquad
\Delta_2^{\ell,\star}
=
\Polar\!\left(M_2^\ell D_{\mu^{\ell,\star}}^\alpha\right)
D_{\mu^{\ell,\star}}^\alpha,
\]
which is Eq.~(\ref{eq:bmuon-update}). This \(\Delta^\star\) attains
\(\max_{\Delta\in\mathcal A_\alpha(\mu^\star)}
\langle M,\Delta\rangle=\Phi_M(\mu^\star)\). Since \(\mu^\star\) maximizes
\(\Phi_M\) over \(\mathcal K_K\), the pair
\((\mu^\star,\Delta^\star)\) solves the joint oracle in
Eq.~(\ref{eq:bmuon-oracle}).

\subsection{Extension to Gated MLPs}
\label{app:gated-mlps}
A SwiGLU block, omitting biases, has the form
\[
x\mapsto
W_{\mathrm{down}}^\ell
\Big(
\operatorname{SiLU}(W_{\mathrm{gate}}^\ell x)
\odot
W_{\mathrm{up}}^\ell x
\Big),
\]
where \(\odot\) denotes elementwise multiplication. Hidden channel \(j\) therefore couples row \(j\) of
\(W_{\mathrm{gate}}^\ell\), row \(j\) of \(W_{\mathrm{up}}^\ell\), and
column \(j\) of \(W_{\mathrm{down}}^\ell\). BMuon assigns the same allocation
\(\mu_j^\ell\) to these three components. Let \(M_{\mathrm{gate}}^\ell\), \(M_{\mathrm{up}}^\ell\), and
\(M_{\mathrm{down}}^\ell\) denote their momentum matrices. For a fixed
allocation \(\mu\), the admissible directions are parameterized as
\[
\Delta_{\mathrm{gate}}^\ell
=
D_{\mu^\ell}^\alpha P_{\mathrm{gate}}^\ell,
\qquad
\Delta_{\mathrm{up}}^\ell
=
D_{\mu^\ell}^\alpha P_{\mathrm{up}}^\ell,
\qquad
\Delta_{\mathrm{down}}^\ell
=
P_{\mathrm{down}}^\ell D_{\mu^\ell}^\alpha,
\]
with
\(\|P_{\mathrm{gate}}^\ell\|_{\mathrm{op}},
\|P_{\mathrm{up}}^\ell\|_{\mathrm{op}},
\|P_{\mathrm{down}}^\ell\|_{\mathrm{op}}\le1\).

The same spectral--nuclear duality argument then gives the block contribution
\[
\|D_{\mu^\ell}^\alpha M_{\mathrm{gate}}^\ell\|_*
+
\|D_{\mu^\ell}^\alpha M_{\mathrm{up}}^\ell\|_*
+
\|M_{\mathrm{down}}^\ell D_{\mu^\ell}^\alpha\|_*.
\]
Summing these terms over blocks gives the allocation objective under the same
global constraint \(\mu\in\mathcal K_K\). For an optimal allocation \(\mu^\star\), an optimal direction is
\[
\begin{aligned}
\Delta_{\mathrm{gate}}^{\ell,\star}
&=
D_{\mu^{\ell,\star}}^\alpha
\Polar\!\left(D_{\mu^{\ell,\star}}^\alpha
M_{\mathrm{gate}}^\ell\right),\\
\Delta_{\mathrm{up}}^{\ell,\star}
&=
D_{\mu^{\ell,\star}}^\alpha
\Polar\!\left(D_{\mu^{\ell,\star}}^\alpha
M_{\mathrm{up}}^\ell\right),\\
\Delta_{\mathrm{down}}^{\ell,\star}
&=
\Polar\!\left(M_{\mathrm{down}}^\ell
D_{\mu^{\ell,\star}}^\alpha\right)
D_{\mu^{\ell,\star}}^\alpha.
\end{aligned}
\]
Thus the only change from the two-matrix case is that each hidden-channel allocation is shared by three coupled matrix components.

\subsection{Allocation Curvature}
\label{app:curvature}

We prove Theorem~\ref{thm:curvature} by treating the three regimes separately.

\paragraph{Concavity for \(0<\alpha\le\tfrac12\).}
It suffices to consider one left-scaled term
\(F(\mu):=\|D_\mu^\alpha M\|_*\). Using
\(\|A\|_*=\operatorname{tr}((A^\top A)^{1/2})\),
\[
F(\mu)
=
\operatorname{tr}\!\left[
\left(M^\top D_\mu^{2\alpha}M\right)^{1/2}
\right].
\]
Define
\[
G(x)
:=
\operatorname{tr}\!\left[
\left(M^\top D_xM\right)^{1/2}
\right],
\qquad x\in\mathbb R_+^m.
\]
Then \(F(\mu)=G(\mu^{2\alpha})\), where the power is applied coordinatewise.

The matrix square-root map is operator concave and operator monotone on the
positive semidefinite cone \citep{bhatia1997matrix}. Since
\(x\mapsto M^\top D_xM\) is linear, \(G\) is concave. Moreover, \(G\) is coordinatewise nondecreasing. If \(x\le y\), then
\(D_x\preceq D_y\), hence
\(M^\top D_xM\preceq M^\top D_yM\). Operator monotonicity of the square root
and monotonicity of the trace therefore give \(G(x)\le G(y)\).

For any \(\mu,\nu\in\mathbb R_+^m\) and \(\theta\in[0,1]\), since
\(0<2\alpha\le1\), concavity of \(t\mapsto t^{2\alpha}\) gives
\[
(\theta\mu+(1-\theta)\nu)^{2\alpha}
\ge
\theta\mu^{2\alpha}+(1-\theta)\nu^{2\alpha}
\]
coordinatewise. Hence, using first the monotonicity and then the concavity of
\(G\),
\[
\begin{aligned}
F(\theta\mu+(1-\theta)\nu)
&=
G\!\left((\theta\mu+(1-\theta)\nu)^{2\alpha}\right)
\\
&\ge
G\!\left(\theta\mu^{2\alpha}+(1-\theta)\nu^{2\alpha}\right)
\\
&\ge
\theta G(\mu^{2\alpha})+(1-\theta)G(\nu^{2\alpha})
\\
&=
\theta F(\mu)+(1-\theta)F(\nu).
\end{aligned}
\]
Thus \(\mu\mapsto\|D_\mu^\alpha M\|_*\) is concave for
\(0<\alpha\le\tfrac12\). Since
\(\|MD_\mu^\alpha\|_*=\|D_\mu^\alpha M^\top\|_*\), the same argument applies
to right-scaled terms. Summing over all blocks and both terms in
Eq.~(\ref{eq:bmuon-allocation}) proves that \(\Phi_M\) is concave in \(\mu\).

\paragraph{Convexity at \(\alpha=1\).}
At \(\alpha=1\), both maps \(\mu\mapsto D_\mu M\) and
\(\mu\mapsto MD_\mu\) are linear. Since the nuclear norm is convex,
\(\mu\mapsto\|D_\mu M\|_*\) and
\(\mu\mapsto\|MD_\mu\|_*\) are convex. Summing these terms over all blocks therefore shows that \(\Phi_M\) is convex.

\paragraph{Intermediate regime \(\tfrac12<\alpha<1\).}
Take the special case \(L=1\), \(N=2\), \(K=1\), and \(M_2=0\).

For \(M_1=I_2\),
\(\Phi_M(\mu)=\mu_1^\alpha+\mu_2^\alpha\). Hence, for every \(\alpha<1\),
\[
\Phi_M\!\left(\tfrac12,\tfrac12\right)
=2^{1-\alpha}
>
\frac{\Phi_M(1,0)+\Phi_M(0,1)}{2}
=1,
\]
so \(\Phi_M\) is not convex.

Now let
\(M_1=\begin{pmatrix}1&0\\1&0\end{pmatrix}\). Then
\(\Phi_M(\mu)=(\mu_1^{2\alpha}+\mu_2^{2\alpha})^{1/2}\), and for every
\(\alpha>\tfrac12\),
\[
\Phi_M\!\left(\tfrac12,\tfrac12\right)
=2^{1/2-\alpha}
<
\frac{\Phi_M(1,0)+\Phi_M(0,1)}{2}
=1,
\]
so \(\Phi_M\) is not concave.

Thus, for \(\tfrac12<\alpha<1\), \(\Phi_M\) is in general neither convex nor concave. These examples show that the thresholds in Theorem~\ref{thm:curvature} are sharp: for \(0<\alpha\le1\), \(\Phi_M\) is concave for all momenta if and only if \(\alpha\le\tfrac12\), and convex for all momenta if and only if \(\alpha=1\).

\paragraph{Extreme-point structure at \(\alpha=1\).}
Assume \(K\) integer. The extreme points of \(\mathcal K_K\) are
exactly the binary vectors with \(K\) ones. Indeed, every such vector is
extreme because all its coordinates lie at the endpoints of \([0,1]\).
Conversely, any non-binary \(\mu\in\mathcal K_K\) has at least two fractional
coordinates: it cannot have exactly one since \(\sum_i\mu_i=K\) is an
integer. For two fractional coordinates \(i\neq j\), sufficiently small
\(\varepsilon>0\) gives
\(\mu^\pm=\mu\pm\varepsilon(e_i-e_j)\in\mathcal K_K\) with
\(\mu=\tfrac12(\mu^++\mu^-)\), so \(\mu\) is not extreme.

Since \(\Phi_M\) is continuous and convex at \(\alpha=1\), a maximizer over the polytope
\(\mathcal K_K\) can be chosen at an extreme point. Thus an optimal allocation
can be chosen binary, selecting exactly \(K\) channels. Writing
\(S_\ell:=\{j:(\ell,j)\in S\}\) and
\(D_{S_\ell}:=\operatorname{diag}(\mathbf 1_{S_\ell})\), the allocation
problem becomes
\[
\max_{\mu\in\mathcal K_K}\Phi_M(\mu)
=
\max_{\substack{
S\subseteq\{(\ell,j):1\le\ell\le L,\ 1\le j\le m_\ell\}\\
|S|=K
}}
\sum_{\ell=1}^L
\left[
\|D_{S_\ell}M_1^\ell\|_*
+
\|M_2^\ell D_{S_\ell}\|_*
\right].
\]
Hence, at \(\alpha=1\), the continuous allocation problem admits an equivalent
combinatorial formulation over subsets of \(K\) global channels.

\subsection{Norm Geometry of BMuon}
\label{app:norm}

Using $\mathcal K_K$ and $\mathcal A_\alpha(\mu)$ as defined in the
notation and setup at the beginning of Appendix~\ref{app:geometry}, define
\[
\mathcal A_{\alpha,K}
:=
\bigcup_{\mu\in\mathcal K_K}\mathcal A_\alpha(\mu),
\qquad
\mathcal B_{\alpha,K}
:=
\operatorname{conv}(\mathcal A_{\alpha,K}).
\]

Recall Muon's product norm and unit ball:
\[
\|\Delta\|_{\mathrm{Muon}}
:=
\max_\ell
\max\left\{
\|\Delta_1^\ell\|_{\mathrm{op}},
\|\Delta_2^\ell\|_{\mathrm{op}}
\right\},
\qquad
\mathcal A_{\mathrm{Muon}}
:=
\{\Delta:\|\Delta\|_{\mathrm{Muon}}\le1\}.
\]

\begin{lemma}[BMuon norm and comparison with Muon]
\label{lem:bmuon-norm}
Let \(0<\alpha\le1\) and \(0<K\le N\). Then
\(\mathcal B_{\alpha,K}\) is compact, convex, centrally symmetric, and has the
origin in its interior. Consequently, its Minkowski functional
\(\Omega_{\alpha,K}\) is a norm with unit ball \(\mathcal B_{\alpha,K}\).
Moreover, for every \(\Delta\),
\[
\|\Delta\|_{\mathrm{Muon}}
\le
\Omega_{\alpha,K}(\Delta)
\le
\left(\frac{N}{K}\right)^\alpha
\|\Delta\|_{\mathrm{Muon}}.
\]
\end{lemma}

\begin{proof}
Fix \(\mu\in\mathcal K_K\) and
\(\Delta\in\mathcal A_\alpha(\mu)\). By definition, for every \(\ell\),
there exist \(P_1^\ell,P_2^\ell\) with
\(\|P_1^\ell\|_{\mathrm{op}},\|P_2^\ell\|_{\mathrm{op}}\le1\) such that
\(\Delta_1^\ell=D_{\mu^\ell}^\alpha P_1^\ell\) and
\(\Delta_2^\ell=P_2^\ell D_{\mu^\ell}^\alpha\).
Since \(0\le\mu_j^\ell\le1\),
\(\|D_{\mu^\ell}^\alpha\|_{\mathrm{op}}\le1\). Therefore, by
submultiplicativity of the operator norm,
$\|\Delta_1^\ell\|_{\mathrm{op}}\le1,$ and $\|\Delta_2^\ell\|_{\mathrm{op}}\le1$  for every \(\ell\). Hence \(\mathcal A_\alpha(\mu)\subseteq\mathcal A_{\mathrm{Muon}}\) for every
\(\mu\in\mathcal K_K\), so
\(\mathcal A_{\alpha,K}\subseteq\mathcal A_{\mathrm{Muon}}\). Since the Muon
unit ball is convex,
\(\mathcal B_{\alpha,K}\subseteq\mathcal A_{\mathrm{Muon}}\).

Conversely, consider the uniform allocation
\(\bar\mu_j^\ell=K/N\), which belongs to \(\mathcal K_K\). Then
\(D_{\bar\mu^\ell}^\alpha=(K/N)^\alpha I\), and therefore
\(\mathcal A_\alpha(\bar\mu)
=(K/N)^\alpha\mathcal A_{\mathrm{Muon}}\). Since
\(\mathcal A_\alpha(\bar\mu)\subseteq\mathcal B_{\alpha,K}\),
\begin{equation}
\label{eq:bmuon-ball-sandwich}
\left(\frac{K}{N}\right)^\alpha
\mathcal A_{\mathrm{Muon}}
\subseteq
\mathcal B_{\alpha,K}
\subseteq
\mathcal A_{\mathrm{Muon}}.
\end{equation}

It remains to verify the geometric properties of
\(\mathcal B_{\alpha,K}\). For compactness, let
\[
\mathcal X
:=
\mathcal K_K
\times
\prod_{\ell=1}^L
\Big(
\{P_1^\ell:\|P_1^\ell\|_{\mathrm{op}}\le1\}
\times
\{P_2^\ell:\|P_2^\ell\|_{\mathrm{op}}\le1\}
\Big),
\]
which is compact in finite dimension, and define
\(F(\mu,\{P_1^\ell,P_2^\ell\})
:=((D_{\mu^\ell}^\alpha P_1^\ell,
P_2^\ell D_{\mu^\ell}^\alpha))_{\ell=1}^L\).
The map \(F\) is continuous for \(\alpha>0\), and
\(\mathcal A_{\alpha,K}=F(\mathcal X)\); hence
\(\mathcal A_{\alpha,K}\) is compact. Its convex hull
\(\mathcal B_{\alpha,K}\) is therefore compact in finite dimension and is
convex by definition. Moreover, \(\mathcal A_{\alpha,K}\) is centrally
symmetric, since replacing every \(P_i^\ell\) by \(-P_i^\ell\) replaces
\(\Delta\) by \(-\Delta\); the convex hull of a centrally symmetric set is
centrally symmetric. Finally, the left inclusion in
Eq.~(\ref{eq:bmuon-ball-sandwich}) shows that
\(\mathcal B_{\alpha,K}\) contains a neighborhood of the origin.

Thus the Minkowski functional of \(\mathcal B_{\alpha,K}\) is a norm whose
unit ball is \(\mathcal B_{\alpha,K}\). Since a larger set has a smaller
Minkowski functional, the two inclusions in
Eq.~(\ref{eq:bmuon-ball-sandwich}) give
\(\|\Delta\|_{\mathrm{Muon}}
\le\Omega_{\alpha,K}(\Delta)
\le(N/K)^\alpha\|\Delta\|_{\mathrm{Muon}}\).
\end{proof}

The following Lemma gives a characterization of the norm $\Omega_{\alpha,K}$ in the regime \(0<\alpha\le\tfrac12\).

\begin{lemma}[Characterization of the BMuon norm]
\label{lem:single-allocation}
Let \(0<\alpha\le\tfrac12\) and \(0<K\le N\). Then
\(\mathcal A_{\alpha,K}\) is convex, and hence
$
\mathcal B_{\alpha,K}=\mathcal A_{\alpha,K}.
$ For \(\mu\in\mathcal K_K\), define
\[
\rho_\mu(\Delta)
:=
\max_\ell
\max\left\{
\|D_{\mu^\ell}^{-\alpha}\Delta_1^\ell\|_{\mathrm{op}},
\|\Delta_2^\ell D_{\mu^\ell}^{-\alpha}\|_{\mathrm{op}}
\right\},
\]
with the following convention: if \(\mu_j^\ell=0\), finiteness requires
\(\Delta_{1,j:}^\ell=0\) and \(\Delta_{2,:,j}^\ell=0\); the corresponding
zero rows and columns are then omitted when evaluating the two operator
norms. If either is nonzero, set \(\rho_\mu(\Delta)=+\infty\). Then
\[
\Omega_{\alpha,K}(\Delta)
=
\min_{\mu\in\mathcal K_K}\rho_\mu(\Delta).
\]
\end{lemma}

\begin{proof}
Set \(\beta:=2\alpha\in(0,1]\), and define
\[
\varphi_\beta(z,s)
:=
\begin{cases}
z^2s^{-\beta}, & s>0,\\
0, & (z,s)=(0,0),\\
+\infty, & s=0,\ z\neq0.
\end{cases}
\]

For \(s>0\), the Hessian of \(\varphi_\beta\) is positive
semidefinite: \(\partial_{zz}^2\varphi_\beta=2s^{-\beta}>0\) and
\[
\det\nabla^2\varphi_\beta(z,s)
=
2\beta(1-\beta)z^2s^{-2\beta-2}\ge0,
\]
since \(\beta\in(0,1]\). Hence \(\varphi_\beta\) is jointly convex on
\(\mathbb R\times(0,\infty)\).

It remains to check the boundary \(s=0\). The only boundary point at
which \(\varphi_\beta\) is finite is \((0,0)\). For \(s>0\) and
\(\theta\in[0,1]\),
\[
\varphi_\beta(\theta z,\theta s)
=
\theta^{2-\beta}\varphi_\beta(z,s)
\le
\theta\varphi_\beta(z,s),
\]
because \(2-\beta\ge1\), which is exactly the convexity inequality between
$(z,s)$ and $(0,0)$. If a boundary point has \(s=0\) and
\(z\neq0\), its value is \(+\infty\), so the convexity inequality is
automatic. Thus \(\varphi_\beta\) is convex on
\(\mathbb R\times[0,\infty)\).

It is also lower semicontinuous. It is continuous for \(s>0\). At
\((0,0)\), nonnegativity gives
\[
\liminf_{(z,s)\to(0,0)}\varphi_\beta(z,s)
\ge0
=
\varphi_\beta(0,0).
\]
If \(z_0\neq0\), then as \((z,s)\to(z_0,0)\), we eventually have
\(|z|\ge |z_0|/2\). If \(s=0\), then
\(\varphi_\beta(z,s)=+\infty\); if \(s>0\), then
\[
\varphi_\beta(z,s)
=
z^2s^{-\beta}
\ge
\frac{z_0^2}{4}s^{-\beta}
\longrightarrow+\infty.
\]
Hence
\[
\liminf_{(z,s)\to(z_0,0)}\varphi_\beta(z,s)
=
+\infty
=
\varphi_\beta(z_0,0).
\]
Therefore \(\varphi_\beta\) is lower semicontinuous on
\(\mathbb R\times[0,\infty)\).

For every block \(\ell\),
\[
\begin{aligned}
\|D_{\mu^\ell}^{-\alpha}\Delta_1^\ell\|_{\mathrm{op}}^2
&=
\sup_{\|x\|_2=1}
\|D_{\mu^\ell}^{-\alpha}\Delta_1^\ell x\|_2^2\\
&=
\sup_{\|x\|_2=1}
\sum_{j=1}^{m_\ell}
(\mu_j^\ell)^{-2\alpha}
\langle\Delta_{1,j:}^\ell,x\rangle^2\\
&=
\sup_{\|x\|_2=1}
\sum_{j=1}^{m_\ell}
\varphi_\beta(
\langle\Delta_{1,j:}^\ell,x\rangle,\mu_j^\ell).
\end{aligned}
\]

This interpretation agrees with the zero-allocation convention
in the definition of \(\rho_\mu\). Indeed, if \(\mu_j^\ell=0\) and
\(\Delta_{1,j:}^\ell=0\), then
\(\varphi_\beta(0,0)=0\), which amounts to omitting that row from the
operator norm. If instead \(\mu_j^\ell=0\) and
\(\Delta_{1,j:}^\ell\neq0\), there exists \(x\) such that
\(\langle\Delta_{1,j:}^\ell,x\rangle\neq0\), and hence the supremum is
\(+\infty\), matching the zero-allocation convention for
\(\rho_\mu\).

Similarly,
\[
\|\Delta_2^\ell D_{\mu^\ell}^{-\alpha}\|_{\mathrm{op}}^2
=
\sup_{\|y\|_2=1}
\sum_{j=1}^{m_\ell}
\varphi_\beta(
\langle\Delta_{2,:,j}^\ell,y\rangle,\mu_j^\ell).
\]

For fixed \(\ell,j\), and \(x\), the map
$
(\Delta,\mu)
\mapsto
\varphi_\beta(
\langle\Delta_{1,j:}^\ell,x\rangle,\mu_j^\ell)
$
is jointly convex and lower semicontinuous, since
\(\varphi_\beta\) has these properties and its arguments depend linearly
on \((\Delta,\mu)\). Finite sums preserve both properties, and so does
taking the pointwise supremum over \(x\). Hence
$
(\Delta,\mu)\mapsto
\|D_{\mu^\ell}^{-\alpha}\Delta_1^\ell\|_{\mathrm{op}}^2
$
is jointly convex and lower semicontinuous. The same argument applies to
$
(\Delta,\mu)\mapsto
\|\Delta_2^\ell D_{\mu^\ell}^{-\alpha}\|_{\mathrm{op}}^2$. Finally,
$
g(\Delta,\mu):=\rho_\mu(\Delta)^2
$
is the finite maximum of these functions over the blocks \(\ell\), and
is therefore jointly convex and lower semicontinuous.

We now show that
\begin{equation}
\label{eq:fixed-allocation-unit-ball}
\mathcal A_\alpha(\mu)
=
\{\Delta:\rho_\mu(\Delta)\le1\}.
\end{equation}

If \(\Delta\in\mathcal A_\alpha(\mu)\), then there exist
\(\{P_1^\ell,P_2^\ell\}_{\ell=1}^L\) such that, for every
\(\ell\),
\(\Delta_1^\ell=D_{\mu^\ell}^\alpha P_1^\ell\) and
\(\Delta_2^\ell=P_2^\ell D_{\mu^\ell}^\alpha\), with
\(\|P_1^\ell\|_{\mathrm{op}},\|P_2^\ell\|_{\mathrm{op}}\le1\).
Let \(S_\ell:=\{j:\mu_j^\ell>0\}\). If \(j\notin S_\ell\), then
\(\Delta_{1,j:}^\ell=\Delta_{2,:,j}^\ell=0\). In particular, if
\(S_\ell=\varnothing\), then
\(\Delta_1^\ell=\Delta_2^\ell=0\), so both operator norms appearing in
\(\rho_\mu\) for this block are zero under the zero-allocation convention.

If \(S_\ell\neq\varnothing\), then
\[
D_{\mu^\ell_{S_\ell}}^{-\alpha}\Delta_{1,S_\ell:}^\ell
=
P_{1,S_\ell:}^\ell,
\qquad
\Delta_{2,:,S_\ell}^\ell D_{\mu^\ell_{S_\ell}}^{-\alpha}
=
P_{2,:,S_\ell}^\ell.
\]
Removing rows or columns cannot increase the operator norm, so
\(\|P_{1,S_\ell:}^\ell\|_{\mathrm{op}}\le\|P_1^\ell\|_{\mathrm{op}}\le1\) and
\(\|P_{2,:,S_\ell}^\ell\|_{\mathrm{op}}\le\|P_2^\ell\|_{\mathrm{op}}\le1\).
By the zero-allocation convention defining \(\rho_\mu\), the rows and
columns with \(\mu_j^\ell=0\), which are zero in \(\Delta\), are omitted
from the operator-norm computation. Hence
\(\|D_{\mu^\ell}^{-\alpha}\Delta_1^\ell\|_{\mathrm{op}}
=\|D_{\mu^\ell_{S_\ell}}^{-\alpha}\Delta_{1,S_\ell:}^\ell\|_{\mathrm{op}}
=\|P_{1,S_\ell:}^\ell\|_{\mathrm{op}}\le1\), and similarly
\(\|\Delta_2^\ell D_{\mu^\ell}^{-\alpha}\|_{\mathrm{op}}\le1\).
Since this holds for every \(\ell\), we conclude \(\rho_\mu(\Delta)\le1\).

Conversely, if \(\rho_\mu(\Delta)\le1\), finiteness forces
\(\Delta_{1,j:}^\ell=\Delta_{2,:,j}^\ell=0\) whenever
\(j\notin S_\ell\). Define
\(P_{1,S_\ell:}^\ell
=D_{\mu^\ell_{S_\ell}}^{-\alpha}\Delta_{1,S_\ell:}^\ell\) and
\(P_{2,:,S_\ell}^\ell
=\Delta_{2,:,S_\ell}^\ell D_{\mu^\ell_{S_\ell}}^{-\alpha}\),
and set the remaining rows of \(P_1^\ell\) and columns of \(P_2^\ell\)
to zero. By definition of \(\rho_\mu\),
\(\|P_1^\ell\|_{\mathrm{op}}\le1\) and
\(\|P_2^\ell\|_{\mathrm{op}}\le1\), while
\(\Delta_1^\ell=D_{\mu^\ell}^\alpha P_1^\ell\) and
\(\Delta_2^\ell=P_2^\ell D_{\mu^\ell}^\alpha\).
Thus \(\Delta\in\mathcal A_\alpha(\mu)\).

Therefore, Eq.~(\ref{eq:fixed-allocation-unit-ball}) is proved.

Now take \(\Delta,\Delta'\in\mathcal A_{\alpha,K}\). By definition of
\(\mathcal A_{\alpha,K}\), there exist \(\mu,\mu'\in\mathcal K_K\) such that
\(\Delta\in\mathcal A_\alpha(\mu)\) and
\(\Delta'\in\mathcal A_\alpha(\mu')\). For \(\theta\in[0,1]\), set
\(\mu_\theta=\theta\mu+(1-\theta)\mu'\in\mathcal K_K\). Joint convexity of
\(g(\Delta,\mu)=\rho_\mu(\Delta)^2\) gives
\[
g(\theta\Delta+(1-\theta)\Delta',\mu_\theta)
\le
\theta g(\Delta,\mu)+(1-\theta)g(\Delta',\mu')
\le1.
\]
Hence
\(\theta\Delta+(1-\theta)\Delta'
\in\mathcal A_\alpha(\mu_\theta)\subseteq\mathcal A_{\alpha,K}\).
Therefore \(\mathcal A_{\alpha,K}\) is convex, and thus
\(\mathcal B_{\alpha,K}=\mathcal A_{\alpha,K}\).

Using the definition of the Minkowski functional and the positive
homogeneity of \(\rho_\mu\),
\[
\begin{aligned}
\Omega_{\alpha,K}(\Delta)
&=
\inf\{t>0:\Delta/t\in\mathcal A_{\alpha,K}\}\\
&=
\inf\{t>0:\exists\,\mu\in\mathcal K_K,\ 
\rho_\mu(\Delta)\le t\}\\
&=
\inf_{\mu\in\mathcal K_K}\rho_\mu(\Delta).
\end{aligned}
\]
Finally, for fixed \(\Delta\), \(g(\Delta,\mu)\) is lower
semicontinuous in \(\mu\). Since \(\mathcal K_K\) is compact, it attains
its minimum, and the same allocation minimizes
\(\rho_\mu(\Delta)=\sqrt{g(\Delta,\mu)}\). Thus the last infimum is a
minimum.
\end{proof}

\subsection{Proofs for the Geometric Special Cases}
\label{app:special-cases}

We prove the special cases stated in Section~\ref{sec:special-cases}.

\paragraph{Recovery of Muon.}
If \(K=N\), then \(0\le\mu_j^\ell\le1\) and
\(\sum_{\ell,j}\mu_j^\ell=N\) imply \(\mu_j^\ell=1\) for every channel.
Hence \(D_{\mu^\ell}^\alpha=I\) for every \(\ell\), so
\(\mathcal A_{\alpha,N}=\mathcal A_{\mathrm{Muon}}\). Since
\(\mathcal A_{\mathrm{Muon}}\) is convex,
\(\mathcal B_{\alpha,N}=\mathcal A_{\mathrm{Muon}}\), and therefore
\(\Omega_{\alpha,N}=\|\cdot\|_{\mathrm{Muon}}\).

Likewise, under the extension \(D_\mu^0:=I\),
\(\mathcal A_0(\mu)=\mathcal A_{\mathrm{Muon}}\) for every
\(\mu\in\mathcal K_K\). Thus
\(\mathcal B_{0,K}=\mathcal A_{\mathrm{Muon}}\) and
\(\Omega_{0,K}=\|\cdot\|_{\mathrm{Muon}}\).

\paragraph{Structured \(K\)-support geometry at \(\alpha=1\).}
Assume \(K\) is an integer and let
\(\mathcal V_K:=\{\mathbf 1_S:|S|=K\}\), where
\(S\subseteq\{(\ell,j):1\le\ell\le L,\ 1\le j\le m_\ell\}\) and
\(\mathbf 1_S^\ell\) denotes the block-\(\ell\) component of
\(\mathbf 1_S\), with
\((\mathbf 1_S^\ell)_j=\mathbf 1\{(\ell,j)\in S\}\).
As shown in Appendix~\ref{app:curvature}, the extreme points of
\(\mathcal K_K\) are exactly \(\mathcal V_K\). Since \(\mathcal K_K\) is a polytope, \(\mathcal K_K=\operatorname{conv}(\mathcal V_K)\).

Let \(\mu\in\mathcal K_K\), it can be written as
\(\mu=\sum_r\lambda_r\mathbf 1_{S_r}\), where
\(\lambda_r\ge0\), \(\sum_r\lambda_r=1\), and \(|S_r|=K\).
Now let \(\Delta\in\mathcal A_1(\mu)\). Then, for some
\(P_1^\ell,P_2^\ell\) satisfying
\(\|P_1^\ell\|_{\mathrm{op}},\|P_2^\ell\|_{\mathrm{op}}\le1\),
\(\Delta=
\bigl(D_{\mu^\ell}P_1^\ell,\,
P_2^\ell D_{\mu^\ell}\bigr)_{\ell=1}^L\).
Since \(D_{\mu^\ell}\) is linear in \(\mu^\ell\),
\(\Delta=
\sum_r\lambda_r
\bigl(
D_{\mathbf 1_{S_r}^\ell}P_1^\ell,\,
P_2^\ell D_{\mathbf 1_{S_r}^\ell}
\bigr)_{\ell=1}^L\),
and each term belongs to
\(\mathcal A_1(\mathbf 1_{S_r})\). Therefore
\[
\bigcup_{\mu\in\mathcal K_K}\mathcal A_1(\mu)
\subseteq
\operatorname{conv}
\left(
\bigcup_{|S|=K}\mathcal A_1(\mathbf 1_S)
\right).
\]
Conversely, \(|S|=K\) implies
\(\mathbf 1_S\in\mathcal K_K\). Hence, after taking convex hulls,
\[
\mathcal B_{1,K}
=
\operatorname{conv}
\left(
\bigcup_{\mu\in\mathcal K_K}\mathcal A_1(\mu)
\right)
=
\operatorname{conv}
\left(
\bigcup_{|S|=K}\mathcal A_1(\mathbf 1_S)
\right).
\]

For a binary allocation \(\mathbf 1_S\),
\[
\begin{aligned}
\mathcal A_1(\mathbf 1_S)
&=
\Big\{
\Delta:\ 
\Delta_1^\ell=D_{\mathbf 1_S^\ell}P_1^\ell,\ 
\Delta_2^\ell=P_2^\ell D_{\mathbf 1_S^\ell},\
\|P_1^\ell\|_{\mathrm{op}}\le1,\
\|P_2^\ell\|_{\mathrm{op}}\le1,\ \forall\ell
\Big\}
\\
&=
\Big\{
\Delta:\ 
\Delta_{1,j:}^\ell=0,\ 
\Delta_{2,:,j}^\ell=0
\ \text{for all }(\ell,j)\notin S,\
\|\Delta_1^\ell\|_{\mathrm{op}}\le1,\
\|\Delta_2^\ell\|_{\mathrm{op}}\le1,\ \forall\ell
\Big\}
\\
&=
\Big\{
\Delta:\ 
\operatorname{supp}_{\mathrm{ch}}(\Delta)\subseteq S,\
\|\Delta_1^\ell\|_{\mathrm{op}}\le1,\
\|\Delta_2^\ell\|_{\mathrm{op}}\le1,\ \forall\ell
\Big\}.
\end{aligned}
\]

The second equality uses that multiplication by
\(D_{\mathbf 1_S^\ell}\) zeros exactly the channels outside \(S\) and does
not increase the operator norm; conversely, one may take
\(P_i^\ell=\Delta_i^\ell\).

Consequently,
\[
\begin{aligned}
\mathcal B_{1,K}
&=
\operatorname{conv}
\left(
\bigcup_{|S|=K}
\left\{
\Delta:
\operatorname{supp}_{\mathrm{ch}}(\Delta)\subseteq S,\ 
\|\Delta_1^\ell\|_{\mathrm{op}}\le1,\ 
\|\Delta_2^\ell\|_{\mathrm{op}}\le1,\ \forall\ell
\right\}
\right)
\\
&=
\operatorname{conv}\Big\{
\Delta:
|\operatorname{supp}_{\mathrm{ch}}(\Delta)|\le K,\ 
\|\Delta_1^\ell\|_{\mathrm{op}}\le1,\ 
\|\Delta_2^\ell\|_{\mathrm{op}}\le1,\ \forall\ell
\Big\}.
\end{aligned}
\]

\paragraph{Explicit norm for \(\alpha=K=1\).}
Since \(\mathcal B_{\alpha,K}\) is the unit ball of
\(\Omega_{\alpha,K}\),
\[
\Omega_{\alpha,K}^*(M)
=
\max_{\Delta\in\mathcal B_{\alpha,K}}
\langle M,\Delta\rangle
=
\max_{\mu\in\mathcal K_K}\Phi_M(\mu).
\]
As shown in Appendix~\ref{app:curvature}, at \(\alpha=1\) a maximizer can be chosen at an extreme point of \(\mathcal K_K\). For \(K=1\), these extreme points are exactly the binary vectors with a single nonzero entry.
Therefore
\[
\begin{aligned}
\Omega_{1,1}^*(M)
&=
\max_{\mu\in\mathcal K_1}
\sum_{r=1}^L
\left[
\|D_{\mu^r}M_1^r\|_*
+
\|M_2^rD_{\mu^r}\|_*
\right]
\\
&=
\max_{\ell,j}
\sum_{r=1}^L
\left[
\|D_{\mathbf 1_{\{(\ell,j)\}}^r}M_1^r\|_*
+
\|M_2^rD_{\mathbf 1_{\{(\ell,j)\}}^r}\|_*
\right]
\\
&=
\max_{\ell,j}
\left[
\|D_{\mathbf 1_{\{j\}}}M_1^\ell\|_*
+
\|M_2^\ell D_{\mathbf 1_{\{j\}}}\|_*
\right]
\\
&=
\max_{\ell,j}
\left[
\|M_{1,j:}^\ell\|_2
+
\|M_{2,:,j}^\ell\|_2
\right].
\end{aligned}
\]
The last equality holds because the two matrices in the third equality have
respectively a single nonzero row and a single nonzero column, and the nuclear
norm of each equals the Euclidean norm of its nonzero row or column.

Dualizing gives the result
\[
\Omega_{1,1}(\Delta)
=
\sum_{\ell=1}^L\sum_{j=1}^{m_\ell}
\max\!\left\{
\|\Delta_{1,j:}^\ell\|_2,\,
\|\Delta_{2,:,j}^\ell\|_2
\right\}.
\]

\newpage
\section{Compact-Representation Theory}
\label{app:compact-theory}

Appendix~\ref{app:intrinsic-width} proves
Theorem~\ref{thm:intrinsic-width}, showing that a sufficiently scarce BMuon
budget forces every minimum-cost linear factorization to have minimal hidden
width, whereas Muon admits minimum-cost representations at full width.
Appendix~\ref{app:weight-decay} gives a complementary weight-decay result:
channel norms are controlled by their allocation histories, and channels whose
allocations vanish are driven to zero.

\subsection{Intrinsic-Width Recovery}
\label{app:intrinsic-width}

Let \(X=(X_\ell)_{\ell=1}^L\), where
\(X_\ell\in\mathbb R^{p_\ell\times d_\ell}\) and \(X_\ell\neq0\) for
every \(\ell\). Let \(r_\ell:=\operatorname{rank}(X_\ell)\ge1\), and
let \(\sigma_{\ell,1}\ge\cdots\ge\sigma_{\ell,r_\ell}>0\) denote the
nonzero singular values of \(X_\ell\). Set
\(\sigma_\star:=\max_{\ell,i}\sigma_{\ell,i}\).

Each block \(\ell\) consists of
\(W_1^\ell\in\mathbb R^{m_\ell\times d_\ell}\) and
\(W_2^\ell\in\mathbb R^{p_\ell\times m_\ell}\). Assume \(m_\ell\ge r_\ell\) for every \(\ell\). We consider the linear,
bias-free setting, so block \(\ell\) represents the map
\(W_2^\ell W_1^\ell\). Let \(N:=\sum_{\ell=1}^L m_\ell\) denote the
total number of hidden channels across blocks. An exact factorization
\(W\) satisfies \(W_2^\ell W_1^\ell=X_\ell\) for every \(\ell\).

Recall that, for
\(W=\bigl((W_1^\ell,W_2^\ell)\bigr)_{\ell=1}^L\),
$
\|W\|_{\mathrm{Muon}}
:=
\max_{1\le\ell\le L}
\max\left\{
\|W_1^\ell\|_{\mathrm{op}},
\|W_2^\ell\|_{\mathrm{op}}
\right\},
$
and define
\[
\mathcal R_{\mathrm{Muon}}(X)
:=
\inf_{\substack{W_2^\ell W_1^\ell=X_\ell\\ \forall\ell}}
\|W\|_{\mathrm{Muon}}^2,
\qquad
\mathcal R_{\alpha,K}(X)
:=
\inf_{\substack{W_2^\ell W_1^\ell=X_\ell\\ \forall\ell}}
\Omega_{\alpha,K}(W)^2.
\]

For an exact factorization \(W\), let \(a_\ell(W)\) denote the number
of channels \(j\) such that \(W_{1,j:}^\ell\neq0\) or
\(W_{2,:,j}^\ell\neq0\). Finally, recall
\(d_{\mathrm{eff}}^{(\alpha)}(X)
:=\sum_{\ell=1}^L\sum_{i=1}^{r_\ell}
(\sigma_{\ell,i}/\sigma_\star)^{1/(2\alpha)}\).

\paragraph{Theorem~\ref{thm:intrinsic-width} (restated).}
Let \(0<\alpha\le\tfrac12\) and \(0<K\le N\). Both representation-cost
infima are attained, and
\[
\mathcal R_{\mathrm{Muon}}(X)=\sigma_\star,
\qquad
\mathcal R_{\alpha,K}(X)
=
\sigma_\star
\max\left\{
1,\,
\left(
\frac{d_{\mathrm{eff}}^{(\alpha)}(X)}{K}
\right)^{2\alpha}
\right\}.
\]
Moreover, Muon admits a minimizer \(W^\star\) satisfying
\(a_\ell(W^\star)=m_\ell\) for every \(\ell\), whereas if
\(K\le d_{\mathrm{eff}}^{(\alpha)}(X)\), every BMuon minimizer
\(W^\star\) satisfies \(a_\ell(W^\star)=r_\ell\) for every \(\ell\).

For \(\mu\in\mathcal K_K\), recall the fixed-allocation radius
\[
\rho_\mu(W)
:=
\max_{1\le\ell\le L}
\max\left\{
\|D_{\mu^\ell}^{-\alpha}W_1^\ell\|_{\mathrm{op}},
\|W_2^\ell D_{\mu^\ell}^{-\alpha}\|_{\mathrm{op}}
\right\},
\]
with the zero-allocation convention of
Lemma~\ref{lem:single-allocation}. By that lemma,
\(\Omega_{\alpha,K}(W)=\min_{\mu\in\mathcal K_K}\rho_\mu(W)\).

We first consider the following single-block case.
\begin{lemma}[Fixed-allocation factorization]
\label{lem:fixed-allocation-factorization}
Let \(\alpha>0\). Consider a single block (\(L=1\)). Let
\(X\in\mathbb R^{p\times d}\) have rank \(r\ge1\) and nonzero singular
values \(\sigma_1\ge\cdots\ge\sigma_r>0\). For
\(\mu\in[0,1]^m\), \(m\ge r\), let
\(\mu_{(1)}\ge\cdots\ge\mu_{(m)}\) denote its ordered coordinates.
For \(W=(W_1,W_2)\), define
\[
\rho_\mu(W)
:=
\max\left\{
\|D_\mu^{-\alpha}W_1\|_{\mathrm{op}},
\|W_2D_\mu^{-\alpha}\|_{\mathrm{op}}
\right\},
\]
with the zero-allocation convention of
Lemma~\ref{lem:single-allocation}. Then
\[
\inf_{W_2W_1=X}\rho_\mu(W)^2
=
\max_{1\le i\le r}
\frac{\sigma_i}{\mu_{(i)}^{2\alpha}}.
\]
The right-hand side is \(+\infty\) if \(\mu_{(i)}=0\) for some
\(1\le i\le r\). Whenever it is finite, the infimum is attained.
\end{lemma}

\begin{proof}
If \(\mu_{(r)}=0\), then fewer than \(r\) coordinates satisfy
\(\mu_j>0\). Suppose a $W = (W_1, W_2)$ is an exact factorization of \(X\),
(i.e. \(W_2W_1=X\)) having finite
fixed-allocation radius \(\rho_\mu(W)\). Finiteness of the radius forces
\(W_{1,j:}=0\) and \(W_{2,:,j}=0\) whenever \(\mu_j=0\) (convention of
Lemma~\ref{lem:single-allocation}). Hence
$
X
=
\sum_{j:\mu_j>0}
W_{2,:,j}W_{1,j:}
$
is a sum of fewer than \(r\) rank-one matrices. Therefore
\(\operatorname{rank}(X)<r\), contradicting
\(\operatorname{rank}(X)=r\). Thus, when \(\mu_{(r)}=0\), every $W$ satisfying exact factorization has \(\rho_\mu(W)=+\infty\). Since the right-hand side is also
\(+\infty\) by convention when \(\mu_{(r)}=0\), the claimed equality holds in this case.

Assume now \(\mu_{(r)}>0\), and let \(W=(W_1,W_2)\) be an exact factorization of \(X\),
i.e. \(W_2W_1=X\), with finite fixed-allocation radius
\(\rho_\mu(W)<\infty\). Let \(S:=\{j:\mu_j>0\}\). Finite radius implies
\(W_{1,S^c:}=0\) and \(W_{2,:,S^c}=0\) (convention of
Lemma~\ref{lem:single-allocation}), hence
$
X
=
\bigl(W_{2,:,S}D_{\mu_S}^{-\alpha}\bigr)
D_{\mu_S}^{2\alpha}
\bigl(D_{\mu_S}^{-\alpha}W_{1,S:}\bigr).
$
Using
\(\sigma_i(ABC)\le
\|A\|_{\mathrm{op}}\sigma_i(B)\|C\|_{\mathrm{op}}\),
we obtain, for \(i\le r\),
\[
\sigma_i
\le
\max\left\{
\|D_\mu^{-\alpha}W_1\|_{\mathrm{op}}^2,\,
\|W_2D_\mu^{-\alpha}\|_{\mathrm{op}}^2
\right\}
\mu_{(i)}^{2\alpha}.
\]

The weighted operator norms are interpreted according to
the zero-allocation convention in the definition of \(\rho_\mu\) in Lemma~\ref{lem:single-allocation}.

So, for any exact factorization with finite fixed-allocation radius, for every \(i\le r\),
$
\rho_\mu(W)^2
\ge
\frac{\sigma_i}{\mu_{(i)}^{2\alpha}}
$, hence
\[
\rho_\mu(W)^2
\ge
\max_{1\le i\le r}
\frac{\sigma_i}{\mu_{(i)}^{2\alpha}}.
\]
The same bound is trivial when \(\rho_\mu(W)=+\infty\). Taking the
infimum over all exact factorizations $W$ gives
\[
\inf_{W_2W_1=X}\rho_\mu(W)^2
\ge
\max_{1\le i\le r}
\frac{\sigma_i}{\mu_{(i)}^{2\alpha}}.
\]

For the reverse inequality, permutation of the hidden coordinates leaves
both \(W_2W_1\) and
\(\|D_\mu^{-\alpha}W_1\|_{\mathrm{op}},
\|W_2D_\mu^{-\alpha}\|_{\mathrm{op}}\)
unchanged. So assume without loss of generality that
\(\mu_1\ge\cdots\ge\mu_m\). Let
\(X=U\operatorname{diag}(\sigma_1,\ldots,\sigma_r)V^\top\) be a compact
SVD. Set
\(W_{1,i:}=\sqrt{\sigma_i}\,v_i^\top\) and
\(W_{2,:,i}=\sqrt{\sigma_i}\,u_i\) for \(i\le r\), with all remaining
rows and columns equal to zero. Then \(W_2W_1=X\), and since \(U\) and \(V\) have orthonormal columns, the nonzero singular
values of both $D_\mu^{-\alpha}W_1$ and $W_2D_\mu^{-\alpha}$ are
\(\{\sqrt{\sigma_i}/\mu_i^\alpha\}_{i=1}^r\). Hence
\[
\|D_\mu^{-\alpha}W_1\|_{\mathrm{op}}^2
=
\|W_2D_\mu^{-\alpha}\|_{\mathrm{op}}^2
=
\max_{1\le i\le r}\frac{\sigma_i}{\mu_i^{2\alpha}},
\]
which gives the reverse inequality and completes the proof.
\end{proof}

We now prove Theorem~\ref{thm:intrinsic-width}.
\begin{proof}
We first consider Muon. Recall that an exact factorization
\(W=(W^\ell)_{\ell=1}^L\), with
\(W^\ell=(W_1^\ell,W_2^\ell)\), satisfies
\(W_2^\ell W_1^\ell=X_\ell\) for every
\(\ell\in\{1,\ldots,L\}\). For any such factorization $W$,
\[
\sigma_{\ell,1}
=
\|X_\ell\|_{\mathrm{op}}
\le
\|W_2^\ell\|_{\mathrm{op}}\|W_1^\ell\|_{\mathrm{op}}
\le
\|W\|_{\mathrm{Muon}}^2,
\]
hence \(\|W\|_{\mathrm{Muon}}^2\ge\sigma_\star\).

To attain this bound at full width, let
\(X_\ell=U_\ell\Sigma_\ell V_\ell^\top\) be a compact SVD and choose
\(A_\ell\in\mathbb R^{m_\ell\times r_\ell}\) such that
\(A_\ell^\top A_\ell=I\) and every row of \(A_\ell\) is nonzero. For
example, take its first column to be
\(m_\ell^{-1/2}\mathbf 1\) and complete it to an orthonormal family.
Define
$W_1^\ell=A_\ell\Sigma_\ell^{1/2}V_\ell^\top$ and $W_2^\ell=U_\ell\Sigma_\ell^{1/2}A_\ell^\top$, then \(W_2^\ell W_1^\ell=X_\ell\) and
\(\|W_1^\ell\|_{\mathrm{op}}^2
=\|W_2^\ell\|_{\mathrm{op}}^2=\sigma_{\ell,1}\). Thus
\(\mathcal R_{\mathrm{Muon}}(X)=\sigma_\star\). Since every row of
\(A_\ell\) is nonzero and \(\Sigma_\ell\) has strictly positive
diagonal entries, every row of \(W_1^\ell\) and every column of
\(W_2^\ell\) is nonzero. Hence this minimizer satisfies
\(a_\ell(W)=m_\ell\) for every \(\ell\), which concludes the Muon part
of the proof.

We now turn to BMuon. Lemma~\ref{lem:single-allocation} implies:
\begin{equation}
\label{eq:bmuon-representation-cost}
\mathcal R_{\alpha,K}(X)
=
\inf_{\mu\in\mathcal K_K}
\inf_{\substack{W_2^\ell W_1^\ell=X_\ell\\ \forall\ell}}
\rho_\mu(W)^2.
\end{equation}
For fixed \(\mu\), denote
\[
c_\ell
:=
\inf_{W_2^\ell W_1^\ell=X_\ell}
\max\left\{
\|D_{\mu^\ell}^{-\alpha}W_1^\ell\|_{\mathrm{op}}^2,\,
\|W_2^\ell D_{\mu^\ell}^{-\alpha}\|_{\mathrm{op}}^2
\right\}.
\]

We want to prove that, for this fixed allocation \(\mu\),
\begin{equation}
\label{eq:blockwise-factorization-cost}
\inf_{\substack{W_2^\ell W_1^\ell=X_\ell\\ \forall\ell}}
\rho_\mu(W)^2
=
\max_{1\le\ell\le L} c_\ell.
\end{equation}

For every exact factorization \(W\) and every
\(\ell\in\{1,\ldots,L\}\),
\[
\max\left\{
\|D_{\mu^\ell}^{-\alpha}W_1^\ell\|_{\mathrm{op}}^2,\,
\|W_2^\ell D_{\mu^\ell}^{-\alpha}\|_{\mathrm{op}}^2
\right\}
\ge c_\ell.
\]
Hence \(\rho_\mu(W)^2\ge \max_{1\le\ell\le L}c_\ell\), and thus
\[
\inf_{\substack{W_2^\ell W_1^\ell=X_\ell\\ \forall\ell}}
\rho_\mu(W)^2
\ge
\max_{1\le\ell\le L}c_\ell.
\]
Conversely, suppose first that every \(c_\ell<\infty\).
Lemma~\ref{lem:fixed-allocation-factorization} shows that, for every
\(\ell\), the infimum defining \(c_\ell\) is attained by some exact
factorization \(W^{\ell,\star}\). Let $W^\star:=\bigl(W^{\ell,\star}\bigr)_{\ell=1}^L$, we have $\rho_\mu(W^\star)^2=\max_{1\le\ell\le L}c_\ell$.
Therefore
\[
\inf_{\substack{W_2^\ell W_1^\ell=X_\ell\\ \forall\ell}}
\rho_\mu(W)^2
\le
\max_{1\le\ell\le L}c_\ell.
\]
The same inequality holds if for some \(\ell\), \(c_\ell=+\infty\).
Hence, in all cases, Eq.~(\ref{eq:blockwise-factorization-cost}) holds.

For this fixed \(\mu\), let \(\mu_{(1)}^\ell\ge\cdots\ge\mu_{(m_\ell)}^\ell\)
denote the ordered coordinates of \(\mu^\ell\). By
Lemma~\ref{lem:fixed-allocation-factorization},
$
c_\ell
=
\max_{1\le i\le r_\ell}
\frac{\sigma_{\ell,i}}
     {(\mu_{(i)}^\ell)^{2\alpha}}.
$

Therefore, combining
Eqs.~(\ref{eq:bmuon-representation-cost}) and~(\ref{eq:blockwise-factorization-cost}),
\[
\mathcal R_{\alpha,K}(X)
=
\inf_{\mu\in\mathcal K_K}
\max_{1\le\ell\le L}
\max_{1\le i\le r_\ell}
\frac{\sigma_{\ell,i}}
     {(\mu_{(i)}^\ell)^{2\alpha}}.
\]

Set \(q:=1/(2\alpha)\). For any \(\mu\in\mathcal K_K\), denote
$
M(\mu)
:=
\max_{1\le\ell\le L}
\max_{1\le i\le r_\ell}
\frac{\sigma_{\ell,i}}
     {(\mu_{(i)}^\ell)^{2\alpha}}.
$

Assume first that \(M(\mu)<\infty\). Since \(\mu_{(i)}^\ell\le1\), we have
\(M(\mu)\ge\sigma_{\ell,i}\) for every \(\ell\) and \(i\le r_\ell\).
Hence $M(\mu)\ge\sigma_\star$. Moreover, for every \(\ell\) and \(i\le r_\ell\),
\(M(\mu)\ge \sigma_{\ell,i}/(\mu_{(i)}^\ell)^{2\alpha}\) implies
$
\mu_{(i)}^\ell
\ge
\left(\frac{\sigma_{\ell,i}}{M(\mu)}\right)^q.
$
Therefore
\[
K
=
\sum_{\ell=1}^L\sum_{j=1}^{m_\ell}\mu_j^\ell
\ge
\sum_{\ell=1}^L\sum_{i=1}^{r_\ell}\mu_{(i)}^\ell
\ge
\frac{1}{M(\mu)^q}
\sum_{\ell=1}^L\sum_{i=1}^{r_\ell}\sigma_{\ell,i}^q,
\]
and thus
\[
M(\mu)
\ge
\left(
\frac{\sum_{\ell,i}\sigma_{\ell,i}^q}{K}
\right)^{1/q}.
\]

Hence
$
M(\mu)
\ge
\max\left\{
\sigma_\star,\,
\left(
\frac{\sum_{\ell,i}\sigma_{\ell,i}^q}{K}
\right)^{1/q}
\right\},
$ which also holds directly when \(M(\mu)=+\infty\).

Since this holds for every \(\mu\in\mathcal K_K\), and
\(\mathcal R_{\alpha,K}(X)=\inf_{\mu\in\mathcal K_K}M(\mu)\), we obtain
\begin{equation}
\label{eq:bmuon-cost-lower-bound}
\mathcal R_{\alpha,K}(X)
\ge
\max\left\{
\sigma_\star,\,
\left(
\frac{\sum_{\ell,i}\sigma_{\ell,i}^q}{K}
\right)^{1/q}
\right\}.
\end{equation}

We now prove the reverse inequality. Set
$
T
:=
\max\left\{
\sigma_\star,\,
\left(
\frac{\sum_{\ell,i}\sigma_{\ell,i}^q}{K}
\right)^{1/q}
\right\}.
$
For each block \(\ell\), set
\(\mu_i^\ell=(\sigma_{\ell,i}/T)^q\) for \(1\le i\le r_\ell\),
and set the remaining coordinates to zero. Since \(T\ge\sigma_\star\),
all these coordinates lie in \([0,1]\). Moreover,
\[
\sum_{\ell=1}^L\sum_{i=1}^{r_\ell}\mu_i^\ell
=
\frac{\sum_{\ell,i}\sigma_{\ell,i}^q}{T^q}
\le K.
\]
Let \(S:=\sum_{\ell=1}^L\sum_{i=1}^{r_\ell}\mu_i^\ell\). The current
allocation has total mass \(S\), so the amount still to be added to
obtain an allocation in \(\mathcal K_K\) is \(K-S\). The total mass
that can be added without violating the constraint
\(\mu_j^\ell\le1\) is
$
\sum_{\ell=1}^L\sum_{j=1}^{m_\ell}(1-\mu_j^\ell)
=
N-S.
$
Since \(K\le N\), we have \(K-S\le N-S\). Hence we can increase the
coordinates, without making any of them exceed \(1\), until the total
mass becomes \(K\). Denote the resulting allocation by
\(\widetilde\mu\). Then \(\widetilde\mu\in\mathcal K_K\), and since
\(\widetilde\mu_j^\ell\ge\mu_j^\ell\) for every \(\ell,j\),
\[
\widetilde\mu_{(i)}^\ell
\ge
\mu_{(i)}^\ell
=
\left(\frac{\sigma_{\ell,i}}{T}\right)^q,
\qquad
1\le i\le r_\ell.
\]

We have therefore constructed an allocation
\(\widetilde\mu\in\mathcal K_K\) such that
\[
M(\widetilde\mu)
=
\max_{1\le\ell\le L}
\max_{1\le i\le r_\ell}
\frac{\sigma_{\ell,i}}
     {(\widetilde\mu_{(i)}^\ell)^{2\alpha}}
\le T.
\]
Hence
\begin{equation}
\label{eq:bmuon-cost-upper-bound}
\mathcal R_{\alpha,K}(X)\le T.
\end{equation}

By Eq.~(\ref{eq:bmuon-cost-lower-bound}) and Eq.~(\ref{eq:bmuon-cost-upper-bound}),
\[
\mathcal R_{\alpha,K}(X)
=
\max\left\{
\sigma_\star,\,
\left(
\frac{\sum_{\ell,i}\sigma_{\ell,i}^q}{K}
\right)^{1/q}
\right\}.
\]

Since
\(\sum_{\ell,i}\sigma_{\ell,i}^q
=\sigma_\star^q d_{\mathrm{eff}}^{(\alpha)}(X)\)
and \(1/q=2\alpha\), we obtain
\[
\mathcal R_{\alpha,K}(X)
=
\sigma_\star
\max\left\{
1,\,
\left(
\frac{d_{\mathrm{eff}}^{(\alpha)}(X)}{K}
\right)^{2\alpha}
\right\}.
\]

We now prove attainment. Let \(\widetilde\mu\in\mathcal K_K\) be the
allocation constructed above. By
Lemma~\ref{lem:fixed-allocation-factorization}, for every block
\(\ell\), there exists an exact factorization \(W^{\ell,\star}\)
attaining
$$
\inf_{W_2^\ell W_1^\ell=X_\ell}
\max\left\{
\|D_{\widetilde\mu^\ell}^{-\alpha}W_1^\ell\|_{\mathrm{op}}^2,\,
\|W_2^\ell D_{\widetilde\mu^\ell}^{-\alpha}\|_{\mathrm{op}}^2
\right\}.
$$
Let
\(W^\star:=\bigl(W^{\ell,\star}\bigr)_{\ell=1}^L\),
then \(W^\star\) is an exact factorization and
\[
\rho_{\widetilde\mu}(W^\star)^2
=
\max_{1\le\ell\le L}
\max_{1\le i\le r_\ell}
\frac{\sigma_{\ell,i}}
     {(\widetilde\mu_{(i)}^\ell)^{2\alpha}}
\le T
=
\mathcal R_{\alpha,K}(X).
\]
Since
\(\Omega_{\alpha,K}(W^\star)
=\min_{\nu\in\mathcal K_K}\rho_\nu(W^\star)\) (Lemma~\ref{lem:single-allocation}),
\[
\mathcal R_{\alpha,K}(X)
\le
\Omega_{\alpha,K}(W^\star)^2
\le
\rho_{\widetilde\mu}(W^\star)^2
\le
\mathcal R_{\alpha,K}(X).
\]
Hence all these terms are equal, and \(W^\star\) attains the infimum
defining \(\mathcal R_{\alpha,K}(X)\).

We now prove intrinsic-width recovery. Assume
\(K\le d_{\mathrm{eff}}^{(\alpha)}(X)\). Then
$
\mathcal R_{\alpha,K}(X)
=
\left(
\frac{\sum_{\ell,i}\sigma_{\ell,i}^q}{K}
\right)^{1/q}.
$

Let \(W^\star\) be any minimizer of
\(\mathcal R_{\alpha,K}(X)\). By
Lemma~\ref{lem:single-allocation}, there exists
\(\mu^\star\in\mathcal K_K\) such that
$
\Omega_{\alpha,K}(W^\star)
=
\rho_{\mu^\star}(W^\star).
$
Since \(W^\star\) is a minimizer,
$
\rho_{\mu^\star}(W^\star)^2
=
\mathcal R_{\alpha,K}(X).
$ By Lemma~\ref{lem:fixed-allocation-factorization} applied to each block $\ell$,
\[
\max_{1\le i\le r_\ell}
\frac{\sigma_{\ell,i}}
     {(\mu_{(i)}^{\ell,\star})^{2\alpha}}
\le
\mathcal R_{\alpha,K}(X).
\]
Thus, for every \(i\le r_\ell\),
\[
\mu_{(i)}^{\ell,\star}
\ge
\left(
\frac{\sigma_{\ell,i}}
     {\mathcal R_{\alpha,K}(X)}
\right)^q
.
\]
Therefore
\[
K
=
\sum_{\ell,i}\left(
\frac{\sigma_{\ell,i}}
     {\mathcal R_{\alpha,K}(X)}
\right)^q
\le
\sum_{\ell=1}^L\sum_{i=1}^{r_\ell}
\mu_{(i)}^{\ell,\star}
\le
\sum_{\ell=1}^L\sum_{j=1}^{m_\ell}
\mu_j^{\ell,\star}
=
K.
\]
Hence all the terms are equal and \(\mu_{(j)}^{\ell,\star}=0\) for every \(j>r_\ell\). Since
\(\rho_{\mu^\star}(W^\star)<\infty\), every channel with zero
allocation has both its corresponding row of \(W_1^{\ell,\star}\) and
column of \(W_2^{\ell,\star}\) equal to zero (convention of
Lemma~\ref{lem:single-allocation}). Therefore
\[
a_\ell(W^\star)\le r_\ell.
\]

On the other hand,
$
X_\ell
=
\sum_{j:\,\text{channel }j\text{ active}}
W_{2,:,j}^{\ell,\star}W_{1,j:}^{\ell,\star}
$
is a sum of at most \(a_\ell(W^\star)\) rank-one matrices. Hence
$
r_\ell
=
\operatorname{rank}(X_\ell)
\le
a_\ell(W^\star).
$
Thus
\[
a_\ell(W^\star)=r_\ell
\qquad
\text{for every }\ell,
\]
which completes the proof.
\end{proof}

\subsection{BMuon with Weight-Decay}
\label{app:weight-decay}

We show in this section that, when BMuon is combined with decoupled weight
decay, its channel allocations have a direct effect on the learned weights.
Consider the BMuon oracle with fixed budget \(K\) and decoupled weight decay,
\[
W_{i,t+1}^\ell
=
(1-\eta_t\lambda)W_{i,t}^\ell
-\eta_t\Delta_{i,t}^{\ell,\star},
\]
where \(\lambda>0\), \(\eta_t\ge0\), and \(0\le\lambda\eta_t<1\).
At each step \(t\), let \(\mu_t^\star\in\mathcal K_K\) denote a solution of
Eq.~(\ref{eq:bmuon-allocation}). Then, by Eq.~(\ref{eq:bmuon-update}),
\[
\Delta_{1,t}^{\ell,\star}
=
D_{\mu_t^{\ell,\star}}^\alpha
\Polar\!\left(
D_{\mu_t^{\ell,\star}}^\alpha M_{1,t}^\ell
\right),
\qquad
\Delta_{2,t}^{\ell,\star}
=
\Polar\!\left(
M_{2,t}^\ell D_{\mu_t^{\ell,\star}}^\alpha
\right)
D_{\mu_t^{\ell,\star}}^\alpha.
\]

For a channel \(c=(\ell,j)\), write
\(\mu_{c,t}:=\mu_{j,t}^{\ell,\star}\) and
\(r_{c,t}:=
(\|W_{1,t,j:}^\ell\|_2^2+\|W_{2,t,:,j}^\ell\|_2^2)^{1/2}\).
For \(0\le s\le T\), define
\(D_{s,T}:=\prod_{u=s}^{T-1}(1-\lambda\eta_u)\), with \(D_{T,T}:=1\).

\begin{theorem}[Allocation controls channel norms]
\label{thm:weight-decay-app}
Let \(0<\alpha\le1\). For every channel \(c\) and every \(T\ge1\),
\begin{equation}
\label{eq:weight-decay-channel-bound}
r_{c,T}
\le
D_{0,T}r_{c,0}
+
\sqrt2
\sum_{t=0}^{T-1}
\eta_tD_{t+1,T}\mu_{c,t}^{\alpha}.
\end{equation}
Moreover, writing \(r_T:=(r_{c,T})_{c=1}^N\),
\begin{equation}
\label{eq:weight-decay-global-bound}
\|r_T\|_{1/\alpha}
\le
D_{0,T}\|r_0\|_{1/\alpha}
+
\frac{\sqrt2K^\alpha}{\lambda}(1-D_{0,T}).
\end{equation}
If \(\sum_{t=0}^\infty\eta_t=\infty\), then
\begin{equation}
\label{eq:weight-decay-global-asymptotic}
\limsup_{T\to\infty}\|r_T\|_{1/\alpha}
\le
\frac{\sqrt2K^\alpha}{\lambda}.
\end{equation}
\end{theorem}

\begin{proof}
For \(c=(\ell,j)\), let
\(P_{1,t}^\ell
=\Polar(D_{\mu_t^{\ell,\star}}^\alpha M_{1,t}^\ell)\) and
\(P_{2,t}^\ell
=\Polar(M_{2,t}^\ell D_{\mu_t^{\ell,\star}}^\alpha)\).
Since
\(\|P_{1,t}^\ell\|_{\mathrm{op}}\le1\) and
\(\|P_{2,t}^\ell\|_{\mathrm{op}}\le1\),
\[
\|\Delta_{1,t,j:}^{\ell,\star}\|_2
\le\mu_{c,t}^\alpha,
\qquad
\|\Delta_{2,t,:,j}^{\ell,\star}\|_2
\le\mu_{c,t}^\alpha.
\]
Using the weight-decay update and the triangle inequality,
\begin{align}
r_{c,t+1}
&=
\Big(
\|(1-\lambda\eta_t)W_{1,t,j:}^\ell
-\eta_t\Delta_{1,t,j:}^{\ell,\star}\|_2^2
+
\|(1-\lambda\eta_t)W_{2,t,:,j}^\ell
-\eta_t\Delta_{2,t,:,j}^{\ell,\star}\|_2^2
\Big)^{1/2}
\nonumber\\
&\le
(1-\lambda\eta_t)
\Big(
\|W_{1,t,j:}^\ell\|_2^2
+
\|W_{2,t,:,j}^\ell\|_2^2
\Big)^{1/2}
+
\eta_t
\Big(
\|\Delta_{1,t,j:}^{\ell,\star}\|_2^2
+
\|\Delta_{2,t,:,j}^{\ell,\star}\|_2^2
\Big)^{1/2}
\nonumber\\
&\le
(1-\lambda\eta_t)r_{c,t}
+
\sqrt{2}\,\eta_t\mu_{c,t}^\alpha.
\label{eq:weight-decay-channel-recursion}
\end{align}
Iterating Eq.~(\ref{eq:weight-decay-channel-recursion}) gives
Eq.~(\ref{eq:weight-decay-channel-bound}).

Set \(p:=1/\alpha\). Applying the \(\ell_p\) norm to the componentwise
inequality in Eq.~(\ref{eq:weight-decay-channel-recursion}) and using the
triangle inequality,
\[
\|r_{t+1}\|_p
\le
(1-\lambda\eta_t)\|r_t\|_p
+
\sqrt2\,\eta_t\|\mu_t^\alpha\|_p.
\]
Since \(p=1/\alpha\) and \(\sum_c\mu_{c,t}=K\),
\[
\|\mu_t^\alpha\|_{1/\alpha}
=
\left(\sum_c\mu_{c,t}\right)^\alpha
=
K^\alpha.
\]

Iterating this inequality and using
\begin{equation}
\label{eq:weight-decay-telescope}
\lambda
\sum_{t=0}^{T-1}\eta_tD_{t+1,T}
=
1-D_{0,T}
\end{equation}
gives Eq.~(\ref{eq:weight-decay-global-bound}).

Finally,
\(D_{0,T}\le
\exp(-\lambda\sum_{t=0}^{T-1}\eta_t)\).
Hence \(\sum_{t=0}^\infty\eta_t=\infty\) implies \(D_{0,T}\to0\), and
Eq.~(\ref{eq:weight-decay-global-asymptotic}) follows.
\end{proof}

The channelwise bound also shows that sustained under-allocation suppresses
the corresponding channel.

\begin{corollary}[Persistent under-allocation]
\label{cor:weight-decay-suppression}
Assume \(\sum_{t=0}^\infty\eta_t=\infty\). If, for some channel \(c\),
there exist \(t_0\ge0\) and \(\bar\mu_c\in[0,1]\) such that
\(\mu_{c,t}\le\bar\mu_c\) for every \(t\ge t_0\), then
\begin{equation}
\label{eq:weight-decay-underallocation}
\limsup_{T\to\infty}r_{c,T}
\le
\frac{\sqrt2}{\lambda}\bar\mu_c^\alpha.
\end{equation}
In particular, if \(\mu_{c,t}\to0\), then \(r_{c,t}\to0\).
\end{corollary}

\begin{proof}
Applying Eq.~(\ref{eq:weight-decay-channel-bound}) from \(t_0\) gives
\[
r_{c,T}
\le
D_{t_0,T}r_{c,t_0}
+
\sqrt2
\sum_{t=t_0}^{T-1}
\eta_tD_{t+1,T}\mu_{c,t}^{\alpha}.
\]
Since \(\mu_{c,t}\le\bar\mu_c\) for every \(t\ge t_0\),
\[
r_{c,T}
\le
D_{t_0,T}r_{c,t_0}
+
\sqrt2\,\bar\mu_c^\alpha
\sum_{t=t_0}^{T-1}
\eta_tD_{t+1,T}.
\]
Using Eq.~(\ref{eq:weight-decay-telescope}) from \(t_0\),
\[
r_{c,T}
\le
D_{t_0,T}r_{c,t_0}
+
\frac{\sqrt2\,\bar\mu_c^\alpha}{\lambda}
(1-D_{t_0,T}).
\]

Since
\(D_{t_0,T}
\le
\exp(-\lambda\sum_{t=t_0}^{T-1}\eta_t)\to0\), we get Eq.~(\ref{eq:weight-decay-underallocation}).

If \(\mu_{c,t}\to0\), then for every \(\varepsilon>0\) there exists
\(t_0\) such that \(\mu_{c,t}^\alpha\le\varepsilon\) for every
\(t\ge t_0\). Hence
\(\limsup_{T\to\infty}r_{c,T}\le\sqrt2\,\varepsilon/\lambda\).
Since this holds for every $\varepsilon > 0$ and $r_{c,T} \ge0$, we get \(r_{c,T}\to0\).
\end{proof}

Thus, weight decay makes BMuon's allocation persistent in parameter space: channel norms are controlled by their discounted allocation histories, implying that channels receiving vanishing allocations are driven toward zero.

\newpage

\section{MM Allocation Update}
\label{app:mm}

This appendix provides the derivations and proofs underlying the MM allocation
update of Section~\ref{sec:practical}. Appendix~\ref{app:mm-solver} proves the tight
minorization in Eq.~(\ref{eq:mm-surrogate}) and establishes the sign properties of
the allocation scores $a_j^\ell$. Appendix~\ref{app:mm-kkt} solves the
surrogate maximization problem and derives the closed-form update in
Eq.~(\ref{eq:mm-closed-form}), including the determination of $\lambda_\mu$ by
bisection. Appendix~\ref{app:mm-convergence} proves monotonicity of repeated
MM updates and, for $0<\alpha\le\tfrac12$, convergence in objective value to
the global optimum over $\mathcal K_{K,\varepsilon}$ together with global
optimality of every accumulation point. Appendix~\ref{app:mm-floor} bounds
the effect of replacing the original feasible set $\mathcal K_K$ by
$\mathcal K_{K,\varepsilon}$. Appendix~\ref{app:oracle_approximation}
quantifies the two approximations used in practice: performing a single MM
update at each allocation refresh and keeping the allocation fixed between
refreshes, and compares the practical update with a certified allocation
solve end-to-end. Finally, Appendix~\ref{app:overhead} characterizes the
computational overhead of the practical BMuon update and reports training-time
measurements for ImageNet-1K and 1B-parameter LLaMA.

Throughout this appendix, the momenta
$M=\{M_1^\ell,M_2^\ell\}_{\ell=1}^L$ are fixed and all polar factors are
exact. There are $N$ globally budgeted channels, indexed by $(\ell,j)$, and
\[
\mathcal K_K
:=
\left\{
\mu\in[0,1]^N:
\sum_{\ell,j}\mu_j^\ell=K
\right\},
\qquad
\mathcal K_{K,\varepsilon}
:=
\left\{
\mu\in[\varepsilon,1]^N:
\sum_{\ell,j}\mu_j^\ell=K
\right\},
\]
where $0<K<N$ and $0<\varepsilon<K/N$. For fixed $M$, the oracle allocation
objective of Eq.~(\ref{eq:bmuon-allocation}) is
\[
\Phi_M(\mu)
=
\sum_{\ell=1}^L
\left[
\|D_{\mu^\ell}^{\alpha}M_1^\ell\|_*
+
\|M_2^\ell D_{\mu^\ell}^{\alpha}\|_*
\right],
\qquad 0<\alpha<1.
\]
Given $\bar\mu\in\mathcal K_{K,\varepsilon}$, the MM surrogate is
\[
Q(\mu\mid\bar\mu)
=
\sum_{\ell,j}
a_j^\ell(\bar\mu)(\mu_j^\ell)^\alpha,
\]
where
\[
a_j^\ell(\bar\mu)
=
\langle M_{1,j:}^\ell,\bar P_{1,j:}^\ell\rangle
+
\langle M_{2,:,j}^\ell,\bar P_{2,:,j}^\ell\rangle,
\]
with
$\bar P_1^\ell=\Polar(D_{\bar\mu^\ell}^\alpha M_1^\ell)$ and
$\bar P_2^\ell=\Polar(M_2^\ell D_{\bar\mu^\ell}^\alpha)$.
One MM step maximizes $Q(\cdot\mid\bar\mu)$ over
$\mathcal K_{K,\varepsilon}$.

\subsection{Minorization and allocation scores}
\label{app:mm-solver}

We first prove the tight minorization in Eq.~(\ref{eq:mm-surrogate}). Fix an
allocation $\bar\mu\in\mathcal K_{K,\varepsilon}$ and, for each block, let
\[
\bar P_1^\ell
=
\Polar(D_{\bar\mu^\ell}^\alpha M_1^\ell),
\qquad
\bar P_2^\ell
=
\Polar(M_2^\ell D_{\bar\mu^\ell}^\alpha),
\]
as in Section~\ref{sec:practical}. By spectral--nuclear duality,
$\|A\|_*=\max_{\|P\|_{\mathrm{op}}\le1}\langle A,P\rangle$.

We will use the following equality characterization both here and in the convergence analysis.

\begin{lemma}
\label{lem:polar-equality}
Let $\|P\|_{\mathrm{op}}\le1$. If
$\langle A,P\rangle=\|A\|_*$, then
\[
AP^\top=(AA^\top)^{1/2},
\qquad
P^\top A=(A^\top A)^{1/2}.
\]
\end{lemma}

\begin{proof}
Let $A=U\Sigma V^\top$ be a compact SVD. Then
\[
\langle A,P\rangle
=
\sum_r \sigma_r u_r^\top P v_r,
\qquad
|u_r^\top P v_r|\le1.
\]
Since $\sigma_r>0$, equality
$\langle A,P\rangle=\sum_r\sigma_r$ implies
$u_r^\top P v_r=1$ for every $r$. Moreover,
$\|Pv_r\|_2\le1$, so equality in Cauchy--Schwarz gives
$Pv_r=u_r$. Similarly, $\|P^\top u_r\|_2\le1$ gives
$P^\top u_r=v_r$. Hence $PV=U$ and $P^\top U=V$, and therefore
\[
AP^\top=U\Sigma U^\top=(AA^\top)^{1/2},
\qquad
P^\top A=V\Sigma V^\top=(A^\top A)^{1/2}.
\]
\end{proof}

Since $\|\bar P_1^\ell\|_{\mathrm{op}}\le1$ and
$\|\bar P_2^\ell\|_{\mathrm{op}}\le1$, for any allocation $\mu$,
\begin{align}
\Phi_M(\mu)
&=
\sum_{\ell=1}^L
\left[
\|D_{\mu^\ell}^\alpha M_1^\ell\|_*
+
\|M_2^\ell D_{\mu^\ell}^\alpha\|_*
\right]
\notag\\
&\ge
\sum_{\ell=1}^L
\left[
\left\langle
D_{\mu^\ell}^\alpha M_1^\ell,\bar P_1^\ell
\right\rangle
+
\left\langle
M_2^\ell D_{\mu^\ell}^\alpha,\bar P_2^\ell
\right\rangle
\right]
\notag\\
&=
\sum_{\ell,j}
\left(
\langle M_{1,j:}^\ell,\bar P_{1,j:}^\ell\rangle
+
\langle M_{2,:,j}^\ell,\bar P_{2,:,j}^\ell\rangle
\right)
(\mu_j^\ell)^\alpha
=
Q(\mu\mid\bar\mu).
\label{eq:mm-minorization-proof}
\end{align}
At $\mu=\bar\mu$, using
$\langle A,\Polar(A)\rangle=\|A\|_*$,
$
Q(\bar\mu\mid\bar\mu)
=
\sum_{\ell=1}^L
\left[
\|D_{\bar\mu^\ell}^\alpha M_1^\ell\|_*
+
\|M_2^\ell D_{\bar\mu^\ell}^\alpha\|_*
\right]
=
\Phi_M(\bar\mu).
$
Hence the minorization is tight, proving Eq.~(\ref{eq:mm-surrogate}).

We now establish the sign properties of the scores.

Fix $(\ell,j)$ and consider first the row-side term. Write
$A=D_{\bar\mu^\ell}^\alpha M_1^\ell$,
$\bar P_1^\ell=\Polar(A)$, and
$S=(AA^\top)^{1/2}$. Since $\bar P_1^\ell$ attains the
spectral--nuclear dual maximum, Lemma~\ref{lem:polar-equality} gives
$A(\bar P_1^\ell)^\top=S\succeq0$,
\[
S_{jj}
=
(\bar\mu_j^\ell)^\alpha
\langle M_{1,j:}^\ell,\bar P_{1,j:}^\ell\rangle
\ge0.
\]
Since $\bar\mu\in\mathcal K_{K,\varepsilon}$, we have
$\bar\mu_j^\ell\ge\varepsilon>0$. Applying the same argument to the column-side term
gives
\begin{equation}
\begin{aligned}
\langle M_{1,j:}^\ell,\bar P_{1,j:}^\ell\rangle
&=
\frac{
\left[
\left(
D_{\bar\mu^\ell}^\alpha
M_1^\ell(M_1^\ell)^\top
D_{\bar\mu^\ell}^\alpha
\right)^{1/2}
\right]_{jj}
}{
(\bar\mu_j^\ell)^\alpha
},
\\
\langle M_{2,:,j}^\ell,\bar P_{2,:,j}^\ell\rangle
&=
\frac{
\left[
\left(
D_{\bar\mu^\ell}^\alpha
(M_2^\ell)^\top M_2^\ell
D_{\bar\mu^\ell}^\alpha
\right)^{1/2}
\right]_{jj}
}{
(\bar\mu_j^\ell)^\alpha
}.
\end{aligned}
\label{eq:mm-score-identities}
\end{equation}
Both numerators in Eq.~(\ref{eq:mm-score-identities}) are diagonal entries of
positive semidefinite matrices. Hence both score contributions are
nonnegative, and therefore
\[
a_j^\ell
=
\langle M_{1,j:}^\ell,\bar P_{1,j:}^\ell\rangle
+
\langle M_{2,:,j}^\ell,\bar P_{2,:,j}^\ell\rangle
\ge0.
\]

We next establish strict positivity when all channels are nontrivial.

If $M_{1,j:}^\ell\neq0$, then
\[
(S^2)_{jj}
=
(AA^\top)_{jj}
=
(\bar\mu_j^\ell)^{2\alpha}
\|M_{1,j:}^\ell\|_2^2
>0.
\]
Since $S$ is symmetric positive semidefinite,
$S^2\preceq\|S\|_{\mathrm{op}}S$. Hence
$
(S^2)_{jj}
\le
\|S\|_{\mathrm{op}}S_{jj},
$
so $(S^2)_{jj}>0$ implies $S_{jj}>0$. Since
$\bar\mu_j^\ell>0$, we obtain
$\langle M_{1,j:}^\ell,\bar P_{1,j:}^\ell\rangle>0$.

The same argument applied to the column-side term shows that
$M_{2,:,j}^\ell\neq0$ implies
$\langle M_{2,:,j}^\ell,\bar P_{2,:,j}^\ell\rangle>0$.
Therefore, if
$M_{1,j:}^\ell\neq0$ or $M_{2,:,j}^\ell\neq0$, then
$a_j^\ell>0$. In particular, under the nontrivial-channel condition
(i.e., $M_{1,j:}^\ell\neq0$ or $M_{2,:,j}^\ell\neq0$ for every $(\ell,j)$), all scores are strictly positive for every
$\bar\mu\in\mathcal K_{K,\varepsilon}$.

\subsection{Closed-form surrogate maximization}
\label{app:mm-kkt}

We now derive the closed-form update in Eq.~(\ref{eq:mm-closed-form}). Under the
nontrivial-channel condition, Appendix~\ref{app:mm-solver} gives
$a_j^\ell>0$ for every $(\ell,j)$. Since $0<\alpha<1$, each function
$\mu_j^\ell\mapsto a_j^\ell(\mu_j^\ell)^\alpha$ is strictly concave on
$[\varepsilon,1]$. Hence $Q(\mu\mid\bar\mu)$ is strictly concave, and its
maximization over the convex set $\mathcal K_{K,\varepsilon}$ has a unique
maximizer, which we denote by $\mu^+$.

Since $0<\varepsilon<K/N<1$, the uniform allocation
$\mu_j^\ell=K/N$ satisfies the budget constraint and both bound constraints strictly. Hence Slater's condition holds, and the KKT conditions are necessary and sufficient. Introduce a multiplier $\lambda_\mu$ for the budget constraint
and multipliers $\rho_j^\ell,\sigma_j^\ell\ge0$ for the lower and upper bounds,
respectively. The stationarity and complementary-slackness conditions are
\begin{align}
\alpha a_j^\ell(\mu_j^\ell)^{\alpha-1}
-\lambda_\mu+\rho_j^\ell-\sigma_j^\ell &=0,
\label{eq:mm-kkt-stationarity}\\
\rho_j^\ell(\mu_j^\ell-\varepsilon)&=0,
\qquad
\sigma_j^\ell(1-\mu_j^\ell)=0.
\label{eq:mm-kkt-complementarity}
\end{align}

Because $K<N$, at least one coordinate of any feasible allocation is strictly
smaller than $1$. For such a coordinate $\sigma_j^\ell=0$, and
Eq.~(\ref{eq:mm-kkt-stationarity}) implies $\lambda_\mu>0$.

For a fixed $(\ell,j)$, define
\[
\widehat\mu_j^\ell
=
\left(
\frac{\alpha a_j^\ell}{\lambda_\mu}
\right)^{1/(1-\alpha)}.
\]
If $\varepsilon<\mu_j^\ell<1$, both bound multipliers vanish and
$\mu_j^\ell=\widehat\mu_j^\ell$. If $\mu_j^\ell=\varepsilon$, then
$\sigma_j^\ell=0$ and stationarity gives
$\lambda_\mu\ge\alpha a_j^\ell\varepsilon^{\alpha-1}$, equivalently
$\widehat\mu_j^\ell\le\varepsilon$. Similarly, if $\mu_j^\ell=1$, then
$\rho_j^\ell=0$ and stationarity gives
$\lambda_\mu\le\alpha a_j^\ell$, equivalently
$\widehat\mu_j^\ell\ge1$. Hence the unique maximizer is
\[
(\mu_j^\ell)^+
=
\operatorname{clip}_{[\varepsilon,1]}
\left[
\left(
\frac{\alpha a_j^\ell}{\lambda_\mu}
\right)^{1/(1-\alpha)}
\right],
\]
which proves Eq.~(\ref{eq:mm-closed-form}).

It remains to determine $\lambda_\mu$. Define
\[
H(\lambda)
:=
\sum_{\ell,j}
\operatorname{clip}_{[\varepsilon,1]}
\left[
\left(
\frac{\alpha a_j^\ell}{\lambda}
\right)^{1/(1-\alpha)}
\right],
\qquad \lambda>0.
\]
Since every $a_j^\ell>0$, $H$ is continuous and nonincreasing, with
$H(\lambda)\to N$ as $\lambda\downarrow0$ and
$H(\lambda)\to N\varepsilon$ as $\lambda\uparrow\infty$.
Because $N\varepsilon<K<N$, there exists
$\lambda_\mu>0$ such that $H(\lambda_\mu)=K$, and such a value can be found by bisection. This completes the derivation of the MM allocation update.

\subsection{Convergence of repeated MM updates}
\label{app:mm-convergence}

We now prove the convergence statement of Section~\ref{sec:practical}.

\begin{theorem}[Convergence of repeated MM updates]
\label{thm:mm-convergence}
Assume the nontrivial-channel condition
(i.e., $M_{1,j:}^\ell\neq0$ or $M_{2,:,j}^\ell\neq0$ for every $(\ell,j)$),
and let $0<\alpha\le\tfrac12$. Starting from any
$\mu^{(0)}\in\mathcal K_{K,\varepsilon}$, define
\[
\mu^{(t+1)}
=
\arg\max_{\mu\in\mathcal K_{K,\varepsilon}}
Q(\mu\mid\mu^{(t)}).
\]
Then
\[
\Phi_M(\mu^{(t)})
\longrightarrow
\Phi_\varepsilon^\star
:=
\max_{\mu\in\mathcal K_{K,\varepsilon}}\Phi_M(\mu),
\]
and every accumulation point of $\{\mu^{(t)}\}_{t\ge0}$ is a global
maximizer of $\Phi_M$ over $\mathcal K_{K,\varepsilon}$.
\end{theorem}

\begin{proof}
By Appendix~\ref{app:mm-solver}, the surrogate is a tight minorizer:
$Q(\mu\mid\bar\mu)\le\Phi_M(\mu)$ for every $\mu$, with equality at
$\mu=\bar\mu$. Since $\mu^{(t+1)}$ maximizes
$Q(\cdot\mid\mu^{(t)})$,
\[
\Phi_M(\mu^{(t+1)})
\ge
Q(\mu^{(t+1)}\mid\mu^{(t)})
\ge
Q(\mu^{(t)}\mid\mu^{(t)})
=
\Phi_M(\mu^{(t)}).
\]
Hence $(\Phi_M(\mu^{(t)}))_t$ is nondecreasing. Since
$\mathcal K_{K,\varepsilon}$ is compact and $\Phi_M$ is continuous,
$\Phi_M(\mu^{(t)})$ converges.

Define the MM map
$
T(\bar\mu)
:=
\arg\max_{\mu\in\mathcal K_{K,\varepsilon}}
Q(\mu\mid\bar\mu).
$
By Appendix~\ref{app:mm-kkt}, under the nontrivial-channel condition the
surrogate has a unique maximizer. Thus $T$ is well defined.

We next show that $T$ is continuous. By Eq.~(\ref{eq:mm-score-identities}),
\[
a_j^\ell(\bar\mu)
=
\frac{
\left[
\left(
D_{\bar\mu^\ell}^\alpha
M_1^\ell(M_1^\ell)^\top
D_{\bar\mu^\ell}^\alpha
\right)^{1/2}
\right]_{jj}
}{
(\bar\mu_j^\ell)^\alpha
}
+
\frac{
\left[
\left(
D_{\bar\mu^\ell}^\alpha
(M_2^\ell)^\top M_2^\ell
D_{\bar\mu^\ell}^\alpha
\right)^{1/2}
\right]_{jj}
}{
(\bar\mu_j^\ell)^\alpha
}.
\]
Since $\bar\mu\in\mathcal K_{K,\varepsilon}$, we have
$\bar\mu_j^\ell\ge\varepsilon>0$. Moreover,
$\bar\mu\mapsto D_{\bar\mu^\ell}^\alpha$ is continuous, and the matrix square
root is continuous on the set of positive semidefinite matrices. Therefore
$\bar\mu\mapsto a_j^\ell(\bar\mu)$ is continuous for every $(\ell,j)$.
Hence
$(\mu,\bar\mu)\mapsto Q(\mu\mid\bar\mu)
=\sum_{\ell,j}a_j^\ell(\bar\mu)(\mu_j^\ell)^\alpha$
is continuous on
$\mathcal K_{K,\varepsilon}\times\mathcal K_{K,\varepsilon}$.
Since $\mathcal K_{K,\varepsilon}$ is compact and
$Q(\cdot\mid\bar\mu)$ has a unique maximizer, Berge's maximum theorem,
together with uniqueness of the maximizer, implies that $T$ is continuous.

Since $\mathcal K_{K,\varepsilon}$ is compact, the sequence
$\{\mu^{(t)}\}$ has at least one accumulation point. Let $\mu^\star$ be any
accumulation point. There exists a
subsequence $\mu^{(t_k)}\to\mu^\star$. Since
$\Phi_M(\mu^{(t)})$ converges,
\[
\Phi_M(\mu^{(t_k+1)})
-
\Phi_M(\mu^{(t_k)})
\longrightarrow0.
\]
Using $\mu^{(t_k+1)}=T(\mu^{(t_k)})$ and continuity of $T$ and $\Phi_M$ gives
$
\Phi_M(T(\mu^\star))
=
\Phi_M(\mu^\star).
$

Suppose that $T(\mu^\star)\neq\mu^\star$. Since $T(\mu^\star)$ is the unique
maximizer of $Q(\cdot\mid\mu^\star)$,
$
Q(T(\mu^\star)\mid\mu^\star)
>
Q(\mu^\star\mid\mu^\star).
$
By minorization and tightness,
\[
\Phi_M(T(\mu^\star))
\ge
Q(T(\mu^\star)\mid\mu^\star)
>
Q(\mu^\star\mid\mu^\star)
=
\Phi_M(\mu^\star),
\]
which is a contradiction. Therefore
$T(\mu^\star)=\mu^\star$.

It remains to show that every fixed point is globally optimal. We first show
first-order tangency in directional derivatives.

For a function $F$, write
$F'(x;d):=\lim_{s\downarrow0}[F(x+sd)-F(x)]/s$
for its one-sided directional derivative.

By spectral--nuclear duality,
$\|D_{\mu^\ell}^\alpha M_1^\ell\|_*
=
\max_{\|P\|_{\mathrm{op}}\le1}
\langle D_{\mu^\ell}^\alpha M_1^\ell,P\rangle$.
For fixed $P$, the directional derivative at $\bar\mu$ in direction $d$ is
\[
\sum_j
\alpha(\bar\mu_j^\ell)^{\alpha-1}d_j^\ell
\langle M_{1,j:}^\ell,P_{j:}\rangle .
\]
Let $A=D_{\bar\mu^\ell}^\alpha M_1^\ell$. By
Lemma~\ref{lem:polar-equality}, every maximizing $P$ satisfies
$AP^\top=(AA^\top)^{1/2}$, so
$(\bar\mu_j^\ell)^\alpha
\langle M_{1,j:}^\ell,P_{j:}\rangle
=[(AA^\top)^{1/2}]_{jj}$.
Thus the directional derivative above is the same for every maximizing $P$.
Since the operator-norm unit ball is compact and the inner objective is
continuously differentiable in $\mu$ on the positive orthant, Danskin's
theorem, together with the identical argument for the column-side term, gives
\[
\Phi_M'(\bar\mu;d)
=
\sum_{\ell,j}
\alpha a_j^\ell(\bar\mu)
(\bar\mu_j^\ell)^{\alpha-1}d_j^\ell
=
Q'(\bar\mu\mid\bar\mu;d).
\]

Since $T(\mu^\star)=\mu^\star$, for every
$\mu\in\mathcal K_{K,\varepsilon}$,
\[
Q'(\mu^\star\mid\mu^\star;\mu-\mu^\star)\le0.
\]
By first-order tangency,
$\Phi_M'(\mu^\star;\mu-\mu^\star)\le0$.
For $0<\alpha\le\tfrac12$, $\Phi_M$ is concave by Theorem~\ref{thm:curvature}, and therefore
\[
\Phi_M(\mu)
\le
\Phi_M(\mu^\star)
+
\Phi_M'(\mu^\star;\mu-\mu^\star)
\le
\Phi_M(\mu^\star).
\]
Thus $\mu^\star$ is a global maximizer of $\Phi_M$ over
$\mathcal K_{K,\varepsilon}$.

Since $\mu^{(t_k)}\to\mu^\star$ and $\Phi_M$ is continuous,
$\Phi_M(\mu^{(t_k)})\to\Phi_M(\mu^\star)=\Phi_\varepsilon^\star$.
Because the full sequence $\Phi_M(\mu^{(t)})$ converges, its limit must also be
$\Phi_\varepsilon^\star$. Hence
$\Phi_M(\mu^{(t)})\to\Phi_\varepsilon^\star$.
\end{proof}

\subsection{Effect of the allocation floor}
\label{app:mm-floor}

We finally quantify the effect of replacing the original feasible set
$\mathcal K_K$ by $\mathcal K_{K,\varepsilon}$. Let
\[
\Phi^\star
:=
\max_{\mu\in\mathcal K_K}\Phi_M(\mu),
\qquad
\Phi_\varepsilon^\star
:=
\max_{\mu\in\mathcal K_{K,\varepsilon}}\Phi_M(\mu).
\]

\begin{proposition}[Effect of the allocation floor]
\label{prop:mm-floor}
Let $0<\alpha\le\tfrac12$ and $0<\varepsilon<K/N$. Then
\[
0
\le
\Phi^\star-\Phi_\varepsilon^\star
\le
\frac{\varepsilon N}{K}
\left[
\Phi^\star
-
\Phi_M\!\left(\frac{K}{N}\mathbf 1\right)
\right].
\]
\end{proposition}

\begin{proof}
Since
$\mathcal K_{K,\varepsilon}\subseteq\mathcal K_K$, we immediately have
$\Phi_\varepsilon^\star\le\Phi^\star$.

Let $\mu^\star\in\mathcal K_K$ be a maximizer of $\Phi_M$, and define
$\delta:=\varepsilon N/K$ and
$u:=(K/N)\mathbf 1$. Since $0<\varepsilon<K/N$, we have
$0<\delta<1$. Consider
$\widetilde\mu:=(1-\delta)\mu^\star+\delta u$.
For every $(\ell,j)$,
$\widetilde\mu_j^\ell\ge\delta K/N=\varepsilon$, while
$\sum_{\ell,j}\widetilde\mu_j^\ell=K$ and
$\widetilde\mu_j^\ell\le1$. Hence
$\widetilde\mu\in\mathcal K_{K,\varepsilon}$.

By Theorem~\ref{thm:curvature}, $\Phi_M$ is concave for
$0<\alpha\le\tfrac12$. Therefore
\[
\Phi_\varepsilon^\star
\ge
\Phi_M(\widetilde\mu)
\ge
(1-\delta)\Phi_M(\mu^\star)
+
\delta\Phi_M(u)
=
(1-\delta)\Phi^\star
+
\delta\Phi_M\!\left(\frac{K}{N}\mathbf 1\right).
\]
Rearranging and using $\delta=\varepsilon N/K$ gives
\[
\Phi^\star-\Phi_\varepsilon^\star
\le
\frac{\varepsilon N}{K}
\left[
\Phi^\star
-
\Phi_M\!\left(\frac{K}{N}\mathbf 1\right)
\right],
\]
which completes the proof.
\end{proof}

\subsection{Approximation of the allocation oracle}
\label{app:oracle_approximation}

As in Algorithm~\ref{alg:bmuon}, BMuon does not recompute the allocation
$\mu$ at every optimizer step. Instead, every five optimizer steps it performs
an \emph{allocation refresh}: $\mu$ is recomputed and then kept fixed for the
next four optimizer steps. The budget may be scheduled over training, so we write $K_t$ for the budget at
optimizer step $t$. At a refresh step $t$, the new allocation is computed under
the current budget $K_t$.

The practical implementation therefore introduces two approximations.
First, at each allocation refresh, we perform only one MM update instead of
solving the allocation problem to convergence.
Second, we refresh $\mu$ only once every five optimizer steps, rather than at
every optimizer step.
We measure these two errors separately, and then test whether replacing the
practical one-step update by a certified solve at each refresh changes the
training and compression results.

\paragraph{Two paired training runs.}
We use the same CIFAR-100 ViT and training setup as in
Appendix~\ref{app:cifar100-posttraining}, except that we train for 100 epochs.
As in our training experiments, we use $\alpha=1/2$.
For each seed, we run two paired versions of BMuon that differ only in how the
allocation $\mu$ is computed at an allocation refresh.
The MLPs contain $N=4608$ hidden channels in total.
The budget $K_t$ decreases linearly from $N$ to $0.3N$ during the first third
of training and is then kept fixed.
In this experiment, $K_t$ changes only at allocation refreshes, so it is
constant during the four intervening optimizer steps.
In both runs, $\mu$ is refreshed every five optimizer steps and is held fixed
between two consecutive refreshes.

The two runs differ only in how $\mu$ is computed at a refresh step $t$.
In the \textbf{1-step refresh run}, we start from the allocation carried from the
previous refresh, compute the MM scores using $\mathrm{NS}_5$, the five-step
Newton--Schulz approximation defined in Appendix~\ref{app:muon-algo}, and
perform one MM update under the current budget $K_t$.
In the \textbf{certified refresh run}, we instead use exact polar factors and
repeat the MM update until the stopping condition
Eq.~(\ref{eq:oracle-reference-stopping}) is satisfied.
In both runs, the allocation computed at step $t$ satisfies the current
constraint $\sum_i \mu_i=K_t$.

Apart from this difference, the two runs use the same BMuon weight update.
For each seed, they also use the same initialization, data order, augmentation
stream, and budget schedule.

\paragraph{Notation and certified error bound.}
Let $M_t$ denote all MLP momentum matrices at optimizer step $t$, and define
$\mathcal C_t=\{\mu\in[\varepsilon,1]^N:\mathbf 1^\top\mu=K_t\}$.
For $\alpha=\tfrac12$, the allocation objective is
\[
\Phi_M(\mu)
=
\sum_{\ell=1}^{L}
\left(
\left\|D_{\mu^{\ell}}^{1/2}M_1^{\ell}\right\|_*
+
\left\|M_2^{\ell}D_{\mu^{\ell}}^{1/2}\right\|_*
\right).
\]
We write
$\Phi_t^\star=\max_{\mu\in\mathcal C_t}\Phi_{M_t}(\mu)$.

For any feasible $\mu$, define
$g_t(\mu)=\max_{\nu\in\mathcal C_t}
\langle\nabla\Phi_{M_t}(\mu),\nu-\mu\rangle$.
We call $g_t(\mu)$ the Frank--Wolfe gap.
Since $\Phi_{M_t}$ is concave (Theorem~\ref{thm:curvature}), for every $\nu\in\mathcal C_t$,
$\Phi_{M_t}(\nu)
\leq
\Phi_{M_t}(\mu)
+
\langle\nabla\Phi_{M_t}(\mu),\nu-\mu\rangle$.
Maximizing over $\nu$ gives
$\Phi_t^\star-\Phi_{M_t}(\mu)\leq g_t(\mu)$, which implies in particular that $g_t(\mu)\geq 0$ for every feasible $\mu$.

At optimizer step $t$, let $\mu_t^{\mathrm{ref}}\in\mathcal C_t$ denote the
allocation obtained by repeatedly applying the exact-polar MM update over
$\mathcal C_t$ (Theorem~\ref{thm:mm-convergence} shows that the repeated
MM updates converge) until
\begin{equation}
    \frac{g_t(\mu_t^{\mathrm{ref}})}
    {\Phi_{M_t}(\mu_t^{\mathrm{ref}})}
    \leq 10^{-4}.
    \label{eq:oracle-reference-stopping}
\end{equation}

Define
$U_t=
\Phi_{M_t}(\mu_t^{\mathrm{ref}})
+
g_t(\mu_t^{\mathrm{ref}})$.
The previous inequality gives $U_t\geq\Phi_t^\star$.

For any feasible candidate $\widetilde\mu$, we therefore define
$d_t(\widetilde\mu)
=
1-
\Phi_{M_t}(\widetilde\mu)/U_t$.
Because $U_t\geq\Phi_t^\star$,
\[
1-\frac{\Phi_{M_t}(\widetilde\mu)}{\Phi_t^\star}
\leq
1-\frac{\Phi_{M_t}(\widetilde\mu)}{U_t}
=
d_t(\widetilde\mu).
\]
Thus $d_t(\widetilde\mu)$ is a certified upper bound on the relative loss in
the exact allocation objective incurred by using $\widetilde\mu$ instead of
an optimal allocation.

In particular, for $\mu_t^{\mathrm{ref}}$,
\[
d_t(\mu_t^{\mathrm{ref}})
=
\frac{g_t(\mu_t^{\mathrm{ref}})}
{\Phi_{M_t}(\mu_t^{\mathrm{ref}})
 + g_t(\mu_t^{\mathrm{ref}})}
\leq
\frac{g_t(\mu_t^{\mathrm{ref}})}
{\Phi_{M_t}(\mu_t^{\mathrm{ref}})}
\leq 10^{-4},
\]
where the last inequality follows from
Eq.~(\ref{eq:oracle-reference-stopping}).

Hence $\mu_t^{\mathrm{ref}}$ is certified to have relative objective loss at
most $10^{-4}$ from the exact optimum. We therefore use
$\mu_t^{\mathrm{ref}}$ as a certified approximation, in objective value, to
an optimal allocation at step $t$, and call it the \emph{reference
allocation}. In the \emph{certified refresh} run, at each refresh step $t$,
we use $\mu_t^{\mathrm{ref}}$ as the BMuon allocation and keep it fixed
until the next refresh.

For each run and seed, we perform the diagnostics described below at ten
refresh steps: the first refresh at which $K_t/N$ reaches approximately
$0.9,0.8,\ldots,0.3$, followed by refreshes near $50\%$, $75\%$, and $99\%$
of training.
We denote this set of ten refresh steps by $\mathcal D$.

\paragraph{The two approximation errors.}
Figure~\ref{fig:oracle_approximation} measures separately the two approximations
used by the practical implementation.

At each $t\in\mathcal D$, let $\widehat\mu_t$ denote the allocation obtained by
applying the practical update rule once: one MM step with scores computed using
$\mathrm{NS}_5$, starting from the allocation carried from the previous
refresh. Figure~\ref{fig:oracle_approximation}(a) reports
$d_t(\widehat\mu_t)$.

For the red curve, we apply the practical one-step allocation update to the
1-step refresh run at step $t$. For the blue curve, we apply the same update
to the certified refresh run at step $t$, only for this diagnostic; the
resulting allocation is not used to continue the certified refresh run.

At each sampled training stage $t\in\mathcal D$, the plotted value for each curve is the median over the three seeds, and the shaded region is their minimum--maximum range. The vertical dashed line marks the end of the budget decrease, after which $K_t/N=0.3$ is fixed.

The certified deficit of one practical allocation update remains of order
$10^{-4}$ to $10^{-3}$ over the sampled states.
Its largest sampled value is about $1.2\times10^{-3}$ on states from the
1-step refresh run and about $1.1\times10^{-3}$ on states from the certified
refresh run.
The similar behavior on the two trajectories shows that the accuracy of the
one-step rule is not specific to states produced by that same rule.

Figure~\ref{fig:oracle_approximation}(b) measures the second approximation:
keeping $\mu$ fixed between two refreshes.
We use the certified refresh run for this experiment.
For a diagnostic refresh $t\in\mathcal D$, BMuon uses
$\mu_t^{\mathrm{ref}}$ and keeps it fixed for the following five optimizer
steps.
For $s\in\{0,1,2,3,4\}$, we evaluate this same allocation under the current
momentum state $M_{t+s}$ and report
$d_{t+s}(\mu_t^{\mathrm{ref}})$.
This quantity upper-bounds the relative objective loss incurred by continuing
to use the allocation computed at step $t$, instead of recomputing an optimal
allocation for the current momentum state $M_{t+s}$.
Thus $s=0$ corresponds to the refresh step, while $s=4$ corresponds to the
last optimizer step before the next refresh.

To avoid hiding a difficult sampled window by averaging over different training
steps, Figure~\ref{fig:oracle_approximation}(b) first takes the worst case over
$t\in\mathcal D$.
For seed $j$ and offset $s$, define
$w_{j,s}
=
\max_{t\in\mathcal D}
d_{t+s}^{(j)}(\mu_t^{\mathrm{ref},(j)})$.
The plotted value at offset $s$ is the median of
$w_{0,s},w_{1,s},w_{2,s}$, and the shaded region is their minimum--maximum range.

Even under this worst-over-sampled-windows summary, the certified deficit
after four optimizer steps remains on the order of $10^{-3}$.
Thus, at the sampled states, keeping $\mu$ fixed for the four steps between
two refreshes produces only a small additional error in the allocation
objective.

\begin{figure}[t]
    \centering

    \begin{minipage}[t]{0.42\linewidth}
        \centering
        \includegraphics[width=\linewidth]
        {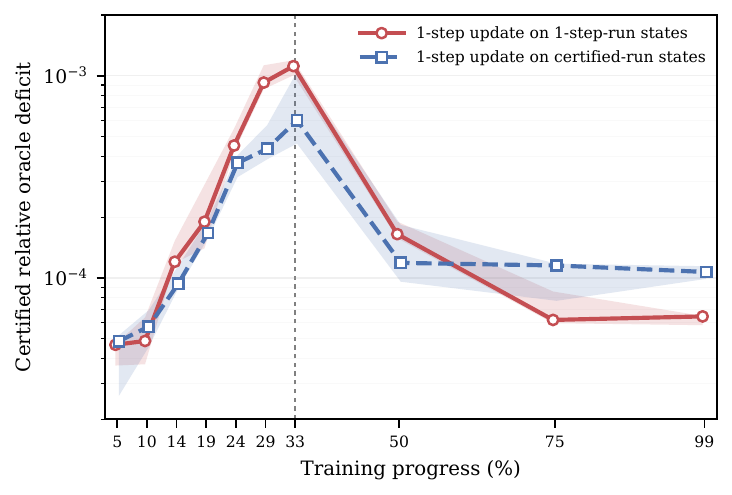}

        \vspace{1mm}
        \textbf{(a)} One allocation update.
    \end{minipage}
    \hspace{0.04\linewidth}
    \begin{minipage}[t]{0.42\linewidth}
        \centering
        \includegraphics[width=\linewidth]
        {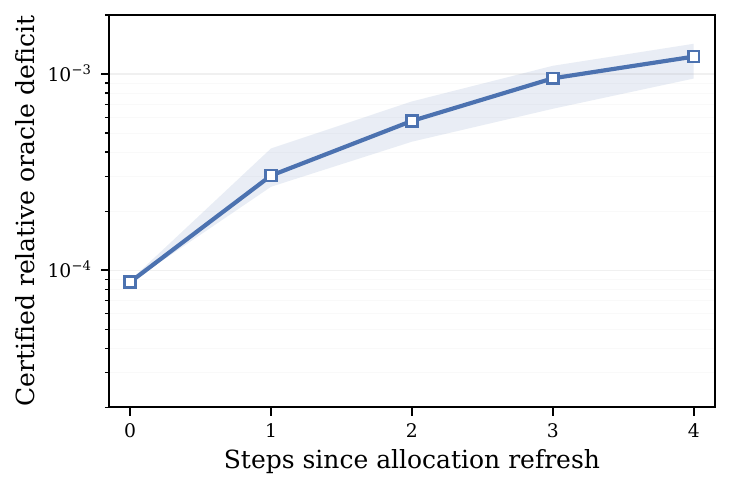}

        \vspace{1mm}
        \textbf{(b)} Allocation held between refreshes.
    \end{minipage}

    \caption{
    \textbf{Accuracy of the practical allocation approximation.}
    \textbf{(a)}
    Certified relative deficit of one practical
    $\mathrm{NS}_5$-scored MM update.
    The same update is evaluated on states generated by the 1-step refresh
    run and by the certified refresh run.
    At each training stage, the line is the median over three seeds and the
    shaded region is their minimum--maximum range.
    The vertical dashed line marks the beginning of the fixed-budget phase
    $K_t/N=0.3$.
    \textbf{(b)}
    Certified relative deficit caused by keeping the allocation installed by
    the certified refresh run fixed for $s=0,\ldots,4$ optimizer steps.
    For each seed and each $s$, we first take the largest deficit over the ten
    sampled windows; the line is the median of these three seed-wise maxima
    and the shaded region is their minimum--maximum range.
    }
    \label{fig:oracle_approximation}
\end{figure}

\paragraph{End-to-end effect.}
The previous experiments measure the approximation error in the allocation
objective. We now test whether this error changes the trained model.

Figure~\ref{fig:oracle_end_to_end} compares the two runs.
Recall that both refresh $\mu$ every five optimizer steps: the 1-step refresh
run uses one practical MM update, while the certified refresh run uses
$\mu_t^{\mathrm{ref}}$.

Figure~\ref{fig:oracle_end_to_end}(a) shows CIFAR-100 test accuracy during
training.
Figure~\ref{fig:oracle_end_to_end}(b) shows the accuracy of the final models
after structured MLP-channel pruning, without fine-tuning.
For each model, we rank all MLP channels globally by the corresponding entries
of its final allocation $\mu$ and remove those with the smallest values.
In both panels, the plotted value for each curve is the median over the three
seeds, and the shaded region is their minimum--maximum range.

Using $\mu_t^{\mathrm{ref}}$ at each refresh leaves test accuracy and
post-training compressibility unchanged within the observed seed-to-seed
variability.
At the final epoch, the absolute differences between the paired runs are
$0.69$, $0.26$, and $0.11$ percentage points for seeds $0$, $1$, and $2$,
respectively.
At $70\%$ pruning, they are $0.61$, $0.22$, and $0.46$ percentage points.
For comparison, changing the seed of the 1-step refresh run gives differences
of $0.09$--$0.80$ percentage points at the final epoch and
$0.05$--$0.87$ percentage points at $70\%$ pruning.
The sign of the paired difference is not consistent across seeds.

\begin{figure}[t]
    \centering
    \includegraphics[width=0.85\linewidth]
    {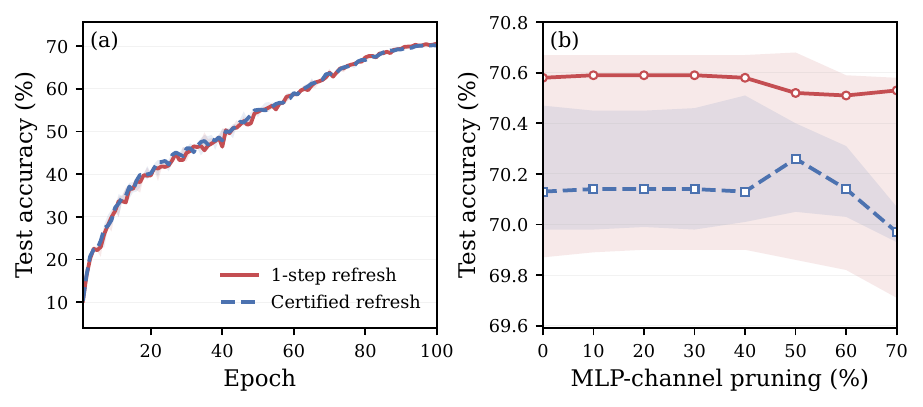}

    \caption{
    \textbf{End-to-end effect of the allocation approximation.}
    We compare the 1-step refresh run with the certified refresh run; both
    refresh $\mu$ every five optimizer steps.
    \textbf{(a)}
    CIFAR-100 test accuracy during 100 epochs of training.
    \textbf{(b)}
    Test accuracy after global structured MLP-channel pruning of the final
    trained model, using its final $\mu$ to rank channels, without fine-tuning.
    For each curve, the line is the median over three seeds and the shaded region is their minimum--maximum range.
    }
    \label{fig:oracle_end_to_end}
\end{figure}

\paragraph{Scope of the comparison.}
Both runs refresh $\mu$ every five optimizer steps. Thus, the end-to-end
comparison isolates the difference between one practical MM update and the
certified allocation $\mu_t^{\mathrm{ref}}$ at each refresh; it does not
compare against a run that recomputes a certified allocation at every optimizer
step. Figure~\ref{fig:oracle_approximation}(b) separately shows that keeping
$\mu$ fixed during the four intervening steps introduces only a small
objective error at the sampled states.

\subsection{Computational Overhead}
\label{app:overhead}

Relative to Muon, BMuon's additional computation is confined to the budgeted
matrices (the MLP weights; Algorithm~\ref{alg:bmuon}); all other parameters
are updated exactly as in Muon.
On steps where the allocation is not updated, BMuon uses the same number of
Newton--Schulz polar evaluations as Muon, with only additional channelwise
row/column scalings by $D_\mu^\alpha$.
On an allocation-update step, BMuon performs one additional polar evaluation
per budgeted matrix to construct the MM surrogate, before recomputing the
weighted polar factor used for the parameter update.
Updating the allocation every $\Delta T_\mu$ steps therefore raises the
amortized number of polar evaluations per budgeted matrix from $1$ to
$1+1/\Delta T_\mu$ per optimizer step, i.e.\ $1.2$ for $\Delta T_\mu=5$.
Once the surrogate polar factors are available, the channel scores require
only row/column inner products, and the allocation update consists of
channelwise operations and a scalar bisection over the global budget.
The only persistent additional optimizer state is the allocation vector
$\mu\in\mathbb{R}^N$, one scalar per budgeted channel, which is negligible
next to the momentum buffers.

We verify empirically that this overhead is small in practice by comparing
end-to-end training times against Muon at the allocation-update frequency
used throughout our experiments, $\Delta T_\mu=5$, on the ImageNet-1K
ViT-S/16 and 1B-parameter LLaMA settings of
Appendices~\ref{app:gradual-pruning-details} and~\ref{app:llama-details}.
All measurements use eight H100 GPUs and the same Muon and BMuon
implementations as the corresponding training runs.

\paragraph{ImageNet-1K.}
Table~\ref{tab:imagenet-overhead} reports the mean epoch time for each full
$150$-epoch gradual-pruning run (training only, excluding evaluation).
Averaged over three seeds, BMuon and Muon require $92.5\pm0.4$ and
$91.9\pm1.2$ seconds per epoch at $90\%$ final pruning, and
$92.1\pm0.5$ and $91.9\pm0.6$ seconds at $95\%$.
The mean differences ($0.6\%$ and $0.3\%$) are comparable to the
seed-to-seed variation, so the runs do not resolve an end-to-end slowdown
of BMuon relative to Muon.

\begin{table}[t]
\centering
\caption{\textbf{ImageNet-1K training time.}
Mean epoch time in seconds over each full $150$-epoch gradual-pruning run,
excluding evaluation. Differences between optimizers are comparable to
seed-to-seed variation.}
\label{tab:imagenet-overhead}
\begin{tabular}{ccccc}
\toprule
& \multicolumn{2}{c}{$90\%$ pruning}
& \multicolumn{2}{c}{$95\%$ pruning} \\
\cmidrule(lr){2-3} \cmidrule(lr){4-5}
Seed & BMuon & Muon & BMuon & Muon \\
\midrule
0 & 92.55 & 92.06 & 92.47 & 92.50 \\
1 & 92.02 & 90.72 & 92.27 & 91.59 \\
2 & 92.81 & 93.05 & 91.57 & 91.48 \\
\bottomrule
\end{tabular}
\end{table}

\paragraph{LLaMA.}
Table~\ref{tab:llama-overhead} reports timing measurements from the
$13.1$B-token LLaMA training runs.
For each run, the reported time per step is the total measured training time
divided by the $6{,}242$ training steps.
The training time excludes evaluation and checkpointing.

\begin{table}[t]
\centering
\caption{\textbf{LLaMA training time.}
Mean time per training step and total training time for each full
$13.1$B-token run, excluding evaluation and checkpointing.
The difference between optimizers is comparable to seed-to-seed variation
and does not indicate a slowdown of BMuon.}
\label{tab:llama-overhead}
\begin{tabular}{ccccc}
\toprule
& \multicolumn{2}{c}{BMuon}
& \multicolumn{2}{c}{Muon} \\
\cmidrule(lr){2-3} \cmidrule(lr){4-5}
Seed
& ms/step & Time (h)
& ms/step & Time (h) \\
\midrule
0 & 4832.4 & 8.379 & 4929.3 & 8.547 \\
1 & 4836.0 & 8.385 & 4893.5 & 8.485 \\
2 & 4906.5 & 8.507 & 4932.0 & 8.551 \\
\bottomrule
\end{tabular}
\end{table}

Averaged over three seeds, a step takes $4858\pm42$\,ms with BMuon and
$4918\pm22$\,ms with Muon.
The $1.2\%$ difference is comparable to the seed-to-seed
variation and can also be affected by normal system-level variation between
separate jobs, such as job placement, per-GPU performance, and communication, so we do not attribute it to the
optimizer.
The runs show no end-to-end slowdown of BMuon relative to Muon.

\newpage
\section{Experimental Details}
\label{app:experimental-details}

This section provides the experimental details and additional results for
Section~\ref{sec:experiments}. Appendix~\ref{app:motivation_pruning} gives the
ImageNet-1K post-training pruning setup and full pruning results.
Appendix~\ref{app:cifar100-posttraining} gives the CIFAR-100 post-training
compression setup and full pruning results.
Appendix~\ref{app:k-ablation} studies the effect of the global budget $K$.
Appendix~\ref{app:gradual-pruning-details} gives the gradual-pruning setup and
architecture controls for CIFAR-100 and ImageNet-1K.
Appendix~\ref{app:llama-details} provides the experimental details and full
pruning results for the LLaMA/FineWeb experiments.

\subsection{ImageNet-1K Post-Training Structured Pruning}
\label{app:motivation_pruning}

\paragraph{Architecture and training setup.}
We use a custom PyTorch ViT-S/16 implementation following the plain-ViT
design of~\citet{beyer2022better}: $224\times224$ inputs, $16\times16$
patches, embedding dimension $384$, 12 Transformer blocks, six attention
heads, MLP width $1536$, fixed two-dimensional sinusoidal positional
embeddings, global average pooling over patch tokens, and a pre-logits
representation layer. The model therefore contains
$12\times1536=18{,}432$ MLP hidden channels. Our ImageNet-1K optimization
setup follows the ViT-S/16 configuration of~\citet{zhong2026sharpness}.
Matrix-valued parameters within the Transformer blocks are optimized with
Muon, while the remaining parameters are optimized with AdamW. We use global
batch size $4096$, Muon learning rate $3\sqrt{3}\times10^{-3}$, AdamW
learning rate $3\times10^{-3}$, weight decay $0.3$, Muon momentum $0.95$
with Nesterov, and dropout $0.1$. Training uses sigmoid binary cross entropy
with label smoothing $0.1$, random resized crops, and horizontal flips,
without Mixup, CutMix, RandAugment, random erasing, or stochastic depth.
Muon uses five Newton--Schulz iterations following~\citet{jordan2024muon}
and the Moonlight-style RMS-matching scale
$0.2\sqrt{\max(d_{\mathrm{out}},d_{\mathrm{in}})}$~\citep{liu2025muon}.
All models are trained for 150 epochs with 5000 warmup steps followed by
cosine learning-rate decay to zero, using a half-length adaptation of the
300-epoch, 10,000-step ViT-S/16 schedule of~\citet{zhong2026sharpness}.

\paragraph{BMuon training.}
BMuon is applied only to the MLP matrices; the remaining Muon-eligible
Transformer matrices use ordinary Muon. We use $\alpha=1/2$, one MM
allocation update every five optimizer steps, and $\varepsilon=10^{-6}$.
Let $N=18{,}432$ denote the total number of MLP hidden channels. The global
allocation budget is initialized at $K(0)=N$ and linearly annealed to $K_{\mathrm{target}}=5{,}530\approx0.30N$
during the first $25\%$ of training, after which it is held fixed for the
remaining $75\%$.

\paragraph{Post-training pruning.}
After training, we take the final epoch-150 checkpoint and perform one-shot
global structured pruning of coupled MLP hidden channels. Each pruning level
is evaluated independently from the same dense checkpoint; pruning is
therefore not cumulative, and no post-pruning fine-tuning or recovery is
performed. For channel $j$ in block $\ell$, we use the coupled magnitude
score
$s_j^{\ell}=\sqrt{
\|W^\ell_{1,j:}\|_2^2+
\|W^\ell_{2,:,j}\|_2^2
}$. All $18{,}432$ channels are ranked jointly across the 12 blocks, subject to
a minimum of four surviving channels per block. We evaluate global
MLP-channel pruning levels $
0,10,20,30,40,50,60,70,80,90,95\%.$ Figure~\ref{fig:motivation} compares ordinary Muon and BMuon using this same magnitude-based pruning rule.

\paragraph{Evaluation.}
We run seeds $0,1,2$ and evaluate Top-1 accuracy on the full 50,000-image
ImageNet-1K validation set. Reported curves show the mean and sample standard
deviation across the three seeds. Table~\ref{tab:imagenet_posttrain_pruning}
reports the exact results corresponding to Figure~\ref{fig:motivation}.

\begin{table}[t]
\centering
\caption{\textbf{ImageNet-1K post-training structured pruning.}
Top-1 accuracy (\%) after global MLP-channel pruning of ViT-S/16 using coupled
weight-magnitude ranking. Results are mean $\pm$ sample standard deviation over
three seeds. No post-pruning recovery is performed.}
\label{tab:imagenet_posttrain_pruning}
\small
\begin{tabular}{c|cc}
\toprule
\textbf{MLP pruning (\%)} &
\textbf{Muon + mag.} &
\textbf{BMuon + mag.} \\
\midrule
0  & $75.230 \pm 0.022$ & $74.174 \pm 0.220$ \\
10 & $64.201 \pm 1.184$ & $74.183 \pm 0.236$ \\
20 & $26.762 \pm 1.387$ & $74.188 \pm 0.230$ \\
30 & $ 3.167 \pm 0.784$ & $74.180 \pm 0.228$ \\
40 & $ 0.513 \pm 0.211$ & $74.182 \pm 0.225$ \\
50 & $ 0.249 \pm 0.057$ & $74.177 \pm 0.230$ \\
60 & $ 0.239 \pm 0.025$ & $74.185 \pm 0.227$ \\
70 & $ 0.185 \pm 0.062$ & $74.083 \pm 0.324$ \\
80 & $ 0.149 \pm 0.012$ & $ 0.227 \pm 0.057$ \\
90 & $ 0.132 \pm 0.050$ & $ 0.138 \pm 0.047$ \\
95 & $ 0.105 \pm 0.009$ & $ 0.138 \pm 0.031$ \\
\bottomrule
\end{tabular}
\end{table}

\subsection{CIFAR-100 Post-Training Compression}
\label{app:cifar100-posttraining}

\paragraph{Architecture and shared training.}
We use a six-block Vision Transformer on CIFAR-100 with $4\times4$ patches,
embedding dimension 192, three attention heads, and MLP width 768. All methods
are trained for 200 epochs with seeds $\{0,1,2\}$ using batch size 256. For the matrix optimizer, we use peak learning rate $0.02$, momentum $0.95$, weight decay $0.05$, and five Newton--Schulz steps. Biases, LayerNorm parameters, the CLS token, positional embeddings, patch-embedding layer, and classifier head use AdamW with peak learning rate $3\times10^{-4}$, betas $(0.9,0.999)$, and weight decay $0.05$. We use five warmup epochs followed by cosine decay. Training uses standard CIFAR-100 augmentations, including random cropping, horizontal flipping, RandAugment, random erasing, Mixup, and CutMix. All reported compression results are recomputed from the final epoch-200 checkpoint. Evaluation uses the CIFAR-100 test set with batch size 512 and no test-time augmentation. Reported uncertainties are sample standard deviations over the three seeds.

\paragraph{BMuon.}
BMuon uses $\alpha=1/2$ and a global budget over all 4608 MLP hidden
channels. The budget is linearly annealed from 4608 to 922, corresponding
approximately to $20\%$ of the total channel mass, over the first third of
training and then held fixed. The allocation is updated every five optimizer
steps using one MM update, with $\varepsilon=10^{-6}$. Non-MLP Transformer
matrices use ordinary Muon.

\textbf{Muon+GL.}
Muon+GL applies an exact proximal Group-Lasso step with coefficient
$\lambda_{\mathrm{GL}}$ to coupled MLP hidden-channel groups
$(W^\ell_{1,j:},W^\ell_{2,:,j})$, using the group norm
$\bigl(\|W^\ell_{1,j:}\|_2^2+\|W^\ell_{2,:,j}\|_2^2\bigr)^{1/2}$.
We tune $\lambda_{\mathrm{GL}}\in\{0.01,0.02,0.05,0.1,0.2,0.5\}$, obtaining mean
dense accuracies of 71.50, 71.58, 71.25, 70.03, 69.23, and 61.21\%, respectively.
Increasing $\lambda_{\mathrm{GL}}$ strengthens the sparsity-inducing penalty and
can improve post-training compressibility at the expense of dense accuracy.
To compare at similar pre-compression accuracy, we restrict attention here to
settings within one percentage point of BMuon's $71.46\%$, leaving
$\lambda_{\mathrm{GL}}\in\{0.01,0.02,0.05\}$.
At $50\%$ MLP-channel pruning, these settings retain $9.71\%$, $16.30\%$,
and $71.25\%$ test accuracy, respectively.
We therefore use $\lambda_{\mathrm{GL}}=0.05$ throughout this experiment.

\paragraph{NuMuon.}
NuMuon~\citep{dolatabadi2026numuon} is applied during training to all
Muon-eligible Transformer weight matrices: the fused QKV projection, attention output projection, and the two MLP projections in each of the six blocks, for 24 matrices in total. Following the authors' cosine rank-scheduling strategy, we use a $10\%$ initial full-rank period and, in our implementation, cosine-anneal the relative update rank from $1.0$ to $0.25$ between $10\%$ and $80\%$ of training, holding it at $0.25$ thereafter. Since all eligible matrices have $q=\min(d_{\mathrm{in}},d_{\mathrm{out}})=192$, this corresponds to a
final update rank of $48$. Full-rank steps use the standard Newton--Schulz Muon update, while reduced-rank steps use randomized Block Krylov with two Krylov iterations. The NuMuon paper reports an oversampling setting of $8$; in our implementation, we realize this as a Krylov working dimension $b=\min(q,k+8)$, where $k=\lceil r(t)q\rceil$ is the target update rank.

\paragraph{Post-training compression.}
Compression is restricted to the ViT MLP blocks. We evaluate compression ratios $\{0,0.1,\ldots,0.9,0.95\}$ independently from the same final checkpoint, without cumulative pruning or post-compression fine-tuning. Muon, Muon+GL, and BMuon use global hidden-channel pruning over all six blocks. Magnitude pruning ranks channel $j$ in block $\ell$ by
$s_j^{\ell}=\sqrt{\|W^\ell_{1,j:}\|_2^2+\|W^\ell_{2,:,j}\|_2^2}$.
BMuon is evaluated with both this ranking and its final learned allocation $\mu$. During global pruning, we skip any candidate removal that would leave an MLP block with fewer than four hidden channels.

\paragraph{NuMuon compression and matching.}
At evaluation, NuMuon compresses only the 12 MLP weight matrices using
truncated SVD; attention weights remain dense. We distribute the retained rank budget as evenly as possible across matrices,
with minimum rank 4. At each compression level, we choose the NuMuon rank so that the resulting whole-model parameter count is as close as possible to that of the corresponding structured-pruned model.

Table~\ref{tab:cifar100-posttraining} reports the exact post-training
compression results corresponding to Fig.~\ref{fig:cifar100-posttraining}.

\begin{table*}[t]
\centering
\caption{\textbf{CIFAR-100 post-training compression results.}
Test accuracy (\%) reported as mean $\pm$ sample standard deviation over
three seeds. All compression levels are evaluated independently from the final
epoch-200 checkpoint, without post-compression fine-tuning.}
\label{tab:cifar100-posttraining}
\small
\begin{tabular}{c|ccccc}
\toprule
Compression
& Muon
& Muon+GL
& NuMuon
& BM-mag
& BM-$\mu$ \\
\midrule
0\%  & $71.49\pm0.21$ & $71.25\pm0.02$ & $69.95\pm0.13$ & $71.46\pm0.15$ & $71.46\pm0.15$ \\
10\% & $66.12\pm0.86$ & $71.25\pm0.02$ & $69.90\pm0.08$ & $71.46\pm0.15$ & $71.46\pm0.15$ \\
20\% & $20.26\pm3.53$ & $71.25\pm0.02$ & $69.79\pm0.17$ & $71.46\pm0.15$ & $71.46\pm0.15$ \\
30\% & $5.72\pm0.55$  & $71.25\pm0.02$ & $69.60\pm0.04$ & $71.46\pm0.15$ & $71.46\pm0.16$ \\
40\% & $4.18\pm0.64$  & $71.25\pm0.02$ & $68.70\pm0.14$ & $71.46\pm0.15$ & $71.46\pm0.16$ \\
50\% & $3.77\pm1.38$  & $71.25\pm0.02$ & $66.65\pm0.36$ & $71.46\pm0.15$ & $71.46\pm0.14$ \\
60\% & $2.94\pm0.35$  & $30.45\pm4.40$ & $61.49\pm0.56$ & $71.46\pm0.15$ & $71.42\pm0.09$ \\
70\% & $2.81\pm0.49$  & $11.31\pm0.67$ & $20.81\pm5.48$ & $71.46\pm0.15$ & $71.38\pm0.05$ \\
80\% & $1.93\pm0.78$  & $6.01\pm0.46$  & $3.84\pm0.83$  & $60.84\pm9.74$ & $70.57\pm0.18$ \\
90\% & $2.19\pm0.60$  & $3.76\pm0.76$  & $1.86\pm0.68$  & $1.62\pm0.20$  & $5.82\pm1.46$ \\
95\% & $2.06\pm0.77$  & $3.63\pm0.35$  & $1.58\pm0.39$  & $1.44\pm0.41$  & $2.71\pm0.73$ \\
\bottomrule
\end{tabular}
\end{table*}

\subsection{Effect of the Global Budget}
\label{app:k-ablation}

We ablate the global budget $K$ on the CIFAR-100 post-training compression
setup of Appendix~\ref{app:cifar100-posttraining}.

\paragraph{Setup.} 
We use the same six-block ViT, 200-epoch training recipe, and optimization 
settings as in Appendix~\ref{app:cifar100-posttraining}. We train BMuon with 
$\alpha=1/2$ and 
$
K/N \in \{0.1,0.2,\ldots,0.9\}, 
$ 
with no budget schedule and holding all other 
hyperparameters fixed. After training, we perform global MLP-channel 
pruning without recovery, ranking channels either by coupled weight magnitude, 
$
s_j^{\ell} 
= 
\sqrt{\|W^\ell_{1,j:}\|_2^2+\|W^\ell_{2,:,j}\|_2^2}, 
$ or by the final learned allocation $\mu$.

\begin{figure*}[t]
    \centering
    \begin{subfigure}[t]{0.49\textwidth}
        \centering
        \includegraphics[width=\linewidth]{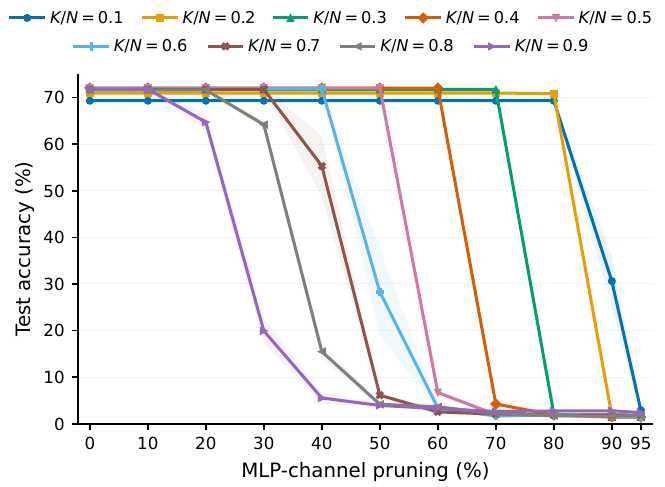}
        \caption{Magnitude ranking.}
        \label{fig:k-ablation-mag}
    \end{subfigure}
    \hfill
    \begin{subfigure}[t]{0.49\textwidth}
        \centering
        \includegraphics[width=\linewidth]{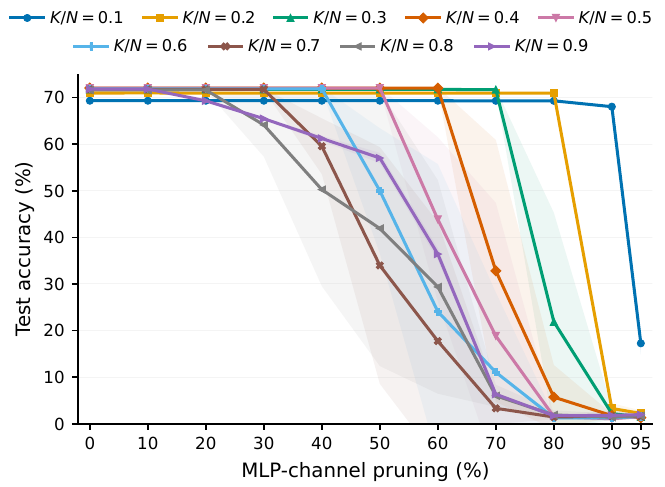}
        \caption{$\mu$ ranking.}
        \label{fig:k-ablation-mu}
    \end{subfigure}
    \caption{\textbf{Effect of the budget on post-training compression.}
    CIFAR-100 test accuracy after global MLP-channel pruning for different
    budget ratios $K/N$. Panel (a) uses coupled weight-magnitude ranking, while
    panel (b) uses the final learned allocation $\mu$. Mean $\pm$ sample std. over three seeds.}
    \label{fig:k-ablation}
    \vspace{-0.7em}
\end{figure*}

Figure~\ref{fig:k-ablation} shows that the budget directly controls the
compression level favored during training. With magnitude pruning, models
remain robust approximately up to pruning a fraction $1-K/N$ of the channels,
and degrade sharply beyond this point. The same pattern holds with $\mu$
ranking, but the degradation is generally more gradual, further showing that
$\mu$ provides a useful channel-importance signal. For the most restrictive
budgets, particularly $K/N=0.1$ and $0.2$, the accuracy of the unpruned model
also decreases slightly, indicating a tradeoff between aggressive
compressibility and performance before pruning.

\subsection{Gradual Structured Pruning Experiment}
\label{app:gradual-pruning-details}

All gradual-pruning experiments use seeds $0,1,2$. Only MLP hidden channels are pruned, and all surviving channels are ranked globally across blocks at each pruning event. We enforce at least four surviving channels per block. Muon-mag and BMuon-mag use the coupled magnitude score $s_j^{\ell}=\sqrt{\|W^\ell_{1,j:}\|_2^2+\|W^\ell_{2,:,j}\|_2^2}$, while BMuon-$\mu$ uses $s_j^{\ell,\mu}=\mu_j^\ell$. After pruning, training resumes from the surviving weights and optimizer state; pruned channels never return.

\paragraph{BMuon schedule.}
On both CIFAR-100 and ImageNet-1K, BMuon is applied only to the MLP matrices, with $\alpha=1/2$, one MM step every five optimizer steps, and $\varepsilon=10^{-6}$. At the start of each training phase $i$, let $N_{\mathrm{start},i}$ denote the current number of surviving channels and let $N_{\mathrm{end},i}$ denote the number of channels to retain at the next pruning event. During the first half of phase $i$, the BMuon budget is linearly annealed from $N_{\mathrm{start},i}$ to $N_{\mathrm{end},i}$; during the second half, it is held fixed at $N_{\mathrm{end},i}$. At the end of the phase, pruning reduces the network to exactly $N_{\mathrm{end},i}$ surviving channels. After pruning, the allocation coordinates $\mu$ of all surviving channels are reset to one.

\paragraph{CIFAR-100.}
We use a six-block ViT with $32\times32$ inputs, $4\times4$ patches, embedding dimension $112$, four attention heads, and MLP width $448$, giving $2688$ prunable channels. Training lasts 200 epochs with batch size 256. Muon uses learning rate $0.02$, momentum $0.95$ without Nesterov, weight decay
$0.05$, and five Newton--Schulz steps; auxiliary parameters use AdamW with
learning rate $3\times10^{-4}$ and weight decay $0.05$. We use five epochs of
linear warmup followed by cosine decay to $1\%$ of the initial learning rate. Training uses random cropping and horizontal flipping, RandAugment $(2,9)$, random erasing with probability $0.1$, and Mixup/CutMix with parameters $0.8/1.0$ and total probability $0.5$.

For the $90\%$ experiment, cumulative pruning levels of $20,40,60,75,85,$ and $90\%$ are applied after epochs $30,60,90,120,150,$ and $180$, respectively, followed by 20 recovery epochs. The corresponding surviving channel counts are $2150,1613,1075,672,403,$ and $269$. For the $95\%$ experiment, cumulative pruning levels of $20,40,60,75,85,90,$ and $95\%$ are applied after epochs $26,52,78,104,130,155,$ and $180$, respectively, followed by 20 recovery epochs. The corresponding surviving channel counts are $2150,1613,1075,672,403,269,$ and $134$.

\paragraph{ImageNet-1K.}
We use the same ViT-S/16 architecture, ImageNet-1K data pipeline, base
optimizer hyperparameters, and 150-epoch learning-rate schedule as in
Appendix~\ref{app:motivation_pruning}. Unlike the post-training pruning
experiment, here physical MLP-channel removal is interleaved with training
according to the gradual-pruning protocol described above.

For the $90\%$ experiment, cumulative pruning levels of $20,40,60,70,80,$ and $90\%$ are applied after epochs $22,44,66,87,108,$ and $129$, respectively, followed by 21 recovery epochs. The corresponding surviving channel counts are $14{,}746,11{,}059,7{,}373,5{,}530,3{,}686,$ and $1{,}843$. For the $95\%$ experiment, cumulative pruning levels of $20,40,60,70,80,90,$
and $95\%$ are applied after epochs $19,38,57,76,95,114,$ and $132$,
respectively, followed by 18 recovery epochs. The corresponding surviving
channel counts are 14,746, 11,059, 7,373, 5,530, 3,686, 1,843, and 922.
\paragraph{Evaluation.}
For CIFAR-100, we report final Top-1 accuracy on the test set at epoch 200. For ImageNet-1K, we report final Top-1 accuracy on the full 50,000-image validation set at epoch 150. In both cases, evaluation is performed after the final recovery stage. Reported means and standard deviations are over seeds $0,1,2$.
\paragraph{Architecture controls.}
For each BMuon-$\mu$ pruning run, we record the final per-block MLP widths. To evaluate the selected architecture on its own, we instantiate a new model with these widths and train it from scratch with standard Muon, without pruning, using the same training recipe and horizon. As a baseline, we also train a uniform-width model from scratch with standard Muon, with its total MLP channel count matched as closely as possible to the corresponding architecture discovered by BMuon-$\mu$. Seed pairing is one-to-one: the architecture discovered with seed $i$ is retrained from scratch using seed $i$. Thus, for seeds $0,1,2$, we perform three retraining runs rather than a $3\times3$ crossing of discovery and retraining seeds.

The selected architectures are reported in
Tables~\ref{tab:cifar-architectures} and~\ref{tab:imagenet-architectures}.

\begin{table}[t]
\centering
\caption{\textbf{BMuon-$\mu$ architectures on CIFAR-100.}
Per-block MLP widths for blocks $1$--$6$, selected at the end of gradual pruning.}
\label{tab:cifar-architectures}

\small
\setlength{\tabcolsep}{6pt}
\renewcommand{\arraystretch}{1.12}

\begin{tabular}{lccc}
\toprule
\textbf{Pruning}
& \textbf{Seed 0}
& \textbf{Seed 1}
& \textbf{Seed 2} \\
\midrule
90\%
& $[119,94,18,4,9,25]$
& $[120,95,20,4,13,17]$
& $[114,99,18,6,13,19]$ \\
95\%
& $[80,37,4,5,4,4]$
& $[78,39,5,4,4,4]$
& $[79,37,5,5,4,4]$ \\
\bottomrule
\end{tabular}
\end{table}

\begin{table*}[t]
\centering
\caption{\textbf{BMuon-$\mu$ architectures on ImageNet-1K.}
Per-block MLP widths for blocks $1$--$12$, selected at the end of gradual pruning.}
\label{tab:imagenet-architectures}

\small
\setlength{\tabcolsep}{6pt}
\renewcommand{\arraystretch}{1.12}

\begin{tabular}{lcl}
\toprule
\textbf{Pruning}
& \textbf{Seed}
& \textbf{MLP widths, blocks 1--12} \\
\midrule
90\% & 0
& $[250,168,194,164,159,153,141,140,106,122,177,69]$ \\
& 1
& $[254,174,204,197,151,132,141,139,98,135,153,65]$ \\
& 2
& $[245,194,216,199,158,119,131,137,104,121,159,60]$ \\
\cmidrule(lr){1-3}
95\% & 0
& $[181,120,120,106,100,67,58,43,40,34,39,14]$ \\
& 1
& $[170,121,133,127,91,60,62,53,34,22,43,6]$ \\
& 2
& $[156,110,107,99,100,73,59,66,37,40,55,20]$ \\
\bottomrule
\end{tabular}
\end{table*}

On CIFAR-100, the uniform baseline uses width $45$ in each of the six blocks at $90\%$ pruning and width $22$ at $95\%$, giving $270$ and $132$ total MLP channels, respectively, compared with $269$ and $134$ channels in the corresponding BMuon-$\mu$ architectures. On ImageNet-1K, the uniform baseline uses width $154$ in each of the 12 blocks at $90\%$ pruning and width $77$ at $95\%$, giving $1848$ and $924$ channels, compared with $1843$ and $922$ channels in the corresponding BMuon-$\mu$ architectures.

\subsection{LLaMA/FineWeb Experimental Details}
\label{app:llama-details}

\paragraph{Implementation provenance.}
Our LLaMA/FineWeb experiments are based on the LLaMA-1B setup of
Muon+~\citep{zhang2026muonplus}\footnote{\url{https://github.com/K1seki221/MuonPlus}}.
We adopt its architecture and training recipe, reducing only the number of
Transformer blocks from $24$ to $17$ to obtain a $986{,}746{,}880$-parameter
model; all other architectural hyperparameters are unchanged.
The Muon+ codebase builds on Muon~\citep{jordan2024muon} and the
\texttt{modded-nanogpt} training framework\footnote{\url{https://github.com/KellerJordan/modded-nanogpt}}.
We retain its LLaMA-2 tokenizer and FineWeb preprocessing and sharding
conventions, Muon parameter routing, auxiliary AdamW-style updates for
non-Muon parameters, momentum/Nesterov convention, five-step Newton--Schulz
polar approximation, matrix learning-rate shape/RMS scaling, warmup--cosine
schedule, and BF16 training conventions.
For our experiments, we integrate BMuon into the SwiGLU MLPs.

\paragraph{Model and data.}
We use a 17-layer pre-norm LLaMA-style causal decoder with hidden dimension
$2048$, $32$ attention and $32$ key--value heads, SwiGLU width $5461$,
context length $4096$, vocabulary size $32{,}000$, RMSNorm
($\epsilon=10^{-5}$), and RoPE with $\theta=10{,}000$.
The token embedding and language-model head are untied.
The model has $986{,}746{,}880$ parameters.
Each SwiGLU layer contains gate, up, and down projections
$W_g,W_u\in\mathbb{R}^{5461\times2048}$ and
$W_d\in\mathbb{R}^{2048\times5461}$,
giving $N=17\times5461=92{,}837$ structural MLP channels in total.
We train on $13{,}090{,}422{,}784$ tokens from the
\texttt{sample-100BT} configuration of FineWeb, using the LLaMA-2 tokenizer.

\paragraph{Training.}
Training uses BF16 with an effective global batch of $512$ sequences,
corresponding to $2{,}097{,}152$ tokens per optimizer step.
We use seeds $\{0,1,2\}$ for all methods.
For Muon matrices, we use learning rate $0.02$, momentum $0.95$ with
Nesterov momentum, five Newton--Schulz iterations, decoupled weight decay
$0.1$, gradient clipping at $1.0$, a $10\%$ linear warmup, and cosine decay
to zero.
Embeddings, normalization parameters, and the language-model head use the
auxiliary AdamW-style branch with
$\beta_1=\beta_2=0.95$ and $\epsilon=10^{-8}$.
In the Muon baseline, attention and MLP matrices use Muon.
For BMuon, the three SwiGLU projection matrices in each block are instead
updated by BMuon, while attention matrices continue to use standard Muon.

\paragraph{BMuon configuration.}
We use $\alpha=1/2$ and a budget $K=18{,}567\approx0.2N$ throughout training, with no budget schedule.
The allocation is updated every five optimizer steps with one MM iteration,
using a numerical floor $\epsilon=10^{-6}$.

\paragraph{Post-training pruning.}
We prune only SwiGLU hidden channels.
Pruning channel $j$ in layer $\ell$ jointly removes
$W_g^{\ell}[j,:]$, $W_u^{\ell}[j,:]$, and $W_d^{\ell}[:,j]$.
For each pruning ratio $r\in\{0,0.1,\ldots,0.9\}$, we restore the unpruned
checkpoint and prune directly to ratio $r$.
Scores are ranked globally across all Transformer blocks, with a minimum of
10 surviving channels per layer.

For both Muon and BMuon, the magnitude score is
\[
s_{\ell j}^{\mathrm{mag}}
=
\left(
\|W_g^{\ell}[j,:]\|_2^2
+
\|W_u^{\ell}[j,:]\|_2^2
+
\|W_d^{\ell}[:,j]\|_2^2
\right)^{1/2}.
\]
For BMuon, we additionally use
\[
s_{\ell j}^{\mu}=\mu_j^{\ell}.
\]
For the activation-aware criterion, let
$h_{\ell j}(x)
=
\operatorname{SiLU}\!\left(W_g^{\ell}[j,:]x\right)
\left(W_u^{\ell}[j,:]x\right)$
denote the hidden activation of channel $j$.
We score it by
\[
s_{\ell j}^{\mathrm{act}}
=
\sqrt{\mathbb{E}_x[h_{\ell j}(x)^2]}\,
\|W_d^{\ell}[:,j]\|_2.
\]
Activation statistics are accumulated in FP32 over $1{,}048{,}576$
calibration tokens.
Each pruned model is evaluated without recovery on $5{,}046{,}272$
validation tokens.

\begin{table*}[t]
\centering
\caption{\textbf{Post-training structured pruning of the 1B-parameter
LLaMA-style model.}
Validation perplexity after global SwiGLU-channel pruning.
Entries report mean $\pm$ sample std. over seeds $\{0,1,2\}$.
No post-pruning recovery is performed.}
\label{tab:llama-pruning-full}

\small
\setlength{\tabcolsep}{4pt}

\resizebox{\textwidth}{!}{%
\begin{tabular}{c|cc|ccc}
\toprule
\multirow{2}{*}{\textbf{MLP pruning}}
&
\multicolumn{2}{c|}{\textbf{Muon}}
&
\multicolumn{3}{c}{\textbf{BMuon}}
\\
& Magnitude & Activation
& Magnitude & Activation & $\mu$ \\
\midrule
$0\%$
& $10.921\pm0.033$
& $10.921\pm0.033$
& $12.843\pm0.051$
& $12.843\pm0.051$
& $12.843\pm0.051$
\\
$10\%$
& $16.100\pm1.411$
& $12.701\pm0.229$
& $12.843\pm0.051$
& $12.843\pm0.051$
& $12.843\pm0.051$
\\
$20\%$
& $22.944\pm3.728$
& $26.249\pm1.974$
& $12.843\pm0.051$
& $12.843\pm0.051$
& $12.843\pm0.051$
\\
$30\%$
& $40.241\pm8.109$
& $62.840\pm7.917$
& $12.843\pm0.051$
& $12.843\pm0.051$
& $12.843\pm0.051$
\\
$40\%$
& $76.861\pm16.225$
& $251.840\pm55.803$
& $12.843\pm0.051$
& $12.843\pm0.051$
& $12.843\pm0.051$
\\
$50\%$
& $160.049\pm43.632$
& $1441.609\pm641.338$
& $12.843\pm0.051$
& $12.843\pm0.051$
& $12.843\pm0.051$
\\
$60\%$
& $283.004\pm52.645$
& $2984.678\pm1476.332$
& $12.843\pm0.051$
& $12.843\pm0.051$
& $12.843\pm0.051$
\\
$70\%$
& $615.949\pm131.121$
& $3543.424\pm1733.400$
& $12.843\pm0.051$
& $12.843\pm0.051$
& $12.843\pm0.051$
\\
$80\%$
& $2486.429\pm1684.264$
& $5290.062\pm2508.521$
& $13.342\pm0.178$
& $12.857\pm0.053$
& $12.854\pm0.053$
\\
$90\%$
& $28140.823\pm19185.679$
& $8382.535\pm5262.682$
& $3091.740\pm1717.202$
& $245.636\pm42.251$
& $1168.909\pm351.685$
\\
\bottomrule
\end{tabular}%
}
\end{table*}

\newpage

\end{document}